\documentclass[11pt, letterpaper]{article}
 \usepackage[margin=1.15in]{geometry}
\usepackage{enumitem}
\usepackage{needspace}
\usepackage{authblk}
\usepackage{tikz}
\usepackage{xcolor}
\usepackage[most]{tcolorbox}
\usepackage[dvipsnames]{xcolor}
\usetikzlibrary{positioning,decorations.pathreplacing,calc}
\usepackage{amsmath, amssymb, amsthm, amsfonts}
\usepackage{graphicx}
\usepackage{tikz}
\usepackage{pdfpages}
\usepackage{caption}
\usepackage{float}
\usepackage{dsfont}
\usepackage[english]{babel}
\usepackage{soul}
\usepackage{mathtools}
\usepackage{thmtools}
\usepackage{thm-restate}

\newtheorem{theorem}{Theorem}[section]
\newtheorem{lemma}[theorem]{Lemma}

\newtheorem{corollary}[theorem]{Corollary}
\newtheorem{proposition}[theorem]{Proposition}
\newtheorem{definition}[theorem]{Definition}
\newtheorem{remark}[theorem]{Remark}

\numberwithin{equation}{section}
\usepackage[hidelinks]{hyperref} 
\usepackage{pdfcomment}
\usepackage{cleveref}

\title{Sharp Low-Degree Thresholds for Planted-vs-Planted Testing}

\date{}

\newcommand{\Adv}{\mathsf{Adv}}
\newcommand{\Corr}{\mathsf{Corr}}

\newcommand{\cJ}{\mathcal{J}}
\newcommand{\cI}{\mathcal{I}}

\newcommand{\E}{\mathbb{E}}
\newcommand{\PP}{\mathbb{P}}
\newcommand{\QQ}{\mathbb{Q}}
\newcommand{\R}{\mathbb{R}}
\newcommand{\U}{\mathcal{U}}
\newcommand{\goodI}{\widehat{\mathcal{I}}}
\newcommand{\goodJ}{\widehat{\mathcal{J}}}
\newcommand{\cU}{\mathcal{U}}

\begin{document}

\author[1]{\mbox{Anda Skeja}}
\author[2]{Daniel Gutiérrez Espinoza}
\author[3]{Fiona Skerman}
\author[4]{Alexander S.\ Wein}
\affil[1,2,3]{%
Department of Mathematics, Uppsala University\\
}
\affil[4]{%
Department of Mathematics, University of California, Davis
}

\maketitle

\renewcommand{\thefootnote}{}

\footnotetext{Emails: \(\{\)\textit{anda.skeja, daniel.gutierrez\_espinoza, fiona.skerman\(\}\)@math.uu.se;
aswein@ucdavis.edu.}}

\renewcommand{\thefootnote}{\arabic{footnote}}
\begin{abstract}%
We establish the first sharp thresholds for low-degree polynomial tests in planted-vs-planted settings, where the goal is to determine with vanishing error which of two structured planted mechanisms generated the observed data. We prove matching low-degree upper and lower bounds for counting communities in the planted submatrix and planted dense subgraph models. The resulting testing threshold coincides, down to the sharp constant, with the known low-degree recovery threshold. In contrast, the task of weak testing, where the goal is to outperform random guessing, does not have a sharp threshold but rather a smooth transition, which we identify.
To prove our results, we develop a framework for planted-vs-planted testing that builds on a latent-variable expansion originating in low-degree recovery and employs new methods to identify and prune non-signal contributions.
\end{abstract}


\section{Introduction}
A widespread phenomenon in high-dimensional inference is the presence of sharp computational phase transitions, in which the inference task shifts from computationally ``easy'' to ``hard'' as the signal strength crosses a critical threshold. A notable example is the \emph{Kesten--Stigum (KS) threshold} in the stochastic block model (SBM)~\cite{decelle-2,AS-general}. When the \emph{signal-to-noise ratio (SNR)} --- a certain function of model parameters --- lies above the so-called KS threshold, the task of inferring (with nontrivial accuracy) the community structure from the graph is ``easy'' in the sense that a polynomial-time algorithm is known. On the other hand, when the SNR lies below the KS threshold, there is no known polynomial-time algorithm, and various heuristics indicate that inference is fundamentally ``hard'', at least for certain classes of algorithms~\cite{decelle-2,HS-bayesian,spectral-planting,sohn2025sharp,DHSS-recovery}. Similar computational thresholds occur in many other statistical models and, in general, may differ from the \emph{statistical threshold} at which \emph{some} algorithm (with unconstrained runtime) succeeds.

There has been sustained interest in pinning down the precise location of these sharp computational thresholds across various models. By \emph{sharp}, we mean that the required runtime changes abruptly when a natural SNR parameter crosses a threshold, and we aim for results that identify the threshold \emph{exactly} (not just up to constant factors). Establishing a sharp threshold result requires two parts: a positive algorithmic result above the threshold and a hardness result below it. While our focus will be on sharp thresholds, we note that not all computational thresholds are sharp: some problems, such as planted clique~\cite{alon-clique} and tensor PCA (see~\cite{smooth-tensor}), exhibit a smooth trade-off between SNR and runtime.  Except in the special case where the statistical and computational thresholds coincide, we do not currently have tools to definitively prove hardness for statistical problems in a complexity-theoretic sense (due to the \emph{average-case} nature of these problems). Instead, a common approach is to show that some restricted class of known algorithms must fail. Two common frameworks are \emph{(i)} methods associated with statistical physics that revolve around analyzing the \emph{belief propagation (BP)} or \emph{approximate message passing (AMP)} algorithms (see~\cite{phys-survey}), and \emph{(ii)} algorithms based on low-degree polynomials (see~\cite{LD-notes,wein2025computational}). There are also other frameworks for average-case complexity, but we focus here on those with a track record of establishing \emph{sharp} thresholds. In the case of the KS threshold for the SBM, the sharp phase transition was first predicted using physics methods~\cite{decelle-1,decelle-2} and later corroborated in the low-degree polynomial model~\cite{HS-bayesian,spectral-planting,sohn2025sharp,DHSS-recovery}.

To complicate things further, there are a few different objectives one might study. \emph{Detection} is the task of hypothesis testing between two different distributions, usually a ``planted'' distribution that contains some hidden structure and a ``null'' distribution that does not. For the SBM, the planted model is a random graph with communities, and the null model is typically an Erd\H{o}s--R\'enyi graph with the same average edge density. \emph{Recovery} is the task of finding the hidden structure, or estimating it to some desired accuracy. The computational thresholds for detection and recovery need not be the same in general. In the SBM, for instance, these two thresholds\footnote{More precisely, ``strong'' (high-probability) detection and ``weak'' (nontrivial) recovery coincide.} coincide for the standard variant where the number of communities is held constant, but detection becomes easier than recovery when the number of communities grows with~$n$ (see~\cite{sbm-many,sbm-many-2}). 

Our focus in this work is on planted-vs-planted testing:  distinguishing between two complex distributions, each with a different type of planted structure. 
Existing results for such problems~\cite{rush2023easier,coloring-clique,fourier-geo,carpentier2025low} are too coarse to establish the type of sharp thresholds we are interested in here. Two questions follow: can planted-vs-planted testing problems exhibit sharp thresholds, pinned down to the leading constant, and is there a systematic framework for proving them? Statistical-physics methods do not seem directly suited to this setting, as they are primarily tailored to recovery. This leaves the low-degree framework---which can handle both testing and recovery---as the natural candidate. We answer both questions affirmatively. We develop a low-degree certificate framework for planted-vs-planted testing and use it to establish the first sharp low-degree thresholds in this setting. For planted models whose observations are conditionally independent given the latent structure, the framework reduces the low-degree lower bound to a linear certificate problem supported on informative subgraphs.

\subsection{Main contributions}
We establish sharp low-degree thresholds for ``counting communities'': the problem of testing whether the data contain $\ell$ or $\ell'$ planted structures, in both the planted submatrix model (PSM) and the planted dense subgraph model (PDS). Here, $\ell\neq\ell'$ are arbitrary fixed positive integers. This testing problem was introduced in~\cite{rush2023easier}, where thresholds were determined up to polylogarithmic factors. 

Our results sharpen this picture to the level of sharp leading constants by developing a general low-degree certificate framework for planted-vs-planted testing. For this problem, our framework yields the exact leading constant for the low-degree strong-testing threshold in the stated sparse regimes, and identifies the weak-testing scale. Here, \emph{strong} testing means testing with vanishing error probability, while \emph{weak} testing means achieving testing power bounded away from that of random guessing. A connection between planted-vs-planted testing and recovery was previously established in \cite{rush2023easier}, where approximate recovery in the one-community planted submatrix model was shown to imply strong testing between the $1$- and $2$-planted models (up to a factor of 2, which we expect can be improved to 1). This left open whether testing might be strictly easier than recovery, and whether the testing threshold depends on the pair $(\ell,\ell')$. For example, distinguishing $1$ from $100$ planted communities might seem easier than distinguishing $99$ from $100$. Our results show that this is not the case at the level of sharp low-degree thresholds: in both models, for every fixed distinct pair $\ell,\ell'$, the strong-testing threshold is the same and matches the corresponding sharp low-degree recovery threshold~\cite{sohn2025sharp}.

\medskip
Briefly, the $\ell$-planted submatrix model is generated as follows. The parameter $\rho$ controls the sparsity of the planted vertices, while $\lambda$ controls the signal strength. For each index~$i\in[n]$, independently assign label $c\in[\ell]$ with probability $\rho/\ell$ and label~$0$ otherwise. Conditioned on these labels, for all $1\le i\le j\le n$, sample $Y_{ij}\sim N(\ell\lambda,1)$ if $i$ and $j$ lie in the same nonzero community, and $Y_{ij}\sim N(0,1)$ otherwise, with $Y_{ji}=Y_{ij}$.

Our low-degree lower bounds are proved by controlling the degree-$D$ advantage
\begin{equation}\label{eq.adv_def}\Adv_{\le D}(\PP,\QQ)
:=\sup_{\deg(f)\le D}
\frac{\E_{\PP}[f(Y)]}{\sqrt{\E_{\QQ}[f(Y)^2]}} .
\end{equation}
Bounds of the form $\Adv_{\le D}(\PP,\QQ)=O(1)$ rule out degree-$D$ strong separation, while bounds $\Adv_{\le D}(\PP,\QQ)=1+o(1)$ rule out degree-$D$ weak separation; see \Cref{lem:separation-weak-strong}. The matching upper bounds are given by explicit low-degree separating polynomials.
\subsubsection{Results}We give separate strong- and weak-testing theorems for both PSM and PDS. For readability, we state the PSM results here; the corresponding PDS statements appear in Theorems~\ref{thm:strong_PDS} and~\ref{thm:Weak_PDS}. See also Figure~\ref{fig:phases_PSM}, which depicts the phase diagram for PSM. Throughout, $n\to\infty$, with $\ell,\ell'$ fixed and all other parameters allowed to depend on~$n$. 
\paragraph*{Strong and Weak Testing Results for the Planted Submatrix Model}\begin{restatable}[Strong testing: PSM]{theorem}{thmPSMstrong}
\label{thm:strong_PSM}
Given parameters $n,\ell,\ell',\rho,\lambda$, with 
$\ell,\ell'$ fixed distinct positive integers, $\rho\in(0,1)$, and
$\lambda>0$, define
$\QQ:=\mathbb{P}_{\mathrm{PSM}}(n,\ell,\rho,\lambda)$ and
$\PP:=\mathbb{P}_{\mathrm{PSM}}(n,\ell',\rho,\lambda)$. For any constant $\varepsilon>0$, there exists a constant $C_0\equiv C_0(\ell,\ell',\varepsilon)>0$ such that the following hold. 
\begin{enumerate}
  \item[(i)] \emph{(Lower bound)}. If 
  $$\lambda \leq (1-\varepsilon)\Big(\rho\sqrt{en}\Big)^{-1}, \qquad D\leq \lambda^{-2}/C_0, \qquad \rho=o(1)$$
  then
  $\Adv_{\leq D}(\PP,\QQ)=O(1).$

  \item[(ii)] \emph{(Upper bound)}. If 
  $$\lambda \geq (1+\varepsilon)\Big(\rho\sqrt{en}\Big)^{-1}, \quad \!\!\! n\rho=\omega(\log^7n), \qquad\!\!\!
\rho=o(\log^{-7}n)
  $$ then there exists a polynomial $f$ of degree at most $C_0\log n$ that strongly separates $\PP$ and $\QQ$.
\end{enumerate}
\end{restatable}
The lower bound also applies to all $D \le \lambda^{-2}/C_0$. Under the standard low-degree heuristic, where degree $D$ is interpreted as a proxy for runtime $n^{\widetilde O(D)}$~\cite{HS-bayesian}, this points to the scale $D\asymp \lambda^{-2}$ as the relevant subcritical degree scale. The polynomial in the upper bound above is based on counting occurrences of balanced unicyclic graphs (BUGs); see Equation~\eqref{def:BUG_f} for the definition and Figure~\ref{fig:BUGs3} for an illustration.
\begin{restatable}
[Weak testing: PSM]{theorem}{thmPSMweak}\label{thm:weak_PSM}
Given parameters $n,\ell,\ell',\rho,\lambda$, with 
$\ell,\ell'$ fixed distinct positive integers, $\rho\in(0,1)$, and
$\lambda>0$, define
$\QQ:=\mathbb{P}_{\mathrm{PSM}}(n,\ell,\rho,\lambda)$ and
$\PP:=\mathbb{P}_{\mathrm{PSM}}(n,\ell',\rho,\lambda)$. There exists a constant $C_0 \equiv C_0(\ell,\ell')>0$ such that the following hold.
\begin{enumerate}
  \item[(i)] \emph{(Lower bound)}. If
  $$\lambda = o \left(({\rho\sqrt{n}})^{-1}\right), \qquad D\leq \lambda^{-2}/C_0, \qquad \rho=o(1),$$
  then
  $\Adv_{\leq D}(\PP,\QQ)=1+o(1).
  $

  \item[(ii)] \emph{(Upper bound)}. If  
  $$\lambda =\Omega\left(({\rho\sqrt{n}})^{-1}\right), \qquad   n\rho =\omega(1),$$
  then the degree-$1$ polynomial $f(Y)=\sum_i Y_{ii}$ weakly separates $\mathbb{P}$ and $\mathbb{Q}$.
\end{enumerate}
\end{restatable}

Our results show that weak planted-vs-planted testing is strictly easier than strong planted-vs-planted testing in terms of the required signal strength $\lambda$: strong testing has the sharp threshold $(\rho\sqrt{en})^{-1}$ whereas weak testing is possible for any small constant times $(\rho\sqrt n)^{-1}$. The strong testing threshold coincides with the sharp low-degree recovery threshold $\lambda_{\mathrm{rec},1}$ for the corresponding one-community models~\cite{submatrix-message-passing,sohn2025sharp}. Informally,
\[
    \lambda_{\ell\,{\rm vs}\,\ell'}^{\rm weak}
    <
    \lambda_{\ell\,{\rm vs}\,\ell'}^{\rm strong}
    =
    \lambda_{{\rm rec},1}^{\rm weak}
    =
    \lambda_{{\rm rec},1}^{\rm strong}.
\]
See Figure~\ref{fig:phases_PSM} for the resulting phase diagram in the PSM setting.
\begin{figure}[t]
    \centering
    \includegraphics[width=0.6\linewidth]{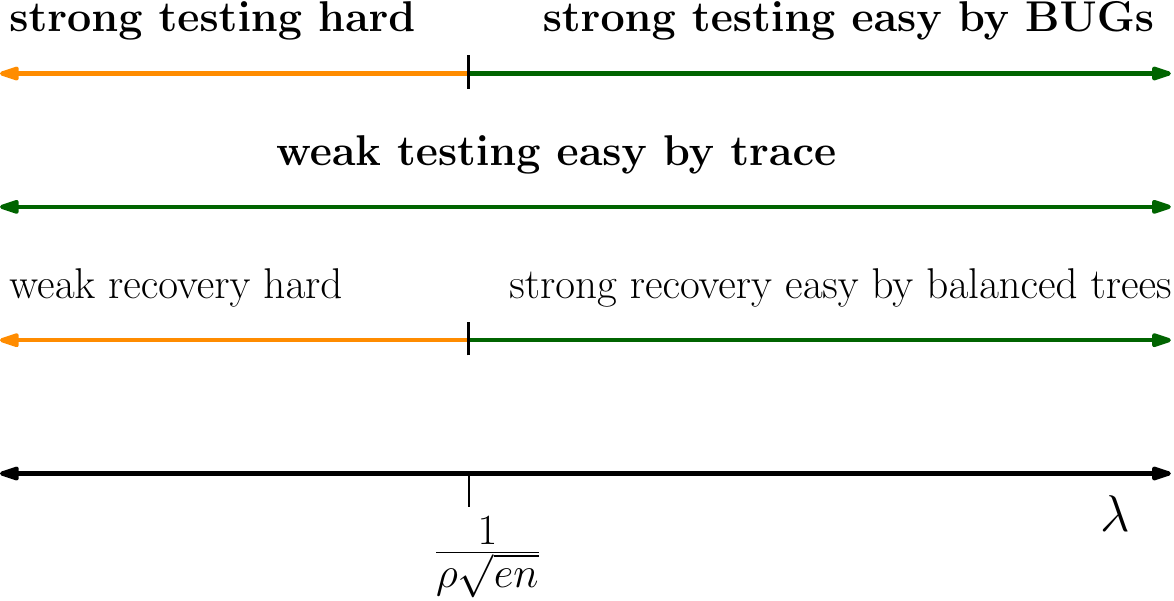}
\caption{Schematic phase diagram for testing between $\ell$ and $\ell'$ communities in the planted submatrix model. Weak testing is low-degree hard for $\lambda=o((\rho\sqrt n)^{-1})$, while strong testing has sharp threshold $(\rho\sqrt{en})^{-1}$, matching the sharp low-degree recovery threshold of~\cite{sohn2025sharp}. Bold statements are proved here.}
    \label{fig:phases_PSM} 
\end{figure}
\begin{figure}[!htb]
    \centering
\includegraphics[width=0.85\linewidth]{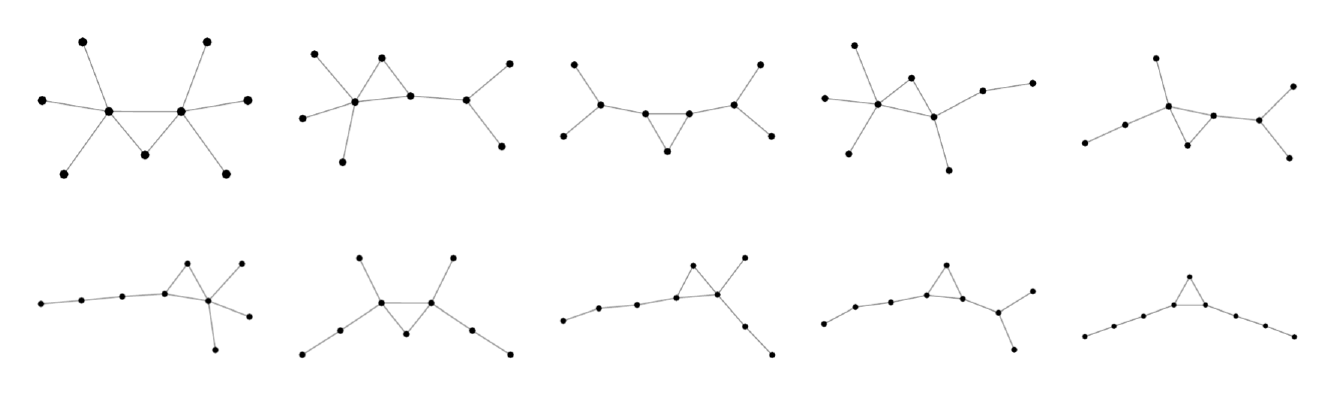}
    \caption{BUGs (Balanced Unicyclic Graphs). A depiction of the non-isomorphic graphs in~$\mathcal{U}_3$.}
    \label{fig:BUGs3}
\end{figure}
\subsubsection{Framework and proof ideas}\label{sec:contributions-framework}

In the usual planted-vs-null setting, the null distribution is a product measure, so the denominator in the low-degree advantage can be controlled by an orthogonal expansion directly in the observed variables. This fails for planted-vs-planted testing: both hypotheses contain latent structure, and the reference law is non-product in the observations. We address this non-product structure in three steps.

First, we pass to an extended space containing the latent variables and the underlying independent noise. This is motivated by the latent-variable orthogonalization method used for sharp low-degree recovery in~\cite{sohn2025sharp}.  At a high level, Bessel's inequality reduces the lower-bound problem to constructing a certificate $u$ that satisfies certain linear equations $u^\top M=c^\top$, and this yields the bound
\[ \Adv_{\le D}^2(\PP,\QQ)\le \|u\|_2^2 .\]

Second, for conditionally independent planted models, as is the case here, we observe that moments factor over connected components. Component consistency reduces verification of the certificate equations to connected graphs, while the norm bound involves the full graph-indexed sum.

Third, the certificate can be pruned to the graph structures that genuinely distinguish the two planted laws. In the applications below, the polynomial-basis indices can be viewed as graphs: an index $\alpha$ records the observed entries in the corresponding basis polynomial, and we identify it with the graph whose edge multiset contains those entries. The pruning principle itself is problem-independent: an indexing graph is \emph{good} if one of its nonempty subgraphs indexes a basis polynomial whose expectation differs under~$\PP$ and~$\QQ$, and \emph{bad} otherwise; see \Cref{def:good_graphs}. We show that the certificate can be supported only on good indexing graphs; see \Cref{lem:reduction_good}. The model-specific work is to identify the good graph structures and bound the corresponding certificate norm.

The steps outlined above are used to prove the low-degree lower bounds; the same moment comparison also guides the construction of the upper-bound
statistics. For PSM and PDS, tree components have identical moments under the $\ell$- and $\ell'$-planted laws, so the first nonzero moment differences arise at the unicyclic level. Accordingly, the strong-testing upper bounds use balanced unicyclic graph statistics; for weak testing, simpler statistics suffice.

\subsection{Related Work}
\paragraph*{Planted-vs-planted testing}
Testing between two models that both contain planted structure, also termed \emph{complex testing}, has been studied in~\cite{rush2023easier,coloring-clique,fourier-geo,carpentier2025low}.
The techniques used in~\cite{carpentier2025low} are based on constructing an almost orthogonal basis for the probability space~$Y$. In contrast, we work with an orthonormal set on an extended space. As the authors of~\cite{carpentier2025low} state, one of the advantages of their method of an almost orthogonal basis is that, while possibly losing a little precision, it is more tractable than finding and calculating with an exact orthonormal set of vectors. Part of our contribution is to show that it is feasible to work with an orthonormal set of vectors in the extended space for a nontrivial planted-vs-planted testing problem.

\paragraph*{Planted submatrix and planted dense subgraph}
The planted submatrix and planted dense subgraph models have both been extensively studied as canonical examples of high-dimensional planted-structure problems. For planted submatrix, information-theoretic limits are well understood in several regimes; see, e.g., \cite{butucea2015sharp,ButuceaIngster2013SparseSubmatrixDetection,rotenberg2024planted}. Computational aspects have often been studied in the planted-vs-null testing formulation, with a pure-noise null and an alternative containing a single planted submatrix \cite{ma2015computational,brennan2018reducibility}. For recovery, the best-known polynomial-time algorithm is a message-passing algorithm, which works up to the sharp threshold~\cite{submatrix-message-passing}. For the planted dense subgraph problem, prior work includes statistical characterizations \cite{arias2013community,verzelen2015community,chen2016statistical}, computational lower bounds for testing \cite{hajek2015computational}, and algorithmic achievability results \cite{bhaskara2010detecting,ames2015guaranteed,chen2016statistical,montanari2015finding}. Low-degree lower bounds for estimation in both models were initiated in \cite{schramm2022computational}, which obtained thresholds up to logarithmic factors, and were sharpened in \cite{sohn2025sharp}.

\paragraph*{Relation to low-degree recovery}
The low-degree framework has been widely used to study detection in planted problems and its statistical--computational tradeoffs; see, e.g., the survey~\cite{wein2025computational}. More recently, it has been extended to recovery, with thresholds obtained up to logarithmic factors in~\cite{schramm2022computational} and sharp thresholds obtained via refined orthogonalization and pruning arguments in~\cite{sohn2025sharp}. Another sharp low-degree lower bound was recently obtained by~\cite{optimal-bot} for root reconstruction in broadcasting on trees,
where the model and techniques differ from those for the planted matrix and graph
recovery problems considered in~\cite{sohn2025sharp}.
By comparison, \emph{testing between two planted models} (planted-vs-planted) has received much less
attention, despite being a natural way to interpolate between detection and recovery.
 A formal connection has been proved between recovery or estimation of the planted structure and appropriately chosen planted-vs-planted testing tasks; see Section~2.3 and Proposition~2.6 of~\cite{rush2023easier}. As noted there for the community-counting problem, however, this proposition does not clearly produce an equivalent recovery problem: the two hypotheses correspond to different latent signal spaces, namely $\ell$- and $\ell'$-community assignments, and the corresponding signal distributions have disjoint supports. Thus, to our knowledge, it would not have been possible to map the testing problems considered in this paper to a recovery problem and then apply the methods in~\cite{sohn2025sharp}. The sharp constant instead requires a direct analysis of the planted-vs-planted advantage. We also point out a recent line of work~\cite{alg-contig,DHSS-recovery,alg-contig-2} that builds some new connections between detection and recovery, though not directly related to the planted-vs-planted tasks considered here.  

 \subsection{Organization}
The rest of the paper is organized as follows. Section~\ref{sec:framework} develops the certificate framework used to prove low-degree lower bounds for planted-vs-planted testing. In Section~\ref{sec:psm-testing}, we prove the strong and weak testing results for the planted submatrix model: the lower bounds use the general certificate framework, while the strong upper bound is given by the BUG polynomial and the weak upper bound by the trace statistic. In Section~\ref{sec:pds-testing}, we prove the corresponding results for the planted dense subgraph model: the lower bounds use the Bernoulli-edge implementation of the certificate framework, while the strong upper bound is given by the BUG statistic together with a thinning reduction and the weak upper bound by a signed triangle statistic.

\section{Low-Degree Framework for Planted-vs-Planted Testing}\label{sec:framework}
 The main obstacle in planted-vs-planted testing is that both hypotheses are planted, so the reference law in the second moment is non-product on the observation space, and the usual orthogonal expansion in the observed variables $Y$ is unavailable. To address this difficulty, we expand against orthonormal functions of the underlying independent variables that generate $Y$, rather than against functions of $Y$ alone. This extends the latent-variable approach of~\cite{sohn2025sharp} to the planted-vs-planted setting.
\subsection{Low-Degree Testing Criteria}
We first recall the low-degree criteria used throughout the paper. Strong and weak separation formalize the success of a specific low-degree statistic, whereas the degree-$D$ advantage quantifies the best possible performance over all degree-$D$ polynomials and is used to prove low-degree hardness.
\subsubsection{Low-Degree Separation}
\begin{definition}
Let $\PP_n$ and $\QQ_n$ be distributions on $\mathbb R^N$, for some $N=N_n$. A degree-$D$ test is a multivariate polynomial $f_n: \mathbb R^N\to\mathbb R$ of degree at most $D$. Such a test $f_n$ strongly separates $\PP_n$ and
$\QQ_n$ if
\[
\sqrt{\max\{\mathrm{Var}_{\PP_n}(f_n),\mathrm{Var}_{\QQ_n}(f_n)\}}
=o\left(\left|\E_{\PP_n}[f_n]-\E_{\QQ_n}[f_n]\right|\right),
\]
and weakly separates $\PP_n$ and $\QQ_n$ if
\[
\sqrt{\max\{\mathrm{Var}_{\PP_n}(f_n),\mathrm{Var}_{\QQ_n}(f_n)\}}
=O\left(\left|\E_{\PP_n}[f_n]-\E_{\QQ_n}[f_n]\right|\right).
\]
\end{definition}
Specifically, strong separation implies vanishing type-I and type-II errors after thresholding $f$ (by Chebyshev's inequality), whereas weak separation implies nontrivial testing power; see
\cite{bandeira2022franz,coja2022statistical} 
and references therein.

\subsubsection{Low-Degree Hardness}
For lower bounds, we control the degree-$D$ advantage defined in~\eqref{eq.adv_def}. Such a bound rules out strong or weak separation by degree-$D$ polynomials, according to the following lemma from \cite{bandeira2022franz,coja2022statistical}.
\begin{lemma}\label{lem:separation-weak-strong}
Fix a sequence $D=D_n$. Let $\Adv_{\leq D}(\PP,\QQ)$ be as defined in \eqref{eq.adv_def}.
    \begin{itemize}
        \item If $\Adv_{\leq D}(\PP,\QQ)=O(1)$ then no degree-$D$ test strongly separates $\PP$ and $\QQ$.
        \item If $\Adv_{\leq D}(\PP,\QQ)=1+o(1)$ then no degree-$D$ test weakly separates $\PP$ and $\QQ$.
    \end{itemize}
\end{lemma}
A standard heuristic is to take the polynomial degree as a proxy for runtime, where failure of all logarithmic-degree polynomials suggests hardness for all polynomial-time algorithms. To provide low-degree evidence for hardness of strong, respectively, weak testing, we upper-bound $\Adv_{\le D}$ for degrees $D=\omega(\log n)$ in the
parameter regimes under consideration. This rules out strong, respectively weak,
separation by any polynomial of degree at most $D$. 

\subsection{Notation}
We will use elements $\alpha\in\{0,1\}^S$ or
$\alpha\in\mathbb N^S$ to index polynomials, where $S$ is finite and
$\mathbb N=\{0,1,2,\ldots\}$. For $\alpha,\beta\in\mathbb N^S$, define $|\alpha|=\sum_{k\in S}\alpha_k$,
$\alpha!=\prod_{k\in S}\alpha_k!$,
$\binom{\alpha}{\beta}=\prod_{k\in S}\binom{\alpha_k}{\beta_k}$,
and write $\beta\le \alpha$, which means $\beta \subseteq \alpha$, for coordinatewise inequality. For $X=(X_k)_{k\in S}$, let $X^\alpha=\prod_{k\in S}X_k^{\alpha_k}$. Typical choices of $S$ are $[n]=\{1,\dots,n\}$ or a set of pairs $(i,j)$
with $1\le i<j\le n$ (or $1\le i\le j\le n$, allowing loops), which we will abbreviate as $\binom{[n]}{2}$ or $\left(\!\!\binom{[n]}{2}\!\!\right)$, i.e., the \emph{multichoose} notation. Under this
identification, $\alpha\in\{0,1\}^{\binom{[n]}{2}}$ can be viewed as a simple graph on vertex set $[n]$, while
$\alpha\in\mathbb N^{\left(\!\!\binom{[n]}{2}\!\!\right)}$ can be viewed as a multigraph on vertex set $[n]$. This allows us to treat $\alpha$ as a graph and refer to its graph-theoretic properties.
For a graph or multigraph $\alpha$, we write $V(\alpha)\subseteq[n]$ for
its set of non-isolated vertices, $E(\alpha)$ for its edge multiset,
$\mathcal{C}(\alpha)$ for its set of connected components, and $C:=|\mathcal{C}(\alpha)|$ for the number of connected components.
We call $\alpha$ \emph{connected} if $V(\alpha)$ is connected by paths, and
declare the empty graph to be connected. We identify the empty graph with
both $0$ and $\varnothing$, and set $|\mathcal C(\varnothing)|=0$. We write
$\alpha\triangle\beta$ for the symmetric difference between $\alpha$ and $\beta$. Finally,
$\alpha\backslash\beta$ denotes the graph obtained from $\alpha$ by deleting
the edges in $\beta$ and then removing any isolated vertices.

\subsection{Connection to Low-Degree Recovery}\label{sec:setup-recovery-testing}
For a scalar recovery task with observation $Y \in \mathbb{R}^N$ or $Y\in \{0,1\}^N$ and latent functional $g(X)$ to be estimated, a proxy for recoverability is the degree-$D$ correlation \cite{schramm2022computational}:
\begin{equation*}\label{eq:corr_def}
\Corr_{\le D}(g(X);Y) := \sup_{\deg(f)\le D} \frac{\mathbb{E}[f(Y)\,g(X)]}{\sqrt{\mathbb{E}[f(Y)^2]\mathbb{E}[g(X)^2]}}.
\end{equation*}

Using the connection between recovery and planted-vs-planted testing,
\cite{rush2023easier} transferred this perspective to the degree-$D$ advantage
$\Adv_{\leq D}$ and obtained coarse thresholds for planted-vs-planted testing. Our work sharpens this picture: by adapting and extending the pruning and orthogonalization techniques of \cite{sohn2025sharp} to the planted-vs-planted
setting, we obtain sharp low-degree thresholds for distinguishing between planted distributions. We now establish the linear-algebraic framework underlying our upper bounds on the degree-$D$ testing advantage.
\subsection{Linear-algebraic certificate for the degree-$D$ advantage}\label{subsec:u-certificate-general}
In this section, we reduce the task of bounding $\Adv_{\le D}$ (see \eqref{eq.adv_def}) to checking a linear certificate. We expand degree-$D$ polynomials in a (multi)graph-indexed basis $\phi$, lower bound the denominator using an orthonormal collection $\psi$ on the extended latent-variable space,  and derive a sufficient linear condition whose solution controls $\Adv_{\le D}$.

\paragraph*{Setup} Fix a basis $\{\phi_\alpha(Y)\}_{\alpha\in\mathcal{I}}$ for the space of polynomials in $Y$ of degree at most $D$, where $\mathcal I$ is an index set that we identify with the relevant class of graphs (or multigraphs, depending on the model). We represent any candidate statistic $f$ as:
\begin{equation}\label{eq:f_expansion}
f(Y)=\sum_{\alpha\in\cI}\hat f_\alpha\,\phi_\alpha(Y),
\qquad \hat f=(\hat f_\alpha)_{\alpha\in\cI}\in\R^{\cI}.
\end{equation}
For the numerator, define the mean vector $c\in\R^{\cI}$ by
\begin{equation}\label{eq:c_def}
c_\alpha := \E_\PP[\phi_\alpha(Y)].
\end{equation}
Then, $\E_\PP[f(Y)]=c^\top \hat f$. The remaining task, which is the most challenging, is to lower-bound the denominator $\E_\QQ[f(Y)^2]$. In the models considered here, $Y$ is generated from an underlying collection of independent variables $W$ under the reference law. We therefore expand on this underlying space rather than directly on the observation space. Adapting the orthogonalization strategy from the low-degree recovery analysis of~\cite{sohn2025sharp}, we choose an orthonormal collection $\{\psi_\beta(W)\}_{\beta\in\mathcal J}$ indexed by a set $\mathcal J$ suited to the variables $W$. We will use Bessel's inequality in this extended space: if $H$ is a Hilbert space and $(e_k)$ is an orthonormal collection in $H$, then $\sum_{k=1}^{\infty} |\langle x, e_{k} \rangle|^{2} \le \|x\|^{2}$ for all $x \in H$. Applying Bessel's inequality to this orthonormal collection in $L^2(\QQ)$ gives
\begin{equation}\label{eq:bessel}
\E_\QQ[f(Y)^2] \ge \sum_{\beta\in\cJ}\E_\QQ[f(Y)\psi_\beta(W)]^2.
\end{equation}
Next, define the matrix $M\in\R^{\cJ\times\cI}$ by
\begin{equation}\label{eq:M_def}
M_{\beta\alpha}:=\E_\QQ[\phi_\alpha(Y)\psi_\beta(W)].
\end{equation}
Using \eqref{eq:f_expansion}, the lower bound \eqref{eq:bessel} can be written as
$\E_\QQ[f(Y)^2]\ge \|M\hat f\|_2^2$. The following proposition provides a sufficient condition for bounding $\Adv_{\le D}(\PP,\QQ)$ via a linear constraint on an auxiliary vector $u$.
\begin{proposition}\label{prop:adv_bound_general}Let $\PP$ and $\QQ$ be distributions on $\R^N$. Let $\{\phi_\alpha(Y)\}_{\alpha\in\mathcal{I}}$ be a basis for $\mathbb R[Y]_{\le D}$.  Suppose that, under $\QQ$, the observation $Y$ is coupled with auxiliary variables $W$, and let the set $\{\psi_\beta(W)\}_{\beta\in\mathcal J}$ be orthonormal in $L^2(\QQ)$. Suppose that, for $M$ and $c$ defined in~\eqref{eq:M_def} and~\eqref{eq:c_def}, there exists $u=(u_\beta)_{\beta\in\mathcal J}$ satisfying  
\begin{equation}\label{eq:u_constraint_gen}
\sum_{\beta \in \cJ} u_\beta M_{\beta \alpha} = c_\alpha
\end{equation}
for all $\alpha \in \cI$.
Then, $$\Adv_{\leq D}^2(\PP,\QQ) \leq \sum_{\beta \in \cJ} u_\beta^2.$$
\end{proposition}
For discrete observation spaces such as $\{0,1\}^N$, a degree-$D$ test is a function on $\{0,1\}^N$ that can be represented by a polynomial on $\mathbb R^N$ of degree at most $D$. Equivalently, since $x_i^2=x_i$ on $\{0,1\}^N$, every function on $\{0,1\}^N$ has a unique multilinear representative, and the degree of the test is defined as the degree of this multilinear polynomial.
\begin{proof}[Proof of Proposition \ref{prop:adv_bound_general}]
 Given the basis $\{\phi_{\alpha}\}$ for $\mathbb R[Y]_{\leq D}$, expand $f(Y)$ as in \eqref{eq:f_expansion}.
Since \(c_\alpha = \E_\PP[\phi_\alpha(Y)]\) by \eqref{eq:c_def}, the numerator of $\Adv_{\leq D}$ becomes
\[
    \E_{\PP}[f(Y)]
    = \sum_{\alpha} \hat{f}_{\alpha} \,  \E_{\PP}[\phi_{\alpha}(Y)]
    = c^{\mathsf{T}} \hat{f}.
\]

By construction, $\{\psi_\beta\}$, is orthonormal in $L^2(\QQ)$, so $\langle \psi_\beta,\psi_{\beta'}\rangle_{L^2(\QQ)}
:=\E_\QQ[\psi_\beta\psi_{\beta'}]
=\mathds 1\{\beta=\beta'\}.$
Bessel's inequality then gives the lower bound on $\E_\QQ[f(Y)^2]$ stated in
\eqref{eq:bessel}. Substituting \eqref{eq:f_expansion} into \eqref{eq:bessel} gives
\[
    \E_{\QQ}[f(Y)^{2}] 
     \geq 
    \sum_{\beta} \left( \sum_{\alpha} \hat{f}_{\alpha} \, \E_{\QQ}\big[ \phi_{\alpha}(Y) \psi_{\beta}(W) \big] \right)^{2}
    = \| M \hat{f} \|^{2},
\]
where the matrix $M = (M_{\beta \alpha})_{\beta \in \mathcal{J}, \, \alpha \in \mathcal{I}}$ is as defined in \eqref{eq:M_def}.
Combining the expression for the numerator with the Bessel lower bound on the denominator yields
\[
    \Adv_{\leq D}(\PP,\QQ)
     \leq
    \sup_{\hat{f}} \frac{c^{\mathsf{T}} \hat{f}}{\|M \hat{f}\|}
    =\sup_{\hat{f}} \frac{u^{\mathsf{T}} M \hat{f}}{\|M \hat{f}\|}
     \leq 
    \sup_{\hat{f}} \frac{\|u\|  \|M \hat{f}\|}{\|M \hat{f}\|}
    = \|u\|.\] \qedhere
\end{proof}
In models where observations are independent conditional on latent variables, as is the case here, the relevant quantities $c$ and $M$ satisfy a component consistency
property. We choose $u$ to satisfy the same property. The linear certificate condition does not force this, but it is a convenient choice: it reduces the verification of
\eqref{eq:u_constraint_gen} to the case of connected $\alpha$; see \Cref{prop:adv_bound_connected}.

\paragraph*{Component consistency} Given $\alpha$ with connected components $\alpha_1,\ldots,\alpha_C$, we say $u$ is \emph{component-consistent} if $u_\varnothing=1$ and
$
u_\alpha=\prod_{i=1}^C u_{\alpha_i}$. We say $M$ is \emph{inclusive} and \emph{component-consistent} if $M_{\beta,\alpha}=0$ whenever $\beta \not\subseteq \alpha$ and for $\beta \subset \alpha$, 
$ M_{\beta,\alpha}=\prod_{i=1}^C M_{\beta_i,\alpha_i}$ where $\beta_i := \beta \cap \alpha_i$, i.e.,\ induced subgraphs of $\beta$ on connected components of $\alpha$. Note that $\beta_i$ might be disconnected. Lastly, we say $c$ is \emph{component-consistent} if for each $\alpha$, $c_\alpha = \prod_{i=1}^Cc_{\alpha_i}$. Note that requiring component consistency reduces the number of constraints to check, but not the expression for the upper bound on $\Adv_{\leq D}$.

\begin{corollary}\label{prop:adv_bound_connected}
Let $\PP,\QQ,\{\phi_\alpha\}_{\alpha\in\mathcal I}, \{\psi_\beta\}_{\beta\in\mathcal J}$ be as in Proposition~\ref{prop:adv_bound_general}. Assume, in addition, that $M$ and $c$ are component-consistent and there is a component-consistent vector $u = (u_\beta)_{\beta \in \cJ}$, satisfying  %
\begin{equation}\label{eq:u_constraint}
\sum_{\beta \in \cJ} u_\beta M_{\beta \alpha} = c_\alpha
\end{equation}
for all connected $\alpha \in \cI$. Then the same identity holds for
every $\alpha\in\mathcal I$, and consequently $\Adv_{\leq D}^2(\PP,\QQ) \leq \sum_{\beta \in \cJ} u_\beta^2$.
\end{corollary}
\begin{proof}
Let $\alpha \in \cI$ be the set of (multi)graphs with $|\alpha|\leq D$ and let $\mathcal{C}(\alpha) = \{\alpha_1, \dots, \alpha_C\}$ be the connected components of $\alpha$. Note that $M_{\beta, \alpha} = 0$ unless $\beta \subseteq \alpha$. Any $\beta \subseteq \alpha$ decomposes uniquely into disjoint induced subgraphs $\beta_i = \beta \cap \alpha_i$. By the component consistency of $u$ and $M$, we can now factor the matrix-vector product $(u^\top M)_\alpha$ as follows:
\begin{align*}
\sum_{\beta \in \cJ} u_\beta M_{\beta \alpha} 
&= \sum_{\beta_1 \subseteq \alpha_1} \dots \sum_{\beta_C \subseteq \alpha_C} \left( \prod_{i=1}^C u_{\beta_i} M_{\beta_i, \alpha_i} \right) = \prod_{i=1}^C \left( \sum_{\beta_i \subseteq \alpha_i} u_{\beta_i} M_{\beta_i, \alpha_i} \right).
\end{align*}
Since each $\alpha_i$ is connected, \eqref{eq:u_constraint} implies that each inner sum equals $c_{\alpha_i}$. Therefore,
\[
\sum_{\beta \in \cJ} u_\beta M_{\beta \alpha}
= \prod_{i=1}^C c_{\alpha_i}
= c_\alpha,
\]
where the last equality uses the component consistency of $c$. Hence, $c^\top = u^\top M$. The bound on $\Adv_{\le D}(\PP,\QQ)$ follows as in Proposition~\ref{prop:adv_bound_general}: using $c^\top=u^\top M$ and Cauchy--Schwarz gives
\[
\Adv_{\le D}(\PP,\QQ)\le \|u\|.
\]
Squaring yields the claim.
\end{proof}
\subsection{Excluding Bad Terms}\label{sec:bad_terms}

\noindent We next identify the informative graph indices for planted-vs-planted testing and show that the certificate can be supported only on these indices. This yields a testing-specific good--bad decomposition; the corresponding pruning step for low-degree recovery was developed in~\cite{sohn2025sharp}. Here, the criterion is model-independent: a graph index is good precisely when it contains a nonempty subgraph whose associated basis polynomial has different expectations under $\PP$ and $\QQ$. The formal definition is as follows.
\begin{definition}[Bad/good graphs for $\PP$ vs.\ $\QQ$]\label{def:good_graphs}
A nonempty connected graph $\alpha$ is \emph{bad} if
\[
\E_\PP[\phi_{\alpha'}(Y)]=\E_\QQ[\phi_{\alpha'}(Y)]
\qquad\text{for \emph{every} nonempty subgraph } \alpha'\subseteq \alpha,
\]
and \emph{good} otherwise (equivalently: there exists \emph{some} nonempty
$\alpha'\subseteq \alpha$ with $\E_\PP[\phi_{\alpha'}(Y)]\neq \E_\QQ[\phi_{\alpha'}(Y)]$).
The empty graph $\varnothing$ is good by convention. A (not necessarily connected)
graph $\alpha\in\mathcal I$ lies in $\goodI$ iff every connected
 component $\alpha_i$ is good; otherwise, $\alpha$ is bad.
\end{definition}
Similarly, for graphs in $\mathcal J$, we write $\beta\in\goodJ$ when $\beta\in\goodI$. In the model-specific applications, elements of $\mathcal J$ may carry additional labels; see Remark~\ref{rem:indexing-hatJ}.
\begin{remark}[Connected bad graphs have only bad subgraphs]\label{rem:downward_new}
If $\alpha$ is connected and bad, and $\tilde\alpha\subseteq\alpha$ is a nonempty subgraph, then $\tilde\alpha$ is also bad. Indeed, let $\tilde\alpha_1,\dots,\tilde\alpha_C$ be the connected components of $\tilde\alpha$. For each $i\in [C]$, every nonempty subgraph $\alpha'\subseteq\tilde\alpha_i$ is also a subgraph of $\alpha$. Since $\alpha$ is bad, this implies $\E_\PP[\phi_{\alpha'}(Y)]=\E_\QQ[\phi_{\alpha'}(Y)]$ for every such $\alpha'$. Hence, each connected component $\tilde\alpha_i$ is bad, and therefore $\tilde\alpha$ is bad.
\end{remark}

The following lemma, combined with Corollary~\ref{prop:adv_bound_connected}, reduces the verification of the constraint
\eqref{eq:u_constraint} from all $\alpha\in\mathcal I$ to connected
$\alpha\in\goodI$, provided $u$ is supported on $\goodJ$. This lemma provides the pruning step for our planted-vs-planted analysis, corresponding to the role played by Lemma~1.4 in the low-degree recovery framework of~\cite{sohn2025sharp}.

\begin{lemma}[Reduction to connected good $\alpha$]
\label{lem:reduction_good}
Assume the setup of Proposition~\ref{prop:adv_bound_general}, with basis
$\{\phi_\alpha\}_{\alpha\in\mathcal I}$ and orthonormal collection
$\{\psi_\beta\}_{\beta\in\mathcal J}$. Let $c$ and $M$ be defined
by~\eqref{eq:c_def} and~\eqref{eq:M_def}, respectively, and let the good sets be
as in Definition~\ref{def:good_graphs}. Assume that $c$ and $M$ are component consistent, that $M_{\beta,\alpha}=0$ whenever $\beta\not\subseteq\alpha$, and that $\psi_\varnothing\equiv1$. Let $u=(u_\beta)_{\beta\in\mathcal J}$ be component-consistent with $u_\varnothing=1$, and suppose
\begin{equation}\label{eq:u_zero_on_bad}
u_\beta = 0 \qquad \text{for every } \beta\notin \goodJ.
\end{equation}
If 
\begin{equation}\label{eq:u_connected_good}
\sum_{\beta\in\hat{\mathcal J}} u_\beta\, M_{\beta,\alpha} = c_\alpha
\qquad \text{for every connected } \alpha\in\goodI,
\end{equation}
then $\sum_{\beta\in\mathcal J} u_\beta\, M_{\beta,\alpha} = c_\alpha$ for every
$\alpha\in\mathcal I$. Therefore, $$\Adv^2_{\leq D}(\PP,\QQ) \leq \sum_{\beta\in\goodJ} u_\beta^2.$$
\end{lemma}

\begin{proof}
By Corollary~\ref{prop:adv_bound_connected}, it suffices to verify
$\sum_{\beta\in\mathcal J} u_\beta M_{\beta,\alpha} = c_\alpha$ for every
\emph{connected} $\alpha\in\mathcal I$. If $\alpha$ is \emph{connected} and \emph{good}, then it is already guaranteed to satisfy \eqref{eq:u_connected_good}. Now consider the case where $\alpha$ is \emph{connected} and \emph{bad}, and split the sum over $\beta$ into  $\beta$ good and $\beta$ bad:
\begin{equation}\label{eq:case2_split}
\sum_{\beta \in \mathcal J} u_\beta\, M_{\beta,\alpha}
=\sum_{\beta \in \goodJ} u_\beta\, M_{\beta,\alpha}
+\sum_{\beta \in \mathcal J \setminus \goodJ} u_\beta\, M_{\beta,\alpha}.
\end{equation}

Observe that the second sum vanishes by~\eqref{eq:u_zero_on_bad}. For the first sum, let $\beta\in\widehat{\mathcal J}$ have $M_{\beta,\alpha}\neq0$. By the assumption $M_{\beta,\alpha}=0$ unless $\beta\subseteq\alpha$, we must have $\beta\subseteq\alpha$. Since $\alpha$ is connected and bad, Remark~\ref{rem:downward_new} implies that every nonempty $\beta\subseteq\alpha$ is bad. Hence, the only possible good $\beta$ is $\beta=\varnothing$. Thus, for $\alpha$ connected and bad,
\begin{equation}\label{eq:after_split}
\sum_{\beta \in \mathcal J} u_\beta\, M_{\beta,\alpha}
= u_\varnothing\, M_{\varnothing,\alpha}= M_{\varnothing,\alpha}.
\end{equation}
Let $d_\alpha:=\E_\QQ[\phi_\alpha(Y)]$. Using the assumption $\psi_\varnothing = 1$, we obtain that $M_{\varnothing,\alpha} = d_\alpha$. Recall also that $u_\varnothing = 1$. Finally, since $\alpha$ is bad and connected, by definition, the expectation of $\alpha$ itself must be equal under $\PP$ and $\QQ$, i.e.,\ $c_\alpha = d_\alpha$. Thus, by~\eqref{eq:after_split}, for $\alpha$ connected and bad,
\[
\sum_{\beta \in \mathcal J} u_\beta\, M_{\beta,\alpha}
=d_\alpha
=c_\alpha,
\]
as required. Having verified the constraint for every connected $\alpha\in\mathcal I$,
Corollary~\ref{prop:adv_bound_connected} and~\eqref{eq:u_zero_on_bad} 
\[
\Adv_{\le D}^2(\PP,\QQ)\le \sum_{\beta\in\goodJ}u_\beta^2+\sum_{\beta\in\cJ\setminus\goodJ}u_\beta^2= \sum_{\beta\in\goodJ}u_\beta^2,
\]
which proves the result.
\end{proof}
Together, Proposition~\ref{prop:adv_bound_general},
Corollary~\ref{prop:adv_bound_connected}, and
Lemma~\ref{lem:reduction_good} reduce the low-degree lower-bound problem to constructing the certificate $u$ on connected good graph indices. Proposition~\ref{prop:adv_bound_general} accepts any $u$ satisfying
$u^\top M=c^\top$, so zeroing $u$ on bad indices is a choice rather than a necessity. This choice removes bad-index contributions from $\|u\|^2$ and, by~\Cref{prop:adv_bound_connected} together with
Remark~\ref{rem:downward_new}, makes the constraint automatic on connected bad indices. Indeed, for such $\alpha$, the only surviving contribution is the $\beta=\varnothing$ term, equal to
$M_{\varnothing,\alpha}=\E_\QQ[\phi_\alpha(Y)]=\E_\PP[\phi_\alpha(Y)]$. Therefore, the remaining
model-specific task is to construct a component-consistent vector $u$, supported
on $\goodJ$, satisfying \eqref{eq:u_connected_good} for every connected
$\alpha\in\goodI$. We carry this out in Sections~\ref{sec:psm-testing} and~\ref{sec:pds-testing} by an explicit recursion on $\alpha$.

\begin{remark}\label{rem:indexing-hatJ}
In the model-specific applications of Sections~\ref{sec:psm-testing} and~\ref{sec:pds-testing}, the good graph index set $\cJ$ is of the form
\[
\cJ=\{(\beta,\gamma):\beta\in\cI,\ \gamma \text{ a certain coloring of } \beta\}.
\]
Specifically, $\gamma$ records the signal structure associated with $\beta$.
Accordingly, the abstract $\cJ$-indexed quantities are written
$u_{\beta\gamma}$, $\psi_{\beta\gamma}$, and $M_{\beta\gamma,\alpha}$.
The corresponding good subset is
$
\goodJ=\{(\beta,\gamma)\in\cJ:\beta\in\goodI, \ \gamma \text{ a coloring of } \beta\}.
$
\end{remark}

\begin{remark}[Notation]
We write $\QQ_\Theta$ and $\PP_\Theta$ for the marginal laws of the latent label vector $\Theta$ under $\QQ$ and $\PP$. In both PSM and PDS, these marginals coincide; only the conditional observation law differs.
\end{remark}

\section{Planted Submatrix Model}\label{sec:psm-testing} We begin with strong testing, where we locate the sharp threshold and prove matching low-degree lower and upper bounds. We then turn to weak testing, where the sharp-threshold behavior is replaced by a smooth transition, whose scale we identify.
\begin{definition}[$\ell$-Planted Submatrix]\label{def:PSM}
Let $n,\ell \in \mathbb{N}$, $\rho \in (0,1)$,
and $\lambda > 0$. We define the $\ell$-planted submatrix model $\mathbb{P}_{\mathrm{PSM}}(n,\ell,\rho,\lambda)$ as follows.
The latent labels $\Theta = (\Theta_1,\dots,\Theta_n)$ are drawn independently from $\{0,1,\dots,\ell\}$, with distribution
\[
\mathrm{Pr}(\Theta_i = c) = \frac{\rho}{\ell}
\quad \text{for each } c \in [\ell],
\qquad
\mathrm{Pr}(\Theta_i = 0) = 1-\rho.
\]
Given $\Theta$, one observes a symmetric matrix $Y \in \mathbb{R}^{n\times n}$ with entries
\[
Y_{ij}
=
\ell\lambda\,\mathbf{1} \{\Theta_i=\Theta_j\in[\ell]\}+Z_{ij},
\qquad 1 \le i,j \le n,
\]
where $(Z_{ij})_{1 \le i \le j \le n}$ are i.i.d.\ $\mathrm{N}(0,1)$ random variables, and $Z_{ji}=Z_{ij}$ for all $i<j$.
\end{definition}

\subsection{Strong Testing}\label{sec:psm-strong-testing}
The following theorem gives matching low-degree bounds across the critical signal strength: below the threshold, the degree-$D$ advantage remains bounded, while above it, a degree-$D$ polynomial strongly separates the two planted laws.
\thmPSMstrong*

The degree-$D$ polynomial in part (ii) of \Cref{thm:strong_PSM} is based on counting occurrences of balanced unicyclic graphs (BUGs); see Figure~\ref{fig:BUGs3} and Equation~\eqref{def:BUG_f} for the definition. 

\noindent In what follows, we give proofs of the lower and upper bounds, respectively.

\subsubsection{Lower Bound}\label{sec:psm-lower-bound-strong}

We establish the tools required to prove part (i) of Theorem~\ref{thm:strong_PSM} by proceeding in four steps, as outlined in Section~\ref{sec:framework}. First, we choose the basis $\{\phi_\alpha\}_{\alpha \in \cI}$ and the orthonormal collection $\{\psi_{\beta\gamma}\}_{\beta\gamma \in \cJ}$ adapted to the Gaussian noise and the planted vertex indicators, and compute the quantities $c_\alpha,d_\alpha$, and $M_{\beta\gamma,\alpha}$. Second, these formulas identify which graphs are good. Third, we construct an explicit feasible vector $u$ supported on the set of good graphs $\goodI$ (see Lemma~\ref{lem:reduction_good}). The feasibility check is based on a cancellation over $\gamma\subseteq V(\beta)$. Finally, we bound $\|u\|^2$. This reduces to counting good multigraphs with a prescribed number of vertices, edges, and connected components. The assumed parameter regime ensures that these sums remain bounded, completing the proof of part~(i).

\subsubsection*{Setting $\phi,\psi$ and establishing properties}
Recall the definition of the $\ell$-planted submatrix model from Definition~\ref{def:PSM}. The first step is to choose the polynomial basis for the observation space. We take the degree-$D$ basis of $\mathbb R[Y]_{\le D}$ to be
\begin{equation}\label{eq:phi_def_mat}
\phi_\alpha(Y)=H_\alpha(Y), \qquad \alpha \in \mathcal I:=\Bigl\{\alpha \in \mathbb N^{\left(\!\!\binom{[n]}{2}\!\!\right)}:\ |\alpha|\le D\Bigr\},
\end{equation}
where $\{H_\alpha\}$ denotes the family of multivariate Hermite polynomials, which are orthogonal with respect to the Gaussian measure \cite{szeg1939orthogonal}. We use the orthonormal normalization of the Hermite polynomials, so that $\E[H_\alpha(Z)H_\beta(Z)]=\mathds{1}_{\alpha=\beta}$ for $Z\sim \mathrm N(0,1)$.

Next, we choose an orthonormal collection in the underlying random variables $W=(Z,\Theta)$. For the present planted-vs-planted problem, we use an extended-space orthogonalization tailored to the testing advantage, building on the latent-variable perspective from low-degree recovery~\cite{sohn2025sharp}. The collection need not be a basis; Bessel's inequality only requires an orthonormal family. Thus, one could work with a larger orthonormal system, for instance, by including additional functions of the latent signal variables. For the present testing problem, we show that the collection used below already supports a certificate satisfying $u^\top M=c^\top$ and therefore suffices to upper-bound $\Adv_{\le D}$. Define
\begin{align}
    &\psi_{\beta \gamma}(Z,\Theta)=H_\beta(Z) \left(\frac{\mathds{1}[\Theta \neq 0]-\rho}{\sqrt{\rho(1-\rho)}}\right)^{\gamma}= H_\beta(Z)
\prod_{i=1}^n
\left(\frac{\mathds 1[\Theta_i\neq0]-\rho}{\sqrt{\rho(1-\rho)}}\right)^{\gamma_i},\label{eq:psi_def_mat} \\
    & \notag \beta \gamma \in \mathcal J:=\{(\beta,\gamma):\beta \in\mathbb N^{\left(\!\!\binom{[n]}{2}\!\!\right)}, |\beta|\leq D, \gamma \in \{0,1\}^n\}.
\end{align}
 In particular, $\psi_{\varnothing\varnothing} \equiv 1$ since $H_\varnothing(Z) = 1$ and the second factor equals $1$ when $\gamma = \varnothing$. 
\begin{lemma}\label{lem:psm:cdM} With $\{\phi_{\alpha}\}_{\alpha \in \cI}$ and $\{\psi_{\beta\gamma}\}_{\beta\gamma \in \cJ}$ as defined in \eqref{eq:phi_def_mat} and \eqref{eq:psi_def_mat}, respectively, let
\[
c_\alpha:=\E_{\PP}[\phi_\alpha(Y)],
\qquad
d_\alpha:=\E_{\QQ}[\phi_\alpha(Y)],
\qquad
M_{\beta\gamma,\alpha}:=\E_{\QQ} \left[\phi_\alpha(Y)\psi_{\beta\gamma}(Z,\Theta)\right].
\]
Then,
\begin{align}
&c_\alpha=\frac{\ell'^{|\mathcal{C}(\alpha)|+|\alpha|-|V(\alpha)|}}{\sqrt{\alpha!}}\lambda^{|\alpha|}\rho^{|V(\alpha)|}
,\qquad d_\alpha=\frac{\ell^{|\mathcal{C}(\alpha)|+|\alpha|-|V(\alpha)|}}{\sqrt{\alpha!}}\lambda^{|\alpha|}\rho^{|V(\alpha)|}
\notag,
\\
 & M_{\beta\gamma,\alpha}=\mathds 1_{\beta\le\alpha}\,\mathds 1_{\gamma\subseteq V(\alpha\backslash\beta)} 
\sqrt{\frac{\beta!}{\alpha!}}
\binom{\alpha}{\beta} 
\ell^{\,|\alpha\backslash\beta|+|\mathcal{C}(\alpha\backslash\beta)|-|V(\alpha\backslash\beta)|} \lambda^{|\alpha\backslash\beta|}\,
\rho^{|V(\alpha\backslash\beta)|}\,
\,
 \bigg(\frac{1 - \rho}{\rho}\bigg)^{\frac{|\gamma|}{2}}
 \notag.
\end{align}
\end{lemma}
For the above choices of $\phi_\alpha$ and $\psi_{\beta\gamma}$, we use the
standard shift identity for normalized Hermite polynomials. Namely, for
$Z\sim N(0,1)$ and deterministic $x$,
\[
h_m(x+Z)=\sum_{r=0}^m \sqrt{\frac{r!}{m!}}\binom{m}{r}x^{m-r}h_r(Z).
\]
Applying this identity coordinatewise gives
\[
H_\alpha(X+Z)
=
\sum_{0\le \beta\le \alpha}
\sqrt{\frac{\beta!}{\alpha!}}\binom{\alpha}{\beta}
X^{\alpha-\beta}H_\beta(Z),
\]
which is the Hermite expansion also used in the latent-variable analysis of~\cite{sohn2025sharp}. Specifically,
\begin{align*}
    H_\alpha(Y)=\prod_{i\leq j}h_{\alpha_{ij}}(X_{ij}+Z_{ij})&=\prod_{i \leq j}\sum_{k=0}^{\alpha_{ij}}\sqrt{\frac{k!}{\alpha_{ij}!}}{ \binom{\alpha_{ij}}{k}}X_{ij}^{\alpha_{ij}-k}h_k(Z_{ij})\\
    &=\sum_{0 \leq \beta \leq \alpha}\sqrt{\frac{\beta!}{\alpha!}}{\binom{\alpha}{\beta}}X^{\alpha - \beta}H_\beta(Z).
\end{align*}
\begin{proof}[Proof of Lemma \ref{lem:psm:cdM}]
Given the Hermite expansion above, we now compute $c_\alpha$, $d_\alpha$, and $M_{\beta\gamma,\alpha}$ as follows
\begin{align*}
    \nonumber c_\alpha&=\E_\PP[H_\alpha(Y)]=\frac{1}{\sqrt{\alpha!}}\E_\PP\left[X^{\alpha}\right]=\frac{\ell'^{|\mathcal{C}(\alpha)|+|\alpha|-|V(\alpha)|}}{\sqrt{\alpha!}}\lambda^{|\alpha|}\rho^{|V(\alpha)|},  \\   
    d_{\alpha}&=\frac{\ell^{|\mathcal{C}(\alpha)|+|\alpha|-|V(\alpha)|}}{\sqrt{\alpha!}}\lambda^{|\alpha|}\rho^{|V(\alpha)|},\\
M_{\beta\gamma,\alpha}&=\E_\QQ\left[H_\alpha(Y)\psi_{\beta\gamma}(Z,\Theta)\right]\\
\nonumber &=\mathds 1_{\beta\le\alpha}\,\mathds 1_{\gamma\subseteq V(\alpha \backslash \beta)} 
\sqrt{\frac{\beta!}{\alpha!}}
\binom{\alpha}{\beta} \E_\QQ\left[X^{\alpha \backslash \beta}\left(\frac{\mathds 1 [\Theta\neq 0]-\rho}{\sqrt{\rho(1-\rho)}}\right)^\gamma\right].
\end{align*}
To complete the proof, note that
\begin{align*}
\E_\QQ\bigg[X^{\alpha \backslash \beta} & \left(\frac{\mathds 1 [\Theta\neq 0]-\rho}{\sqrt{\rho(1-\rho)}}\right)^\gamma\bigg]\\
 & =
 \ell^{\,|\alpha \backslash \beta|+|\mathcal{C}(\alpha \backslash \beta)|-|V(\alpha \backslash \beta)|}\lambda^{|\alpha \backslash \beta|}\,
\rho^{|V(\alpha \backslash \beta)\backslash\gamma|}\, (\rho(1-\rho))^{\frac{|\gamma|}{2}}
\\
&=
\ell^{\,|\alpha \backslash \beta|+|\mathcal{C}(\alpha \backslash \beta)|-|V(\alpha \backslash \beta)|} \lambda^{|\alpha \backslash \beta|}\,
\rho^{|V(\alpha \backslash \beta)|}\,
\,
 \left(\frac{1 - \rho}{\rho}\right)^{\frac{|\gamma|}{2}}. \qquad \qquad \qquad \qquad \qedhere
\end{align*} 
\end{proof}
 \subsubsection*{Excluding bad terms}
By Lemma~\ref{lem:psm:cdM}, 
the distinction between good and bad graphs in the sense of Definition~\ref{def:good_graphs} is governed entirely by the exponent
$|\mathcal{C}(\alpha)|+|\alpha|-|V(\alpha)|$.
Indeed, $c_\alpha$ and $d_\alpha$ have the same dependence on $\lambda$ and $\rho$ and differ by replacing $\ell'$ with $\ell$ in this exponent. The following lemma identifies the connected graphs that are good in the sense of Definition~\ref{def:good_graphs}.

\begin{lemma}\label{lem:tree-uninformative}
A connected multigraph $\alpha$ has $c_\alpha\neq d_\alpha$ if and only if $|\alpha|\ge |V(\alpha)|.$
\end{lemma}
\noindent Equivalently, $\alpha$ contains a cycle in the multigraph sense. In particular, loops and parallel-edge cycles count as cycles.
\begin{proof}[Proof of Lemma \ref{lem:tree-uninformative}]
By Lemma~\ref{lem:psm:cdM}, for connected $\alpha$ on $|V(\alpha)|$ vertices with $|\alpha|$ edges,
$c_{\alpha}/d_{\alpha} = (\ell'/\ell)^{1+|\alpha|-|V(\alpha)|}$. Since $\ell'\neq \ell$, this ratio equals $1$ if and only if $1+|\alpha|-|V(\alpha)|=0$, i.e., if and only if $|\alpha|=|V(\alpha)|-1$. For a connected multigraph, this is exactly the tree case. Hence, $c_\alpha\neq d_\alpha$ if and only if $|\alpha|\ge |V(\alpha)|$, equivalently if and only if $\alpha$ contains a cycle in the multigraph sense.
\end{proof}

\noindent
Thus, connected bad indices are precisely trees. Since every nonempty subgraph of a tree is again a forest, Remark~\ref{rem:downward_new} gives the following characterization: $\alpha\in\goodI$ exactly when none of its connected components is a tree, or equivalently, when every connected component of $\alpha$ contains a cycle. Accordingly, $\beta\gamma\in\goodJ$ if and only if $\beta\in\goodI$ and $\gamma\subseteq V(\beta)$.
\subsubsection*{Constructing $u$} 
By Lemma~\ref{lem:reduction_good}, it suffices to construct a
component-consistent vector $u$ supported on $\goodJ$ with
$u_{\varnothing\varnothing}=1$, such that \eqref{eq:u_constraint} holds for every
connected $\alpha\in\goodI$. We now construct this certificate and then
bound its squared norm.
\begin{proposition}\label{prop:u-PSM}
Define  $u_{\varnothing\varnothing}=1$. For every nonempty $\alpha\in\goodI$ with connected components
$\alpha_1,\ldots,\alpha_C$, and every $\gamma\subseteq V(\alpha)$, let
\[
u_{\alpha\gamma}
=
\left(-\sqrt{\frac{\rho}{1-\rho}}\right)^{|\gamma|}
\prod_{i=1}^C (c_{\alpha_i}-d_{\alpha_i}).
\]
Set $u_{\alpha\gamma}=0$ whenever $\alpha \gamma\notin\goodJ$. Then $u$ is supported on $\goodJ$, is component-consistent, and satisfies
\[
\sum_{\beta\gamma\in\goodJ}
u_{\beta\gamma}M_{\beta\gamma,\alpha}
=
c_\alpha \qquad \text{for all connected} \ \ \alpha \in \goodI.
\]
\end{proposition}
The proof uses the following cancellation identity.
\begin{lemma}\label{lem:psm-cancellation}Let $\alpha \in \goodI$ be connected and nonempty, and let $\beta \subseteq \alpha$ with $\beta \in \goodI$. Then
\[\sum_{\gamma \subseteq V(\beta)}\left(-\sqrt{\frac{\rho}{1-\rho}}\right)^{|\gamma|} M_{\beta\gamma,\alpha}=\begin{cases}
d_\alpha, & \beta=\varnothing,\\
1, & \beta=\alpha,\\
0, & \varnothing\neq\beta\lneq\alpha.
\end{cases}\]
\end{lemma}
\begin{proof}
    First, suppose $\beta=\varnothing$. Since $\gamma \subseteq V(\beta)$ and $V(\beta)=\varnothing$, the only possible choice is $\gamma=\varnothing$. Since $\psi_{\varnothing \varnothing}=1$, we get 
    \[1\cdot M_{\varnothing\varnothing,\alpha}=\E_\QQ[H_\alpha(Y)]=d_\alpha.\]
    This proves the first case. Next, let $\beta=\alpha$.  Orthogonality in the Hermite basis gives
\[
M_{\alpha\gamma,\alpha}
=
\E_\QQ\left[\prod_{i=1}^n
\left(\frac{\mathds 1\{\Theta_i\neq0\}-\rho}{\sqrt{\rho(1-\rho)}}\right)^{\gamma_i}\right].
\]
This equals $1$ if $\gamma=\varnothing$ and $0$ otherwise, since the factors are independent and centered. Hence,
\[
\sum_{\gamma\subseteq V(\alpha)}
\left(-\sqrt{\frac{\rho}{1-\rho}}\right)^{|\gamma|}
M_{\alpha\gamma,\alpha}=1.
\]
Finally, suppose $\varnothing\neq\beta\subsetneq\alpha$. Since
$M_{\beta\gamma,\alpha}=0$ unless $\gamma\subseteq V(\alpha\backslash\beta)$, the sum over $\gamma\subseteq V(\beta)$ is effectively a sum over $\gamma\subseteq V(\alpha\backslash\beta)\cap V(\beta)$. For such $\gamma$, Lemma~\ref{lem:psm:cdM} gives $M_{\beta\gamma,\alpha}=M_{\beta\varnothing,\alpha}
\left(\frac{1-\rho}{\rho}\right)^{|\gamma|/2}.$
Hence,
\[
\sum_{\gamma\subseteq V(\beta)}
\left(-\sqrt{\frac{\rho}{1-\rho}}\right)^{|\gamma|}
M_{\beta\gamma,\alpha}=
M_{\beta\varnothing,\alpha}
\sum_{\gamma\subseteq V(\alpha\backslash\beta)\cap V(\beta)}
(-1)^{|\gamma|}.
\]
Since $\alpha$ is connected and $\varnothing\neq\beta\subsetneq\alpha$, the
edge sets of $\beta$ and $\alpha\setminus\beta$ cannot be vertex-disjoint;
otherwise, $\alpha$ would be disconnected. Thus, some edge of
$\alpha\setminus\beta$ is incident to a vertex of $V(\beta)$, so $S:=V(\alpha\setminus\beta)\cap V(\beta)\neq\varnothing.$
Thus,
$\sum_{\gamma\subseteq S}(-1)^{|\gamma|}=0,$
which proves the third case.
\end{proof}
\noindent We now prove Proposition~\ref{prop:u-PSM}.
\begin{proof}[Proof of Proposition~\ref{prop:u-PSM}]
Fix a connected nonempty $\alpha\in \goodI$. Since $M_{\beta\gamma,\alpha}=0$ unless $\beta\subseteq\alpha$, only such $\beta$ contribute to the sum. We split these contributions into the cases $\beta=\varnothing$, $\beta=\alpha$, and $\varnothing\neq\beta\subsetneq\alpha$. By Lemma~\ref{lem:psm-cancellation}, and by the definition of $u$ for connected $\alpha$,
\[\sum_{\beta\gamma\in\widehat{\mathcal J}} u_{\beta\gamma}M_{\beta\gamma,\alpha}=u_{\varnothing\varnothing}d_\alpha+(c_\alpha-d_\alpha)\cdot 1+0=d_\alpha+c_\alpha-d_\alpha=c_\alpha. \qedhere
\]
\end{proof}

It remains to control the norm of $u$. This reduces to enumerating graphs in
$\goodI$ with prescribed numbers of vertices, edges, and connected components; the required bound is given next.

\begin{lemma}\label{lemma:nr-multigraphs}
  Let \(v,d,C\ge 1\), and set \(k:=d-v\). The number of multigraphs $\alpha\in\widehat{\mathcal I}$ with $v$ vertices, $d$ edges, and $C$ connected components is at most
    \begin{equation*}
\frac{ n^v}{C!}  \binom{ v^2+k}{k}
\sum_{\substack{v_1+\cdots+v_C=v \\ v_i \ge 1}}
\prod_{i=1}^C
\frac{v_i^{ v_i}}{v_i!}.
\end{equation*}
\end{lemma}
\begin{proof}
    See Supplementary Material S.1.1.
\end{proof}
\subsubsection*{Putting it all together}
We now combine the $u$ construction with \Cref{lemma:nr-multigraphs} to prove part~(i) of Theorem~\ref{thm:strong_PSM}.

\begin{proof}[Proof of Theorem~\ref{thm:strong_PSM} \textnormal{(i)}]
By~\Cref{lem:reduction_good} applied to the certificate from Proposition~\ref{prop:u-PSM}, noting that $u_{\varnothing\varnothing}=1$, we have
 \begin{align}
\notag \Adv_{\le D}^2(\PP,\QQ)
\notag \leq \sum_{\alpha\gamma\in\goodJ} u_{\alpha\gamma}^2&=1+
\sum_{\alpha\gamma\in\widehat{\mathcal J}:\alpha\neq\varnothing}
u_{\alpha\gamma}^2 \\
\notag &=
1+
\sum_{\alpha\in\widehat{\mathcal I}:\alpha\neq\varnothing}
\left(\prod_{i=1}^C
(c_{\alpha_i}-d_{\alpha_i})^2\right)
\sum_{\gamma\subseteq V(\alpha)}
\left(\frac{\rho}{1-\rho}\right)^{|\gamma|} \\
&=
1+
\sum_{\alpha\in\widehat{\mathcal I}:\alpha\neq\varnothing}
\left(\prod_{i=1}^C
(c_{\alpha_i}-d_{\alpha_i})^2\right)
\left(1+\frac{\rho}{1-\rho}\right)^{|V(\alpha)|}.\label{eq:adv-psm}
\end{align}

\noindent Let $L:=2\max\{\ell,\ell'\}$. For every connected good component $\alpha_i$, for $i\in [C]$,
\[
c_{\alpha_i}-d_{\alpha_i}
=
\frac{
\ell'^{1+|\alpha_i|-|V(\alpha_i)|}
-
\ell^{1+|\alpha_i|-|V(\alpha_i)|}
}{\sqrt{\alpha_i!}}
\lambda^{|\alpha_i|}\rho^{|V(\alpha_i)|},
\]
and therefore
\[
(c_{\alpha_i}-d_{\alpha_i})^2
\leq
L^{2(1+|\alpha_i|-|V(\alpha_i)|)}
\frac{\lambda^{2|\alpha_i|}\rho^{2|V(\alpha_i)|}}{\alpha_i!}.
\]
Thus, if $\alpha$ has connected components
$\mathcal{C}(\alpha)=\{\alpha_1,\dots,\alpha_C\}$ with $C=|\mathcal C (\alpha)|$, then
\begin{equation}\label{eq:bound-c-d-by-L}
\prod_{\alpha_i\in\mathcal{C}(\alpha)}(c_{\alpha_i}-d_{\alpha_i})^2
\leq
L^{2(C+|\alpha|-|V(\alpha)|)}
\frac{\lambda^{2|\alpha|}\rho^{2|V(\alpha)|}}{\alpha!}.
\end{equation}

\noindent Set $$\tilde{\rho}:=\rho\sqrt{1+\frac{\rho}{1-\rho}}=\frac{\rho}{\sqrt{1-\rho}}.$$ Substituting \eqref{eq:bound-c-d-by-L} into \eqref{eq:adv-psm} and then applying Lemma~\ref{lemma:nr-multigraphs}, we obtain 
   \begin{align*}
&\Adv_{\leq D}^2(\PP,\QQ)\leq
1+
\sum_{v=1}^{D}
({ n}\tilde{\rho}^2\lambda^2)^{v}
\sum_{C=1}^{v}
L^{2C}\frac{1}{C!}
\sum_{\substack{v_1+\cdots+v_C=v\\ v_i\ge 1}}
\prod_{i=1}^C\frac{v_i^{v_i}}{v_i!}
\sum_{k=0}^{D} a_k(v_1,\dots,v_C),
\end{align*}
where
\begin{align*}
&a_k(v_1,\dots,v_C)
:=
(L\lambda)^{2k}
\binom{v^2+k}{k}.
\end{align*}
 
\paragraph*{Step 1: Bounding $\sum_{k=0}^D a_k(v_1,\ldots,v_C)$} 
To bound this expression, we follow similar steps to those
in \cite[Proof of Theorem 2.2 (a)]{sohn2025sharp}. Using $\binom{a}{b}\leq (ea/b)^b$ and splitting the sum we get,
\begin{align*}
\sum_{k=0}^{\,D}
   (L\lambda)^{2k} \binom{ v^2+k}{k}
   &\leq1+\sum_{ k=1}^{v^2}  \bigg((L\lambda)^2e\bigg(1+\frac{v^2}{k}\bigg)\bigg)^k+\sum_{{ k=v^2}}^{D} \bigg((L\lambda)^2e\bigg(1+\frac{v^2}{k}\bigg)\bigg)^k\\       
   &\leq 1+ v^2\sup_{k \in (0,\infty)} \bigg(\frac{2(L\lambda)^2e{ v^2}}{k}\bigg)^k+\sum_{k=v^2}^{D} (2L^2e\lambda^2)^k.
\end{align*}
We now note that the $\sup$ in the first term is attained at $k=2(L\lambda)^2v^2$. This gives $\exp(2L^2\lambda^2 v^2)\le \exp\left(\frac{2L^2v}{C_0}\right)$, where we used  \(1 \leq v \leq D \leq \lambda^{-2}/C_0\), which comes from the assumption on \(D\).
Hence, by choosing $C_0$ sufficiently large, the last term is at most $1$. Thus,
\begin{align}\label{eq:sum-ak}
   \sum_{k=0}^D a_k(v_1,\ldots,v_C)\leq 2+ v^2\exp(2L^2 v /C_0).
\end{align}

 \paragraph*{Step 2: Substituting the bound on $a_k(v_1,\ldots,v_C)$ and controlling the remaining sums} Using the bound on $\sum_{k=0}^D a_k(v_1,\ldots,v_C)$ given in \eqref{eq:sum-ak}, we have
 {\small \begin{align*}     
\Adv_{\leq D}^2(\PP,\QQ)
&\leq
1+
\sum_{v=1}^{D}
(n\tilde{\rho}^2\lambda^2)^{v}
\sum_{C=1}^{v}
\frac{L^{2C}}{C!}
\sum_{\substack{v_1+\cdots+v_C=v\\ v_i\ge 1}}
\prod_{i=1}^C\frac{v_i^{v_i}}{v_i!}
\left(2+ v^2\exp(2L^2v/C_0)\right).
\end{align*}}
Now let \[a_v:=({ n}\tilde{\rho}^2\lambda^2)^{v}
\sum_{C=1}^{v}
\frac{L^{2C}}{C!}
\sum_{\substack{v_1+\cdots+v_C=v\\ v_i\ge 1}}
\prod_{i=1}^C\frac{v_i^{v_i}}{v_i!}.\] To apply Stirling to the denominator, we note that we have $v_i! \geq \sqrt{2\pi v_i} (v_i/e)^{v_i}$. The factors~$e^{v_i}$ produced here multiply over components; since $v_1+\cdots+v_C=v$, their product is $e^v$. This is the source of the constant $e$ in the threshold. In particular, for any $v_1+\ldots+v_C=v$, we have
\[ \prod_{i=1}^C \frac{v_i^{v_i}}{v_i!}\leq \frac{e^v}{(2\pi)^{C/2}} \prod_{i=1}^C \frac{v_i^{v_i}}{v_i^{v_i + 1/2}} = \frac{e^v}{(2\pi)^{C/2}} \prod_{i=1}^C \frac{1}{v_i^{1/2}}  \leq \frac{e^v}{(2\pi)^{C/2}(v-C+1)^{1/2}}
.\]

 The last inequality follows because, subject to $v_1+\cdots+v_C=v$ and
$v_i\ge 1$, the product is maximized when one part equals $v-C+1$ and the
remaining $C-1$ parts equal $1$. Also, the number of terms in the sum is at most $\binom{v-1}{C-1}$ by `stars and bars'. 
Note that $\binom{v-1}{C-1}\leq \binom{v}{C}\leq (\frac{ev}{C})^C$. Combining this with Stirling's formula for $C!$, we obtain

   \begin{align*}
   a_v 
    &\leq  (ne\tilde{\rho}^2 \lambda^2)^v \sum_{C=1}^v \Big(\frac{L^2e^2v}{\sqrt{2\pi} C^2}\Big)^C.
    \end{align*}
    
\noindent Bounding the sum over $C$ by the number of terms times its maximum summand gives
\begin{align*}
v\max_{C\ge 1}\left(\frac{L^2 e^2 v}{\sqrt{2\pi} C^2}\right)^C
 \le
v \exp \left(\frac{2L}{(2\pi)^{1/4}}\sqrt v\right)
\le \exp\left(\log v+\frac{2L}{(2\pi)^{1/4}}\sqrt v\right)\le \exp(\kappa\sqrt v),
\end{align*}
where the last inequality holds since $\log v \le \sqrt v$ for all $v\ge 1$,
with $\kappa:=1+\frac{2L}{(2\pi)^{1/4}}$. 
Since $\tilde\rho=\rho(1+o(1))$, the assumption
$\lambda\le (1-\varepsilon)(\rho\sqrt{en})^{-1}$ implies $\lambda\le (1-\varepsilon/2)(\tilde\rho\sqrt{en})^{-1}
$ for all sufficiently large $n$. We use this form of the bound in what follows.
    \begin{align}
a_v
&\le  (1-\varepsilon/2)^{2v} \exp\left({\kappa\sqrt{v}}\right) 
=  \left((1-\varepsilon/2)^{2\sqrt{v}} \exp({\kappa})\right)^{\sqrt{v}}.
\label{eq.bound_av}
\end{align}
\noindent We note the following. Let $0<\varepsilon <1$, then for every real $j\geq 4\kappa/\varepsilon$, 
 \begin{equation}\label{clm.sqr_root}\bigg(\frac{1-\varepsilon/2}{1-\varepsilon/4}\bigg)^{j}\leq \exp\left({-\kappa}\right).
 \end{equation}
To see that \eqref{clm.sqr_root} holds, observe it is equivalent to prove that the $\log$ of the left-hand side (LHS) is at most $-\kappa$. The $\log$ of the LHS is
$j\log\left(\frac{1-\varepsilon/2}{1-\varepsilon/4}\right),$
which is bounded by $j\log(1-\varepsilon/4)\le -j\varepsilon/4$, since $\frac{1-\varepsilon/2}{1-\varepsilon/4}\le 1-\varepsilon/4$ and $\log(1-x)\le -x$. Thus,~\eqref{clm.sqr_root} follows by our assumption that $j\ge 4\kappa/\varepsilon$. Let $j_0$ be the minimum integer such that $j_0 \ge \frac{4\kappa}{\varepsilon}$, and set $v_0:=j_0^2$. Then, for every $v \ge v_0$, applying \eqref{clm.sqr_root} with $j=\sqrt{v}$ in \eqref{eq.bound_av} gives $a_v\le (1-\varepsilon/4)^{v}$. 
Hence, 
\begin{align*}
\Adv_{\leq D}^2(\PP,\QQ)& \le 1+\sum_{v=1}^D a_v \left(2+v^2\exp \left(\frac{2L^2v}{C_0}\right)\right).
\end{align*}
The finitely many terms $v<v_0$ contribute $O(1)$. Indeed, by
\eqref{eq.bound_av},
$a_v\le \exp(\kappa\sqrt v)\le \exp(\kappa\sqrt{v_0}),$ while $2+v^2\exp(2L^2v/C_0)
\le 2+v_0^2\exp(2L^2v_0/C_0)$ for all $v<v_0$. Since $v_0$ is independent of $n$, the total contribution
from $v<v_0$ is $O(1)$. For the tail $v\ge v_0$, choose $C_0$ large enough and enlarge $v_0$ if necessary, so that 
\[
2+v^2\exp \left(\frac{2L^2v}{C_0}\right) \le \left(1+\varepsilon/8\right)^v \qquad \text{for all} \ v \ge v_0.
\]
Since $a_v \le (1-\varepsilon/4)^v$ for $v \ge v_0$, we obtain
\begin{align*}
    \sum_{v=v_0}^D a_v \left(2+v^2\exp \left(\frac{2L^2v}{C_0}\right)\right) & \le \sum_{v=v_0}^\infty \left((1-\varepsilon/4)(1+\varepsilon/8)\right)^v=O(1),
\end{align*}
because $(1-\varepsilon/4)(1+\varepsilon/8)=1-\varepsilon/8-\varepsilon^2/32<1$.
Combining the finite part, i.e., $v<v_0$, and the tail gives
\[
\Adv_{\leq D}^2(\PP,\QQ)=O(1),
\]
which proves the claim.
\end{proof}

\subsubsection{Upper Bound}\label{sec:upperbound_PSM}
We now introduce the tools required for the proof of part (ii) of Theorem~\ref{thm:strong_PSM}. Lemma~\ref{lem:tree-uninformative} shows that, in the present planted-vs-planted setting, tree-indexed statistics have identical expectations under the two hypotheses and therefore cannot distinguish them. Thus, the first informative connected statistics are cyclic, which explains the appearance of balanced unicyclic graphs in the upper bound. This contrasts with many recovery analyses, where tree-like structures play a central role in belief-propagation and AMP heuristics; see~\cite{wein2025computational} for a survey. Thus, any degree-$D$ statistic used here must involve graphs with cycles, and the minimal such connected graphs are unicyclic. To keep the second-moment analysis tractable, we use a balanced subclass of unicyclic graphs, which we define below.
\paragraph*{ 
BUGs (Balanced Unicyclic Graphs)}
Fix an integer $k=k_n=\Theta(\log n)$, to be specified later (see \eqref{klarge}). Let $\mathcal{U}_k \subseteq \{0,1\}^{\binom{[n]}{2}}$
denote the family of \emph{balanced unicyclic graphs (BUGs)}. An element $\alpha\in\mathcal{U}_k$ is obtained as follows:
start from a root with exactly two children, attach to each child a rooted subtree with exactly $k$ edges, and then add one extra edge connecting the two children of the root. The resulting graph is connected and has exactly one cycle. We realize the minimal cycle as an edge joining the two children of the root rather than as a loop at the root, so that a single BUG family serves both PSM and PDS, where loops are not observed. See Figure~\ref{fig:BUGs3} for an illustration of $\mathcal{U}_3$. The BUG is the cyclic analogue of the balanced rooted trees of~\cite{sohn2025sharp}. The two $k$-edge rooted subtrees provide the balanced structure, while the edge joining the two children of the root supplies the minimal cycle required in the planted-vs-planted setting, where tree statistics are uninformative by Lemma~\ref{lem:tree-uninformative}.

\noindent By construction, each $\alpha\in\mathcal{U}_k$ has $D := 2k+3$
edges. Consider the degree-$D$ polynomial
\begin{equation}\label{def:BUG_f}
f(Y) := \sum_{\alpha\in\mathcal{U}_k} Y^\alpha .
\end{equation}

\noindent The cardinality of $\mathcal U_k$ is
\begin{align}\label{eq:size-U_k-app}
|\mathcal{U}_k|
\nonumber &= n \binom{n-1}{2}  \binom{n-3}{k} \binom{n-3-k}{k} (k+1)^{2(k-1)}\\
 &= (1+o(1))\,\left(4\pi e\,(k+1)^3\right)^{-1}\,(en)^D.
\end{align}

\begin{proposition}[Moment bounds for the BUG polynomial]\label{prop:BUG-moments} Let $k=\Theta(\log n)$ be chosen large enough so that \eqref{klarge} holds, and set $D:=2k+3$. Let $N:=(k+1)^{-6}(en\lambda\rho)^{2D}$. Assume $\lambda \ge (1+\varepsilon)(\rho\sqrt{en})^{-1}$, $n\rho=\omega(D^7)$, and $\rho\left((\ell\lambda)^2+1\right)=o(D^{-7})$. Then, for $f$ defined in~\eqref{def:BUG_f},
\[
\bigl|\E_{\PP}[f(Y)]-\E_{\QQ}[f(Y)]\bigr|  =  \Omega(\sqrt{N}),
\qquad\text{and}\qquad
\max\{\mathrm{Var}_{\QQ}(f(Y)),\mathrm{Var}_{\PP}(f(Y))\} =  o(N).
\]
\end{proposition}
\paragraph*{Roadmap for the second moment}
For $\alpha\in\mathcal U_k$, we have $|\alpha|=|V(\alpha)|=D$ and
$
\E_{\PP}[Y^\alpha]-\E_{\QQ}[Y^\alpha]
=(\ell'-\ell)\lambda^D\rho^D.$
Thus, by \eqref{eq:size-U_k-app}, $
|\E_{\PP}f-\E_{\QQ}f|^2
=\Theta\!\left((k+1)^{-6}(en\lambda\rho)^{2D}\right).$
Set
\begin{equation}\label{eq:N_UB}
    N:=(k+1)^{-6}(en\lambda\rho)^{2D}.
\end{equation}
The factor $e$ comes from the BUG count
$|\mathcal U_k|\asymp (k+1)^{-3}(en)^D$. It remains to prove $\operatorname{Var}(f)=o(N)$. We expand $\E[f^2]$ over pairs $\alpha,\beta\in\mathcal U_k$ and classify the terms according to $\alpha\cap\beta$. The edge-disjoint, vertex-disjoint contribution cancels with $(\E f)^2$. The diagonal terms are controlled by Lemma~\ref{lem:second-moment-equal}; all other terms are controlled by Lemma~\ref{lem:second-moment-not-equal} together with the balanced-overlap enumeration below. 

The main non-canceling contribution comes from Case~1.2, and the subsequent cases show that all other overlap configurations are negligible relative to it. Balancedness forces each nontrivial edge overlap to have either enough branch points or a long shared segment, giving the small factors required by $n\rho=\omega(D^7)$ and $\rho((\ell\lambda)^2+1)=o(D^{-7})$. We next state the pointwise moment bounds used in the variance calculation. For ease of notation, we set \begin{equation}\label{eq:eta}
    \eta:=(\ell\lambda)^2.
\end{equation}
\vspace{-2em}
\begin{lemma}\label{lem:second-moment-equal}
For any $\alpha\in\mathcal{U}_k$, with $\eta$ as defined in \eqref{eq:eta},
\[
\E_\QQ \left[Y^{2\alpha}\right] \le (\eta\rho+1)^{|\alpha|}.
\]
\end{lemma}

Let $m_\triangle$ (respectively $m_\cap$) denote the number of connected components of $\alpha \triangle \beta$  (respectively $\alpha \cap \beta$). We also let $b=|V(\alpha \triangle \beta) \cap V(\alpha \cap \beta)|$ (see~\eqref{eq.defn_b}).

\begin{lemma}\label{lem:second-moment-not-equal}
Let $\alpha,\beta \in \mathcal{U}_k$ with $\alpha \ne \beta$ and $\eta$ as defined in \eqref{eq:eta}.
Then,
\begin{itemize}
    \item If $\alpha\cap \beta$ contains a cycle, $$\E_\QQ \left[Y^{\alpha+\beta}\right] \le \ell^{m_\triangle}(\ell\lambda)^{|\alpha \triangle \beta|}
 \left(\frac{\rho}{\ell}\right)^{|V(\alpha \triangle \beta)|}
 ( \eta + 1)^{b-m_\cap+1}
 (\eta \rho + 1)^{|\alpha \cap \beta| - (b - m_\cap)-1}.$$
 \item If $\alpha \cap \beta$ is a forest, $$\E_\QQ \left[Y^{\alpha+\beta}\right] \le \ell^{m_\triangle}(\ell\lambda)^{|\alpha \triangle \beta|}
 \left(\frac{\rho}{\ell}\right)^{|V(\alpha \triangle \beta)|}
 ( \eta + 1)^{b-m_\cap}
 (\eta \rho + 1)^{|\alpha \cap \beta| - (b - m_\cap)}.$$
\end{itemize}

\end{lemma}
\noindent The proofs of these two lemmas are deferred to
Appendix~\ref{app:pf:technical-lemmas-second-moment-psm}. We now prove Proposition~\ref{prop:BUG-moments}.
\begin{proof}[Proof of Proposition~\ref{prop:BUG-moments}]
Computing $\E_{\PP}[f(Y)]$ is immediate:
\[
\E_{\PP}[f(Y)]
=\sum_{\alpha\in\mathcal{U}_k}\E_{\PP}[Y^\alpha]
=|\mathcal{U}_k| \ell'^{\,C+|\alpha|-|V(\alpha)|} \lambda^D\rho^D .
\]
For $\alpha\in\mathcal{U}_k$, the graph is connected and unicyclic, hence, $C=|\mathcal{C}(\alpha)|=1$ and $|\alpha|=|V(\alpha)|$, so the exponent simplifies to $1$ and therefore
\[
\E_{\PP}[f(Y)]  =  |\mathcal{U}_k| \ell' \lambda^D\rho^D .
\]
The second moment calculation is involved and requires bounds on $
\E_\QQ[Y^{\alpha + \beta}]$.
We organize the sum over $\alpha,\beta\in\cU_k$ by their overlap. The key point is that the moment bound depends not only on the number of shared edges but also on how these edges are arranged within the two BUGs. We therefore classify pairs using the size and component structure of $\alpha\cap\beta$, as well as the component structure of the symmetric difference  $\alpha\triangle\beta$, following the approach of \cite{sohn2025sharp}. We use the following parameters throughout the calculation.

\vspace{0.5em}

Let $s:=\bigl|\alpha \cap \beta\bigr|\geq1$, and call the set of edges 
$\alpha \cap \beta$ the \emph{core}. The core either contains its unique cycle or is a forest. We treat the cyclic case first; the forest case is handled separately following the core/branch-point enumeration in \cite{sohn2025sharp}[Lemma 7.1].

Let $V_c:=V(\alpha \cap \beta)$ denote the \emph{core vertices}. Vertices where the core intersects the remaining edges $\alpha \triangle \beta$ are called \emph{branch points}. We denote their number by
\begin{equation}\label{eq.defn_b}
    b := \bigl| V(\alpha \triangle \beta) \cap V_c \bigr| \ge m_\cap.
\end{equation}
Finally, the number of vertices shared by $\alpha$ and $\beta$ outside the core will be denoted by
\[w:=\bigl|(V(\alpha)\cap V(\beta)) \setminus V_c\bigr|.\]

\noindent We now consider the following three cases for $\alpha \cap \beta$, each with its own subcases.

\medskip

\noindent \paragraph*{Case 1: $\alpha \cap \beta=\varnothing$}
Note that in this case, $w \ge 0$ denotes the number of vertices that are shared between $\alpha$ and $\beta$. For fixed $w \ge 0$, we bound the number of pairs $(\alpha,\beta)$ arising in Case~1 by separating according to whether the root lies in the intersection. Denote the corresponding counts by $N_{\alpha,\beta,w}^{(\mathrm{root})}$ and $N_{\alpha,\beta,w}^{(\lnot \mathrm{root})}$, where $s$ indicates the shared-root subcase. These quantities are bounded above as follows.
\noindent\paragraph*{Case 1.1: Root is shared} First, we choose $w$ shared vertices. Then for each of $\alpha,\beta$ we choose 2 vertices to neighbor the shared root plus $2k-w$ additional vertices, then decide how the vertices are split among the two subtrees, and then finally choose the structure of the subtrees (using Cayley's tree formula to count spanning trees on $k+1$ vertices). 
\begin{align*}
    N_{\alpha, \beta, w}^{(\mathrm{root})} & \leq \binom{n}{w}w\left[\binom{n-1}{2(k+1)-w}\frac{1}{2}\binom{2(k+1)}{k+1}(k+1)^{2k}\right]^2\\
&\leq(1+o(1))\frac{(en)^{4(k+1)+1}}{16\pi^2e(k+1)^6}w \left(\frac{4(k+1)^2}{n}\right)^{w-1} .
\end{align*}
The second inequality holds by expanding out the binomials and applying Stirling's formula.

\noindent\paragraph*{Case 1.2: Root is not shared} We count these as above except we first choose a root for each, then choose $w$ shared vertices.
\begin{align*}
N_{\alpha, \beta, w}^{(\lnot \mathrm{root})}    &\leq n^2 \binom{n-2}{w}\left[\binom{n-2}{2(k+1)-w}\frac12\binom{2(k+1)}{k+1}(k+1)^{2k}\right]^2\\
&\leq (1+o(1)) \frac{(en)^{4(k+1)+2}}{16\pi^2 e^2(k+1)^6}\left(\frac{4(k+1)^2}{n}\right)^w 
\end{align*} 
In Case~1, the graphs have no edges in common and hence, cannot share a cycle. Hence, by Lemma~\ref{lem:second-moment-not-equal},
\begin{align*}
     \E_\QQ[Y^{\alpha+\beta}]&\leq \ell^{m_\triangle}(\ell\lambda)^{|\alpha \triangle \beta|} \left(\frac{\rho}{\ell}\right)^{|V(\alpha \triangle \beta)|} 
(\eta + 1)^{b-m_\cap} (\eta \rho + 1)^{|\alpha \cap \beta| - (b - m_\cap)}.
\end{align*}
Noting $b=m_\cap=|\alpha \cap \beta|=0$, we have
\begin{align*}
\E_\QQ[Y^{\alpha+\beta}] & \leq \ell^{m_\triangle+|\alpha\triangle\beta|-|V(\alpha\triangle \beta)|}\lambda^{|\alpha \triangle \beta|} \rho^{|V(\alpha \triangle \beta)|}=\ell^{m_\triangle+w}\lambda^{2D} \rho^{2D-w}.
\end{align*}
Observe that $m_\triangle \le 2\). Indeed, since $\alpha\cap\beta=\varnothing$, we have
$\alpha\triangle\beta=\alpha\cup\beta$. Moreover, $\alpha,\beta\in \cU_k$ are connected, so
$\alpha\cup\beta$ has one connected component if $V(\alpha)\cap V(\beta)\neq\varnothing$,
and two otherwise.
Hence, we obtain the following bounds for Cases 1.1 and 1.2, respectively:
\begin{align*}
    \sum_{\substack{\alpha,\beta \in \cU_k \\ \text{Case 1.1}}}\E_\QQ[Y^{\alpha+\beta}]\leq (1+o(1))\frac{\ell^3(en)^{4(k+1)+1}(\lambda\rho)^{2D}}{16\rho\pi^2e(k+1)^6}  \sum_{w\geq 1}w \left(\frac{4\ell(k+1)^2}{\rho n}\right)^{w-1}, 
\end{align*}
\begin{align*}
    \sum_{\substack{\alpha,\beta \in \cU_k \\ \text{Case 1.2}}}\E_\QQ[Y^{\alpha+\beta}]\leq(1+o(1))\frac{\ell^2 (en)^{4(k+1)+2}(\lambda\rho)^{2D}}{16\pi^2 e^2(k+1)^6} \sum_{w\geq 0}\left(\frac{4\ell(k+1)^2}{\rho n}\right)^w.
\end{align*}

\noindent Let $r_n:=4\ell(k+1)^2/(\rho n)$. Since $k=\Theta(D)$ and $n\rho=\omega(D^7)$, we have $r_n=o(1)$. Hence the two sums over $w$ are both $1+o(1)$. Comparing the remaining leading factors, the ratio of the Case~1.1 bound to the Case~1.2 bound is $\ell/(n\rho)=o(1).$ Thus, the Case~1.1 contribution is negligible, and the dominating contribution comes from Case~1.2. 

The key point is the cancellation of the leading term. The $w=0$ term in
Case~1.2 has the same leading asymptotic as $\E_\QQ[f(Y)]^2$, and therefore,
cancels with $\E_\QQ[f(Y)]^2$ in $\operatorname{Var}_\QQ(f(Y))$. Consequently,
to prove $\operatorname{Var}_\QQ(f(Y))=o(N)$, it remains to show that every
other term in $\E_\QQ[f(Y)^2]$ is $o(N)$, where $N$ is defined in
\eqref{eq:N_UB}.
\subsection*{Case 2: $\alpha=\beta$}
Using the calculations for the size of $\cU_k$ in Equation~\eqref{eq:size-U_k-app} and Lemma~\ref{lem:second-moment-equal}, we obtain:
\begin{align*}
    \sum_{\substack{\alpha,\beta \in \cU_k \\ \text{Case 2}}}\E_\QQ[Y^{\alpha+\beta}]&\leq |\cU_k|(\eta\rho+1)^D\leq(1+o(1)) (4\pi e (k+1)^3)^{-1} (en)^D (\eta\rho+1)^D.
\end{align*}
To show Case 2 is $o(N)$, we need to show that the following holds
\begin{equation*}
(k+1)^{3}\left(\frac{\eta\rho+1}{en\lambda^2\rho^2}\right)^D= o(1).\end{equation*}
To see this, note that
\begin{equation*}
   (k+1)^{3}\left(\frac{\eta\rho+1}{en\lambda^2\rho^2}\right)^D= (k+1)^{3}\left(\frac{(\ell\lambda)^2\rho+1}{en\lambda^2\rho^2}\right)^D=(k+1)^{3}\left({\frac{\ell^2}{en\rho}}+\frac{1}{en\lambda^2\rho^2}\right)^D.
\end{equation*}
Since $\lambda\ge (1+\varepsilon)(\rho\sqrt{en})^{-1}$,
we obtain 
$\frac{1}{en\lambda^2\rho^2}\le (1+\varepsilon)^{-2},$
while $n\rho=\omega(D^7)$ gives $\ell^2/(en\rho)=o(1)$. Thus, the
quantity in parentheses is at most $1-\xi$ for some constant
$\xi>0$. Since $D=\Theta(\log n)$, the factor
$(k+1)^3(1-\xi)^D$ is $o(1)$.
\subsection*{Case 3: $\alpha \neq \beta, \alpha \cap \beta \neq \varnothing$}We split this case according to whether the core contains the unique cycle of the BUGs.

\subsubsection*{Case 3.1: The core $\alpha \cap \beta$ contains a cycle} In this cyclic-core case, no additional root-location factor is needed. Indeed, each BUG has a unique cycle, namely the triangle formed by the root and its two
children. Hence, the cyclic component of the core already identifies the common triangle, and consequently, the roots of both BUGs up to the choices already included in the count below. Let $T$ be the connected component of $\alpha\cap\beta$ containing this shared cycle, and let $s_T:=|V(T)|.$

\paragraph*{Idea} Count $(\alpha,\beta)$ pairs whose overlap has a cycle for the given tuple $(s,s_T,m,b,w)$.

\paragraph*{Step 1}  The number of ways to choose $T-$component in the core is at most:
{\small \begin{align*}
    &\binom{n}{s_T}\binom{s_T}{3} 3! 2\bigg(\sum_{i=2}^{s_T-3} i^{i-1} (s_T-1-i)^{s_T-2-i} \binom{s_T -3}{i-1}\bigg)+\binom{n}{s_T}\binom{s_T}{3}3!(3!-2)(s_T-2)^{s_T-3}\\
    &=\binom{n}{s_T}\binom{s_T}{3}C_T,
\end{align*}}
where up to an absolute constant, $C_T:=\sum_{i=2}^{s_T-3} i^{i-1} (s_T-1-i)^{s_T-2-i} \binom{s_T -3}{i-1}
+(s_T-2)^{s_T-3}$. We claim that $C_T\leq \tilde C s_T^{s_T-5/2}$, for some $\tilde C>0$.
Indeed,
\[
\binom{s_T-3}{i-1}
=
\frac{i(s_T-1-i)}{(s_T-1)(s_T-2)}
\binom{s_T-1}{i}.
\]
By writing $\binom{s_T-1}{i}=\frac{(s_T-1)!}{i!(s_T-1-i)!}$ and using Stirling's upper bound for $(s_T-1)!$ and lower bounds for $i!$ and $(s_T-1-i)!$, we get
\[
\binom{s_T-1}{i} \le C_1\sqrt{\frac{s_T-1}{i(s_T-1-i)}}
\frac{(s_T-1)^{s_T-1}}{i^i(s_T-1-i)^{s_T-1-i}}, \qquad C_1>0.
\]
Therefore, for some constant $C_2>0$,
\[
i^{i-1}(s_T-1-i)^{s_T-2-i}\binom{s_T-3}{i-1}\le C_2 s_T^{s_T-5/2}\frac{1}{\sqrt{i(s_T-1-i)}}.
\]
Since $\sum_{i=1}^{s_T-2}\frac{1}{\sqrt{i(s_T-1-i)}}=O(1)$, we have, for some constant $C_3>0$,
\[
\sum_{i=2}^{s_T-3} i^{i-1}(s_T-1-i)^{s_T-2-i}\binom{s_T-3}{i-1}\le C_3 s_T^{s_T-5/2}. 
\]
Since also $(s_T-2)^{s_T-3}\le s_T^{s_T-3}\le s_T^{s_T-5/2}$, after increasing $\tilde C$ if necessary, we obtain
\[
C_T\le \tilde C s_T^{s_T-5/2}.
\]

\paragraph*{Step 2}
To simplify notation, set $m:=m_\cap$. Note that we can choose the triangle-containing part in at most 
\begin{equation}\label{eq:choose-triangle}
    \binom{n}{s_T}\binom{s_T}{3}\tilde C s_T^{s_T-5/2}
\end{equation}
ways. After this component is fixed, the remaining part of the core is a forest with $s-s_T$ edges and $m-1$ connected components. The remaining forest contribution is controlled by a core/branch-point enumeration adapted to the present overlap structure. We state the required bound in the next lemma and defer the proof, which follows the strategy of \cite[Proof of Theorem 2.2(b), Case 3]{sohn2025sharp}, to the appendix.
\begin{lemma}[Forest completion count]\label{lem:forest-completion-count} Suppose the triangle-containing component of $\alpha\cap\beta$ has already been fixed. Assume that the remaining part of the core is a forest with $s-s_T$ edges and $m-1$ connected components.  Let $b$ be the number of branch vertices and let $w$ be the number of additional common vertices outside the core. Then the number of ways to complete the remaining parts of $\alpha$ and $\beta$ is at most
    \begin{equation*}
        e^{4}(e^2(D-s_T+2)^3)^b\bigg(\frac{(D-s_T+2)^2}{n}\bigg)^w\left(en\right)^{2(D-s_T)-s+s_T-(m-1)}.
    \end{equation*}
\end{lemma}
\begin{proof}
See Appendix~\ref{app:pf-lem-forest-count-completion}.
\end{proof}
Therefore, using \eqref{eq:choose-triangle} and Lemma~\ref{lem:forest-completion-count} for a fixed tuple $s,s_T,m,b,w$, the number of pairs $\alpha,\beta$ arising in Case 3.1 is at most
\begin{align*}
    &\binom{n}{s_T}\binom{s_T}{3} \tilde Cs_T^{s_T-5/2}\biggl\{e^4(e^2(D-s_T+2)^3)^b\bigg(\frac{(D-s_T+2)^2}{n}\bigg)^w\left(en\right)^{2(D-s_T)-s+s_T-(m-1)}\biggr\}\\
    &\le \tilde C_1e^4(e^2(D-s_T+2)^3)^b\bigg(\frac{(D-s_T+2)^2}{n}\bigg)^w\left(en\right)^{2D-s-m+1}, \qquad \tilde C_1>0.
\end{align*}
By Lemma~\ref{lem:second-moment-not-equal}, each of the $(\alpha,\beta)$ pairs under Case 3.1 has 
\begin{align*}
    \E_\QQ[Y^{\alpha+\beta}]&\leq \ell^{m_\triangle+|\alpha \triangle \beta|-|V(\alpha \triangle \beta)|}\lambda^{|\alpha \triangle \beta|} \rho^{|V(\alpha \triangle \beta)|} (\eta + 1)^{b-m+1} (\eta \rho + 1)^{|\alpha \cap \beta| - (b - m)-1}.
\end{align*}
We now note that $m_\triangle \le b$. Indeed, let $K$ be a connected component of $\alpha\triangle\beta$. Since $\alpha$ and $\beta$ are connected and $\alpha\cap\beta\neq\varnothing$, the union $\alpha\cup\beta$ is connected. Thus, $K$ must contain a vertex belonging to the core $\alpha\cap\beta$; otherwise $K$ would form a connected component of $\alpha\cup\beta$ disjoint from the core. Hence, every connected component of $\alpha\triangle\beta$ contains at least one vertex in $V(\alpha\triangle\beta)\cap V(\alpha\cap\beta)$. Since distinct connected components are vertex-disjoint, this gives $m_\triangle\le |V(\alpha\triangle\beta)\cap V(\alpha\cap\beta)|=b.$ Therefore,
\begin{align*}
\E_\QQ[Y^{\alpha+\beta}]&\leq \ell^{b+|\alpha \triangle \beta|-|V(\alpha \triangle \beta)|}\lambda^{|\alpha \triangle \beta|} \rho^{|V(\alpha \triangle \beta)|} (\eta + 1)^{b-m+1} (\eta \rho + 1)^{|\alpha \cap \beta| - (b - m)-1}.
\end{align*}
Recall the following set sizes:
\begin{align}\label{eq:setsize-pf-psm-ub}
    |\alpha \triangle \beta|=2D-2s \qquad \text{and} \qquad |V(\alpha \triangle \beta)|=2(D-s-m+1)+b-w.
\end{align}
Using \eqref{eq:setsize-pf-psm-ub}, the fixed-pair moment bound becomes
\begin{align}    \label{eq:case3-fixed-pair}
\E_\QQ[Y^{\alpha+\beta}]&\leq \ell^{b+2m-2-b+w}\lambda^{2(D-s)} \rho^{2(D-s-m+1)+b-w} (\eta + 1)^{b-m+1} (\eta \rho + 1)^{s - (b - m)-1}.
\end{align}
We now sum this bound over all Case~3.1 pairs. The counting bound for such pairs,
together with \eqref{eq:case3-fixed-pair}, gives the following upper bound for
$\sum_{\substack{\alpha,\beta \in \U_k \\ \mathrm{Case \ 3.1}}}\E_\QQ[Y^{\alpha+\beta}]$:
{\small
\begin{align}
\notag
&\sum_{s,m,b,w}\sum_{s_T\le s}
\tilde C_1  e^4
\bigl(e^2(D+2)^3\bigr)^b
\bigg(\frac{(D+2)^2}{n}\bigg)^w\\
\notag
&\qquad \times
(en)^{2D-s-m+1}
\ell^{2m-2+w}
\lambda^{2(D-s)}
\rho^{2(D-s-m+1)+b-w}
(\eta+1)^{b-m+1}
(\eta\rho+1)^{s-(b-m)-1}.
\end{align}}
After reorganizing the powers of $n,\rho,\lambda,\ell,\eta$, using that
$3\le s_T\le s\le D$ and hence $\sum_{s_T}1\le D$, and using
$D^2/(n\rho)=o(1)$ to absorb the $w$-dependent factor, we obtain
{\small
\begin{align}
\sum_{\substack{\alpha,\beta\in\cU_k\\ \mathrm{Case\ 3.1}}}
\E_\QQ[Y^{\alpha+\beta}]
&\le
(1+o(1))\frac{\eta+1}{\eta\rho+1}
e^6 n\rho^2D(en\lambda\rho)^{2D}\nonumber\\
&\quad\times
\sum_{s,m,b}
\left(\frac{(\eta\rho+1)}{en\lambda^2\rho^2}\right)^s
\left(\frac{e^2\ell^2(D+2)^3}{en\rho}\right)^m
\left(\frac{e^2(D+2)^3\rho(\eta+1)}{\eta\rho+1}\right)^{b-m}.
\end{align}}

Finally, using $\lambda\ge (1+\varepsilon)(\rho\sqrt{en})^{-1}$, the first ratio is at most $1-\chi$ for some $\chi>0$. Indeed,
\[
\frac{\eta\rho+1}{en\lambda^2\rho^2}
=
\frac{1}{en\lambda^2\rho^2}
+
\frac{\ell^2}{en\rho}
\le (1+\varepsilon)^{-2}+o(1),
\]
so this ratio is at most $1-\chi$ for some $\chi=\chi(\varepsilon)>0$ and all
large $n$. Thus,
 \begin{align} 
\notag  \sum_{\substack{\alpha,\beta\in\cU_k\\ \mathrm{Case\ 3.1}}}
\E_\QQ[Y^{\alpha+\beta}] &\leq(1+o(1)) \left(\frac{\eta+1}{\eta\rho+1}\right) e^6n\rho^2D(k+1)^6N\\
   &\label{eq:final-bound} \quad \times \sum_{s,m,b} (1-\chi)^s \left(\frac{e^2\ell^2(D+2)^3}{en\rho}\right)^m\left(\frac{e^2(D+2)^3\rho(\eta+1)}{\eta\rho+1}\right)^{b-m}.
   \end{align}
 The triple sum over $(s,m,b)$ involves three geometric ratios: $(1-\chi)$ in $s$ for a constant $\chi>0$, $O(D^3/(n\rho))$ in $m$, and $O(D^3\rho(\eta+1))$ in $b-m$. Under the assumptions $n\rho=\omega(D^7)$ and $\rho(\eta+1)=o(D^{-7})$, the first ratio is bounded away from $1$, while the other two are $o(1)$. Hence the corresponding geometric sums are uniformly bounded. Recall that $b$ denotes the number of branch points and that $s = |\alpha \cap \beta|$. We next use a structural overlap lemma for BUGs analogous to the balanced-tree overlap lemma in~\cite{sohn2025sharp}[Lemma 7.2]. We prove the required BUG version below.
   \begin{lemma}\label{lem:b-k-conditions}Every pair $\alpha,\beta \in \cU_k$ with $\alpha \cap \beta \neq \varnothing$ must either have $b\geq 2$ or $s\geq k+3$.
   \end{lemma}
\begin{proof}
    See Appendix~\ref{app:pf:lem:b-k-conditions}.
\end{proof}

We now fix the constant in the choice $k=\Theta(\log n)$ so that
\begin{equation}\label{klarge}
    (1-\chi)^{k+1}\leq n^{-1}.
\end{equation}
Starting from \eqref{eq:final-bound}, and using $k=\Theta(D)$, write
\[
A:=O\left(\frac{D^3}{n\rho}\right),
\qquad
B:=O\left(\frac{D^3\rho(\eta+1)}{\eta\rho+1}\right).
\]
Then,
\begin{align*}
&\sum_{\substack{\alpha,\beta\in\mathcal U_k\\ \mathrm{Case\ 3.1}}}
\E_{\mathbb Q}[Y^{\alpha+\beta}]\le
O\left(n\rho^2ND^7\frac{\eta+1}{\eta\rho+1}\right)
\sum_{s,m,b}(1-\chi)^s A^m B^{b-m}.
\end{align*}
By Lemma~\ref{lem:b-k-conditions}, each term satisfies $b\ge2$ or $s\ge k+3$. For $m=1$,
\[
O\left(n\rho^2ND^7\frac{\eta+1}{\eta\rho+1}\right)A
=
O(ND^7)B,
\]
so the $m=1$ contribution is bounded by
\[
O(ND^7)
\sum_{\substack{s,b\ge1\\ b\ge2\ \mathrm{or}\ s\ge k+3}}
(1-\chi)^sB^b .
\]
For $m\ge2$, extracting two powers of $A$ gives
\[
O\left(\frac{ND^{13}}{n}\frac{\eta+1}{\eta\rho+1}\right)
\sum_{m\ge2}\sum_{b\ge m} A^{m-2}B^{b-m}.
\]
Therefore,
\begin{align*}
&\sum_{\substack{\alpha,\beta\in\mathcal U_k\\ \mathrm{Case\ 3.1}}}
\E_{\mathbb Q}[Y^{\alpha+\beta}]\le
O(ND^7)
\sum_{\substack{s,b\ge1\\ b\ge2\ \mathrm{or}\ s\ge k+3}}
(1-\chi)^sB^b
+
O\left(\frac{ND^{13}}{n}\frac{\eta+1}{\eta\rho+1}\right),
\end{align*}
where the remaining geometric sum in the $m\ge2$ term is bounded since $A=o(1)$ and $B=o(1)$. In the first sum, the alternative $b\ge2$ yields two powers of $B$, while $s\ge k+3$ yields a factor $n^{-1}$ by \eqref{klarge}; in the latter case, we still retain the mandatory factor $B$, since $b\ge1$. Hence,
\begin{align}
\notag &\sum_{\substack{\alpha,\beta\in\mathcal U_k\\ \mathrm{Case\ 3.1}}}
\E_{\mathbb Q}[Y^{\alpha+\beta}]\\
&\notag \le
O \left(ND^{7}\right)
\left[
O \left(D^6\rho^2(\eta+1)^2\right)
+
O \left(n^{-1}D^3\rho(\eta+1)\right)
\right]+
O\left(\frac{ND^{13}}{n}\cdot \frac{\eta+1}{\eta\rho+1}\right)\\
&=
O \left(ND^{13}\rho^2(\eta+1)^2\right)
+
O \left(\frac{ND^{10}\rho(\eta+1)}{n}\right)
+
O\left(\frac{ND^{13}}{n}\cdot \frac{\eta+1}{\eta\rho+1}\right).\label{eq.psm_ub_case_3.1}
\end{align}
The three terms in \eqref{eq.psm_ub_case_3.1} are all $o(N)$. Indeed, using $\rho(\eta+1)=o(D^{-7})$,
\[D^{13}\rho^2(\eta+1)^2=\frac{1}{D} \left(D^7\rho(\eta+1)\right)^2=o(1),
\]
and $n\rho=\omega(D^7)$,
\[\frac{D^{13}}{n}\cdot \frac{\eta+1}{\eta\rho+1}\le
\frac{D^{13}(\eta+1)}{n}=\frac{1}{D}\left(D^7\rho(\eta+1)\right)\left(\frac{D^{7}}{n\rho}\right)=o(1).
\]
The middle term satisfies
\[
\frac{D^{10}\rho(\eta+1)}{n}=\frac{D^3}{n}\left(D^7\rho(\eta+1)\right)=o(1).
\]
Hence, the Case~3.1 contribution is $o(N)$.

\subsubsection*{Case 3.2: The core $\alpha \cap \beta$ is a forest}
It remains to bound the contribution of pairs $\alpha,\beta \in \cU_k$ for which $\alpha \ne \beta $ and $\alpha \cap \beta\ne \varnothing$ is a forest. We claim that the total contribution of such pairs is $o(N)$. We bound this contribution by the core/branch-point enumeration, analogous to the form used for tree overlaps in \cite[Proof of Theorem~2.2 (b), Case 3]{sohn2025sharp}. An additional bookkeeping is that the statistic sums over BUGs as unrooted edge sets, while the rooted-tree enumeration has the root fixed as part of the object. Once the forest core and the remaining vertices are fixed, each of the two BUG roots has at most $D$ possible locations. Thus, reusing the rooted enumeration costs an additional factor at most $D^2$. Using the proof of \Cref{lem:forest-completion-count}, now with $s$ core edges and $m$ core components, the number of such pairs is at most,
\[
        D^2 e^2
        \bigl(e^2(D+2)^3\bigr)^b
        \left(\frac{(D+2)^2}{n}\right)^w
        (en)^{2D-(s+m)}.
\]
For a fixed pair $(\alpha,\beta)$, the forest case of Lemma~\ref{lem:second-moment-not-equal} gives
$$\E_\QQ \left[Y^{\alpha+\beta}\right] \le \ell^{m_\triangle}(\ell\lambda)^{|\alpha \triangle \beta|}
 \left(\frac{\rho}{\ell}\right)^{|V(\alpha \triangle \beta)|}
 ( \eta + 1)^{b-m}
 (\eta \rho + 1)^{|\alpha \cap \beta| - (b - m)}.$$
Using \( |\alpha\triangle\beta|=2D-2s\), 
\(|V(\alpha\triangle\beta)|=2(D-s-m)+b-w\), and $m_\triangle\le b$, and combining with the counting bound above, we obtain
\begin{align*}
&\sum_{\substack{\mathrm{Case\ 3.2}}}
\mathbb E_Q[Y^{\alpha+\beta}]\le
O(ND^8)
\sum_{s,m,b}
(1-\chi)^s
\left(\frac{C(D+2)^3}{en\rho}\right)^m
\left(\frac{C(D+2)^3\rho(\eta+1)}{\eta\rho+1}\right)^{b-m}.
\end{align*}
By Lemma~\ref{lem:b-k-conditions}, every pair in Case~3.2 satisfies \(b\ge2\) or \(s\ge k\). Splitting into \(m=1\) and \(m\ge2\), and using \((1-\chi)^k\le n^{-1}\), gives
\begin{align}
\notag\sum_{\substack{\alpha,\beta \in \cU_k\\ \mathrm{Case\ 3.2}}}
\mathbb E_Q[Y^{\alpha+\beta}]&\le
O\left(\frac{ND^8}{n\rho^2(\eta+1)}\right)
\left[
O\left(D^6\rho^2(\eta+1)^2\right)
+
O\left(n^{-1}D^3\rho(\eta+1)\right)
\right]\\
\notag&\quad+
O\left(\frac{ND^{14}}{n^2\rho^2}\right)
\sum_{m\ge2}\sum_{b\ge m}
O\left(D^3(n\rho)^{-1}\right)^{m-2}
O\left(D^3\rho(\eta+1)\right)^{b-m}\\
&=
O\left(\frac{D^{14}(\eta+1)}{n}N\right)
+
O\left(\frac{D^{11}}{n^2\rho}N\right)
+
O\left(\frac{D^{14}}{n^2\rho^2}N\right).\label{eq.psm_ub_case3.2}
\end{align}

The geometric sums are bounded since $D^3/(n\rho)=o(1)$ and $D^3\rho(\eta+1)=o(1)$, which follow from $n\rho=\omega(D^7)$ and $\rho(\eta+1)=o(D^{-7})$, respectively. The three terms in \eqref{eq.psm_ub_case3.2} are $o(N)$. Indeed,
\[\frac{D^{14}(\eta+1)}{n}=\frac{D^7}{n\rho}\bigl(D^7\rho(\eta+1)\bigr)=o(1),\]
using $n\rho=\omega(D^7)$ and $\rho(\eta+1)=o(D^{-7})$,
\[\frac{D^{11}}{n^2\rho}=\frac{D^4}{n}\frac{D^7}{n\rho}=o(1), \qquad \text{and} \qquad \frac{D^{14}}{n^2\rho^2}=\left(\frac{D^7}{n\rho}\right)^2=o(1).\]
Hence, the Case~3.2 contribution is $o(N)$.
Therefore, the assumptions needed for Case 3 to be $o(N)$ are
$n\rho=\omega(D^7)$ and $\rho(\eta+1)=o(D^{-7})$.
 Recall that the dominant contribution to $\E_\QQ[f(Y)^2]$ comes from Case 1.2 at $w=0$:
\[
\sum_{\substack{\alpha,\beta\in\U_k\\\text{Case 1.2},\, w=0}} \E_\QQ[Y^{\alpha+\beta}]
= (1+o(1))\,(16\pi^2 e^2 (k+1)^6)^{-1}\,\ell^2\,(en)^{2D}(\lambda\rho)^{2D}.
\]
All other contributions (Cases 1.1, 1.2 with $w\geq 1$, 2, 3) are $o(N)$. From~\eqref{eq:size-U_k-app},
\[
\E_\QQ[f(Y)]^2 = |\U_k|^2\,\ell^2\,\lambda^{2D}\rho^{2D}
= (1+o(1))\,(4\pi e\, k^3)^{-2}\,\ell^2\,(en)^{2D}(\lambda\rho)^{2D}.
\]
Since $(4\pi e\, k^3)^{-2} = (1+o(1))\,(16\pi^2 e^2\,(k+1)^6)^{-1}$, the leading terms match, giving
$\mathrm{Var}_\QQ(f(Y)) = o(N)$. The same holds under $\PP$ with $\ell$ replaced by $\ell'$.
Finally, since $\ell\neq\ell'$, 
\[
|\E_\PP[f(Y)] - \E_\QQ[f(Y)]|
= |\ell'-\ell||\U_k|\lambda^D\rho^D
\geq (1-o(1))\frac{|\ell'-\ell|}{4\pi e\,k^3}(en\lambda\rho)^D
= \Omega(\sqrt{N}). \qedhere
\]
\end{proof}
\noindent Combining the results so far, we are ready to prove part~(ii) of
Theorem~\ref{thm:strong_PSM}.

\begin{proof}[Proof of Theorem~\ref{thm:strong_PSM}\textnormal{(ii)}] Let $k=\Theta(\log n)$ be as in Proposition~\ref{prop:BUG-moments} and set
$D=2k+3$. Set $\lambda_0=(1+\varepsilon)(\rho\sqrt{en})^{-1}$. At signal strength $\lambda_0$, since $\eta=(\ell\lambda_0)^2$ and $\ell$ is fixed,
\[\rho(\eta+1)=\frac{\ell^2(1+\varepsilon)^2}{en\rho}+\rho=
o(D^{-7}),
\]
using $n\rho=\omega(D^7)$ and $\rho=o(D^{-7})$. Hence, Proposition~\ref{prop:BUG-moments} applies at signal strength $\lambda_0$, giving
\[
\bigl|\E_\PP[f(Y)]-\E_\QQ[f(Y)]\bigr|=\Omega(\sqrt N),
\qquad
\max\{\operatorname{Var}_\QQ(f(Y)),\operatorname{Var}_\PP(f(Y))\}=o(N),
\]
where all expectations and variances in this display are taken at signal strength $\lambda_0$. Therefore,
\[
\frac{\left|\E_\PP[f(Y)]-\E_\QQ[f(Y)]\right|}
{\sqrt{\max\{\operatorname{Var}_\QQ(f(Y)),\operatorname{Var}_\PP(f(Y))\}}}
=
\frac{\Omega(\sqrt N)}{\sqrt{o(N)}}=\omega(1),
\]
implying that the BUG polynomial $f$ strongly separates at signal strength $\lambda_0$. Now let $\lambda\ge\lambda_0$ and set $a:=\lambda_0/\lambda$. Given an observation
$Y$ at signal strength $\lambda$, let
\[
\widetilde Y=aY+\sqrt{1-a^2}\,\widetilde X,
\]
where $\widetilde X$ is an independent standard Gaussian symmetric matrix. Note that
$Y=X_\lambda+Z$. Here $X_\lambda$ denotes the conditional mean matrix under the law being considered, i.e., for instance under $\QQ$, $(X_\lambda)_{ij}=\ell\lambda\,\mathds 1\{\Theta_i=\Theta_j\in[\ell]\}.$ Thus,
\[
\widetilde Y
=
X_{\lambda_0}
+
\left(aZ+\sqrt{1-a^2}\,\widetilde X\right),
\]
and the bracketed term is again standard Gaussian noise independent of $\Theta$.
Thus, $\widetilde Y$ has the law of an observation at signal strength $\lambda_0$. Define
\[
g(Y):=\E_{\widetilde X}\left[f\left(aY+\sqrt{1-a^2}\,\widetilde X\right)\right].
\]
Then $\deg(g)\le D$. Moreover, the mean gap of $g$ at signal strength $\lambda$
equals the mean gap of $f$ at signal strength $\lambda_0$. Finally, since
$g(Y)=\E_{\widetilde X}[f(\widetilde Y)\mid Y]$ under
either planted law, conditional Jensen gives $\operatorname{Var}(g(Y))
\le
\operatorname{Var}(f(\widetilde Y)).$ Since $\widetilde Y$ is an observation at signal strength $\lambda_0$,
the variance of $g$ at signal strength $\lambda$ is at most the variance of $f$
at signal strength $\lambda_0$, under either planted law. Hence, $g$ strongly
separates $\PP$ and $\QQ$ at signal strength $\lambda$.
\end{proof}
\begin{remark}[Polynomial-time approximation] Color-coding \cite{alon1995color} has been able to give polynomial-time computable approximations to similar tree-based polynomials of degree $O(\log n)$ (see \cite{mao2024testing}). We expect that this approach can be extended to the present case. One may construct $\mathcal{U}_k$ by listing all rooted trees with $k$ edges and attaching two of these to a triangle to create a BUG.
Then, following~\cite{mao2024testing}, deleting one edge from a unicyclic subgraph of a BUG yields either a smaller such object and tree components or all tree components (the latter case matches~\cite{mao2024testing}). We leave a full algorithmic treatment to future work.
\end{remark}

\subsection{Weak Testing}
We next consider weak testing, where the mean gap of a statistic only needs to be of the same order as its standard deviation. This leads to a different scale from strong testing. In the multigraph convention used for PSM, diagonal entries correspond to loops, and loops count as cycles. Thus, the trace statistic $\operatorname{Tr}(Y)=\sum_{i=1}^n Y_{ii}$ is indexed by the smallest connected cyclic multigraph: a single loop. Since this statistic aggregates $n$ independent diagonal contributions, its signal-to-noise ratio is of constant order when $
\lambda\asymp(\rho\sqrt n)^{-1}.$
The theorem below gives the corresponding weak-testing result: below the scale $(\rho\sqrt n)^{-1}$, the degree-$D$ advantage is $1+o(1)$, while at the scale $(\rho\sqrt n)^{-1}$ a simple trace statistic already weakly separates the two planted laws.
\thmPSMweak*
\begin{proof}[Proof of Theorem~\ref{thm:weak_PSM} \textnormal{(i)}]
By~\Cref{prop:adv_bound_connected} and the construction of $u$, we have
\[
\Adv_{\le D}^2(\PP,\QQ)
\le \|u\|^2
=1+\sum_{\alpha\gamma\in\goodJ:\alpha\neq\varnothing}u_{\alpha\gamma}^2 .
\]
Grouping by the number $v$ of vertices, the number $C$ of connected components, and the component sizes $v_1,\dots,v_C$, the same enumeration as in the strong-testing proof gives
\begin{align}
\Adv_{\le D}^2(\PP,\QQ)\le 1+\sum_{v=1}^D ({ n}\tilde\rho^2\lambda^2)^{v}
\sum_{C=1}^v \frac{L^{2C}}{C!}
\sum_{\substack{v_1+\cdots+v_C=v\\ v_i\ge1}}
\prod_{i=1}^C \frac{v_i^{ v_i}}{v_i!}
\sum_{k=0}^D (L\lambda)^{2k}\binom{ v^2+k}{k},
\label{eq:weak-adv-joint-sum}
\end{align}

As in the strong-testing proof, we have that
\[
 \sum_{k=0}^D (L\lambda)^{2k}\binom{v^2+k}{k} \leq 2+{v^2}\exp(2L^2v/C_0).
\]
Choosing $C_0$ sufficiently large and enlarging the constants below, this factor is absorbed into the final $v$-sum. Substituting this into \eqref{eq:weak-adv-joint-sum}, and carrying out the remaining sums over $C$ and $v_1,\dots,v_C$ as in the strong-testing proof gives constants $K,\kappa>0$, depending only on $L$, such that
\begin{equation}\label{eq:weak-adv-final-sum}
\Adv_{\le D}^2(\PP,\QQ)
\le
1+\sum_{v=1}^D \big(Kn\tilde\rho^2\lambda^2\big)^v e^{\kappa\sqrt v}.
\end{equation}
Since $\rho=o(1)$, we have $\tilde\rho=\rho(1+o(1))$. By the assumption
$\lambda=o((\rho\sqrt n)^{-1})$, we have $
x_n:=Kn\tilde\rho^2\lambda^2=o(1).$
Thus, for all sufficiently large $n$, $x_n\le e^{-2(\kappa+1)}$. For all $v\ge1$,
\[
x_n^v e^{\kappa\sqrt v}
\le
\exp(-2(\kappa+1)v+\kappa\sqrt v)
\le
e^{-v},
\]
since $\sqrt v\le v$.
The dominating sequence $\{e^{- v}\}_{v\ge1}$ is summable, and for each fixed $v$,
$x_n^v e^{\kappa\sqrt v}\to0.$ Hence, after summing over $v$, dominated convergence gives a total contribution of $o(1)$. Applying this to \eqref{eq:weak-adv-final-sum} gives $\Adv_{\le D}^2(\PP,\QQ)\le1+o(1),$
and hence, \[\Adv_{\le D}(\PP,\QQ)=1+o(1). \qedhere\]
\end{proof}
\noindent We now prove the weak-testing upper bound, where $f$ is the degree-$1$ trace statistic.
\begin{proof}[Proof of Theorem~\ref{thm:weak_PSM} \textnormal{(ii)}]
  Let $f(Y):=\sum_{i=1}^n Y_{ii}$ be the trace statistic. We compute its mean under $\QQ$; the corresponding expression under $\PP$ is obtained by replacing $\ell$ with $\ell'$.
\begin{align*}
\mathbb{E}_{\mathbb{Q}}[f(Y)] &= \mathbb{E}_{\mathbb{Q}}\Big[\sum_{i=1}^n 
\Big( \ell \lambda \sum_{c=1}^\ell \mathds{1}[\Theta_i = c] + Z_{ii} \Big) \Big]=\ell \lambda n \rho.
\end{align*}
For the second moment, we have
\begin{align*}
    &\mathbb{E}_\mathbb{Q} [f^2(Y)] = \mathbb{E}_\mathbb{Q}
    \Big[\Big(\sum_{i=1}^n \ell\lambda \sum_{c=1}^{\ell}\mathds{1}[\Theta_i = c] + Z_{ii} \Big)^2\Big]\\
    &= \mathbb{E}_\mathbb{Q}
    \Big[\ell^2 \lambda^2 \sum_{i=1}^n \sum_{j=1}^n \Big(\sum_{c=1}^{\ell}\mathds{1}[\Theta_i = c]
    \sum_{c=1}^{\ell}\mathds{1}[\Theta_j = c] \Big)
    + \sum_{i=1}^n Z_{ii}^2\Big].
\end{align*}
Here the first sum corresponds to the terms with $i=j$, the second to the terms with $i\neq j$, while the final term comes from the noise matrix,
\begin{align*}
    &= \mathbb{E}_{\mathbb{Q}} 
    \Big[\ell^2\lambda^2 \Bigg( \sum_{i=1}^n \sum_{c=1}^{\ell}\mathds{1}[\Theta_i = c]
    +  \sum_{i=1}^n \sum_{\substack{j=1 \\ j \neq i}}^n \Big( \sum_{c=1}^{\ell}\mathds{1}[\Theta_i = c]
    \sum_{c=1}^{\ell}\mathds{1}[\Theta_j = c] \Big) \Bigg)
    + \sum_{i=1}^n Z_{ii}^2\Big]\\
    &= \ell^2\lambda^2  \Big(n \ell \frac{\rho}{\ell} + 
    n(n-1) \ell^{2}\frac{\rho^2}{\ell^2} \Big)
    + n.
\end{align*}
Therefore, 
\begin{align*}
    \operatorname{Var}_{\mathbb{Q}}[f(Y)] &=  \ell^{2}\lambda^2 n \rho + \ell^{2}\lambda^2 n(n-1) \rho^2 + n - \ell^{2} \lambda^2 n^2 \rho^2\\
    &= \ell^{2}\lambda^2n\rho(1-\rho) + n.
\end{align*}
The same computation under $\mathbb{P}$ gives $\operatorname{Var}_\mathbb{P}[f(Y)]=
(\ell')^2\lambda^2 n\rho(1-\rho)+n.$
Since $\ell$ and $\ell'$ are fixed constants, the same bounds below apply to both variances.
Without loss of generality, assume $\ell > \ell^{\prime}$. Then we require, for some $\zeta > 0$,
\[
\mathbb{E}_{\mathbb{Q}}[f(Y)] - \mathbb{E}_{\mathbb{P}}[f(Y)] =(\ell - \ell^{\prime}) \lambda n \rho  \geq  
\sqrt{\frac{\zeta}{2}}\sqrt{\ell^{2}\lambda^2n\rho(1-\rho) + n}.
\]
Equivalently,
\[
(\ell - \ell^{\prime})^2 \lambda^2 n^2 \rho^2    \geq \frac{\zeta}{2} (\ell^{2}\lambda^2n\rho(1-\rho) + n).
\]
Thus, it suffices that
\[
\lambda^2 \geq \frac{\zeta}{(\ell-\ell^{\prime})^2 n \rho^2}
\quad \text{ and } \quad
\frac{(\ell - \ell^{\prime})^2 n \rho}{\ell^{2}(1-\rho)} \geq \zeta.
\]
 For the first inequality, the hypothesis $\lambda=\Omega((\rho\sqrt{n})^{-1})$ implies that there exists some $\varepsilon>0$ such that $\lambda\ge\varepsilon(\rho\sqrt{n})^{-1}$. By choosing $\zeta = \varepsilon^2(\ell-\ell^{\prime})^2$, the first inequality follows immediately. The hypothesis $n\rho = \omega(1)$ implies that the second inequality is satisfied for all sufficiently large $n$.
\end{proof}

\section{Planted Dense Subgraph Model}\label{sec:pds-testing}
We now prove the strong and weak planted testing results for the planted dense
subgraph model. After the natural normalization in \eqref{eq:pds-lambda}, the resulting low-degree thresholds match those of the planted submatrix model.
\begin{definition}[$\ell$-Planted Dense Subgraph]\label{def:PDS}
Let \(n,\ell\in\mathbb N\), \(\rho\in(0,1)\), and \(0\le p_0\le p_1\le 1\).
Set $\delta:=p_1-p_0$
and assume \(p_0+\ell\delta\le 1\). We define the
\(\ell\)-planted dense subgraph model
\(\mathbb P_{\mathrm{PDS}}(n,\ell,\rho,p_0,p_1)\) as follows.
The latent labels \(\Theta=(\Theta_1,\dots,\Theta_n)\) are drawn independently
from \(\{0,1,\dots,\ell\}\), with
\[
\mathrm{Pr}(\Theta_i=c)=\frac{\rho}{\ell}
\quad\text{for each }c\in[\ell],
\qquad
\mathrm{Pr}(\Theta_i=0)=1-\rho.
\]
Given \(\Theta\), the edge variables \(\{Y_{ij}\}_{1\le i<j\le n}\) are
independent and satisfy
\[
Y_{ij}\sim
\mathrm{Ber} \left(
p_0+\ell\delta\,\mathds 1\{\Theta_i=\Theta_j\in[\ell]\}
\right).
\]
\end{definition}
Equivalently, the within-community edge probability in the $\ell$-model is
$q_\ell:=p_0+\ell\delta$. The notation $p_1$ is used only to parametrize
the normalized signal $\delta=p_1-p_0$. Thus, $p_1$ is not the
within-community edge probability for $\ell>1$. 
\subsection{Strong Testing}
Define the normalized signal strength
\begin{equation}\label{eq:pds-lambda}
\lambda := \frac{p_1-p_0}{\sqrt{p_0(1-p_0)}}.    
\end{equation}
The next theorem gives matching low-degree lower and upper bounds in terms of $\lambda$, $\rho$, and the polynomial degree $D$.

\begin{theorem}[Strong testing: PDS]\label{thm:strong_PDS}
Given parameters $n,\ell,\ell',\rho,p_0,p_1$, with
$\ell,\ell'$ fixed distinct positive integers, $\rho\in(0,1)$, and
$p_0,p_1\in(0,1)$, define
$\QQ:=\mathbb{P}_{\mathrm{PDS}}(n,\ell,\rho,p_0,p_1)$ and
$\PP:=\mathbb{P}_{\mathrm{PDS}}(n,\ell',\rho,p_0,p_1)$.
Set $\lambda$ as in~\eqref{eq:pds-lambda}, and assume
$p_0+\max\{\ell,\ell'\}(p_1-p_0)<1$. For any constant $\varepsilon>0$, there exists a constant $C_0\equiv C_0(\ell,\ell',\varepsilon)>0$ such that the following hold.
\begin{enumerate}
  \item[(i)] \emph{(Lower bound)}. If
  \[
  \lambda \le (1-\varepsilon)\big(\rho\sqrt{en}\big)^{-1},
  \qquad
  D \le \lambda^{-2}/C_0, \qquad \rho=o(1)
  \]
  then $\Adv_{\le D}(\PP,\QQ)=O(1)$.

  \item[(ii)] \emph{(Upper bound)}. If
  \[
  \lambda \ge (1+\varepsilon)\big(\rho\sqrt{en}\big)^{-1},
  \quad
  \rho=\omega(n^{-1}\log^7 n),
  \quad
  \rho=o(\log^{-7}n),
  \quad
  p_0=\omega(n^{-1}\log^{14}n),
  \]
 then there exists a polynomial of degree at most $C_0\log n$ that strongly separates $\mathbb{P}$ and $\mathbb{Q}$.
\end{enumerate}
\end{theorem}

\subsubsection{Lower Bound}\label{app:PDS:lower-bound}
The PDS lower bound is obtained by applying the general certificate framework explained in Section \ref{sec:framework} to the Bernoulli edge observation model. The main changes relative to PSM are the choice of observation basis and the verification of the corresponding moment and
cancellation identities. The moment formulas in Lemma~\ref{lem:PDS-cMd} again identify the good graphs as those whose connected components contain cycles. The remaining norm bound is then controlled by the same good-multigraph enumeration used in the PSM proof, after the normalization given in \eqref{eq:pds-lambda}.
\subsubsection*{Setting $\phi, \psi$ and establishing properties}
We choose the observation basis $\{\phi_\alpha\}_\alpha$ and the extended-space orthonormal collection $\{\psi_{\beta\gamma}\}_{\beta\gamma}$ needed to apply~\Cref{prop:adv_bound_connected}. Define
\begin{equation}\label{eq:phi_def_PDS}
\phi_\alpha(Y)=(Y-p_0)^\alpha, \qquad \alpha \in \cI:=\Bigl\{\alpha\in\{0,1\}^{\binom{[n]}{2}}:\ |\alpha|\le D\Bigr\}.
\end{equation}

Since the observations are Bernoulli, every polynomial function of degree at most $D$ has a unique multilinear representative, so this is a basis for $\mathbb R[Y]_{\le D}$ on the observation space.
For $\beta \in \{0,1\}^{\binom{[n]}{2}} $ and $\gamma\in\{0,1\}^n$, choose
{\small\begin{equation}\label{eq:psi-PDS}
\psi_{\beta \gamma}(Y,\Theta)=\prod_{\substack{1\leq i<j\leq n:\\ \Theta_i=\Theta_j\in [\ell]}}\left(\frac{Y_{ij}-q_\ell}{\sqrt{q_\ell(1-q_\ell)}}\right)^{\beta_{ij}}\prod_{\substack{1\leq i<j\leq n:\\ \neg(\Theta_i=\Theta_j\in[\ell])}}\left(\frac{Y_{ij}-p_0}{\sqrt{p_0(1-p_0)}}\right)^{\beta_{ij}}\prod_{i=1}^n \left(\frac{\mathds{1}{[\Theta_i\neq 0]}-\rho}{\sqrt{\rho(1-\rho)}}\right)^{\gamma_i}.
\end{equation}}
\begin{lemma}\label{lem:orthornormal-psi:PDS}
The collection $\{\psi_{\beta\gamma}\}_{\beta \in \{0,1\}^{\binom{[n]}{2}},\gamma\in \{0,1\}^n}$ is orthonormal.
\end{lemma}
\begin{proof}
    Fix $\beta,\beta' \in \{0,1\}^{\binom{[n]}{2}}$ and $\gamma,\gamma' \in \{0,1\}^n$. We show that $\E_\QQ[\psi_{\beta\gamma}\psi_{\beta'\gamma'}]=\mathds{1}\{\beta=\beta',\gamma=\gamma'\}$. Conditionally on $\Theta$, the edge variables are independent, and 
    \[Y_{ij}|\Theta \sim \begin{cases}
        \mathrm{Ber}(q_\ell), & \Theta_i=\Theta_j \in [\ell],\\
        \mathrm{Ber}(p_0), & \neg(\Theta_i=\Theta_j\in[\ell]).
    \end{cases}\]
Therefore,
\begin{align*}
    \E_\QQ\Bigg[\prod_{\substack{1\leq i<j\leq n:\\ \Theta_i=\Theta_j\in [\ell]}}\left(\frac{Y_{ij}-q_\ell}{\sqrt{q_\ell(1-q_\ell)}}\right)^{\beta_{ij}+\beta'_{ij}}\!\!\!\!\!\!\!\prod_{\substack{1\leq i<j\leq n:\\ \neg(\Theta_i=\Theta_j\in[\ell])}}\left(\frac{Y_{ij}-p_0}{\sqrt{p_0(1-p_0)}}\right)^{\beta_{ij}+\beta'_{ij}} \Bigg| \Theta\Bigg]=\mathds{1}[\beta=\beta'].
\end{align*}
By the tower property, 
\begin{align*}
    \E_\QQ[\psi_{\beta\gamma}\psi_{\beta'\gamma'}]=\mathds{1}\,\{\beta=\beta'\} \E_{ \QQ_\Theta}\prod_{i=1}^n \left(\frac{\mathds{1}[\Theta_i\neq 0]-\rho}{\sqrt{\rho(1-\rho)}}\right)^{\gamma_i+\gamma'_i}=\mathds 1 \{\beta=\beta',\gamma=\gamma'\},
\end{align*}
where the last equality uses that $\mathds 1 [\Theta_i\neq 0]\overset{\mathrm{iid}}{\sim} \mathrm{Ber}(\rho)$.
    \end{proof}

\noindent The following lemma records the forms of \(c_\alpha\), \(d_\alpha\), and \(M_{\beta\gamma,\alpha}\). Note,
\(M_{\beta\gamma,\alpha}=0\) unless \(\beta\subseteq\alpha\).
\begin{lemma}\label{lem:PDS-cMd}
With \(\phi_\alpha\) and \(\psi_{\beta\gamma}\) as defined in \eqref{eq:phi_def_PDS} and \eqref{eq:psi-PDS}, respectively, let
\[
c_\alpha:=\E_{\PP}[\phi_\alpha(Y)],
\qquad
d_\alpha:=\E_{\QQ}[\phi_\alpha(Y)],
\qquad
M_{\beta\gamma,\alpha}:=\E_{\QQ} \left[\phi_\alpha(Y)\psi_{\beta\gamma}(Y,\Theta)\right].
\]
Then
\[
c_\alpha
=
(\ell')^{|\mathcal{C}(\alpha)|+|\alpha|-|V(\alpha)|}\,
\delta^{|\alpha|}\rho^{|V(\alpha)|}, \qquad 
d_\alpha
=
\ell^{|\mathcal{C}(\alpha)|+|\alpha|-|V(\alpha)|}\,
\delta^{|\alpha|}\rho^{|V(\alpha)|},
\]
and
\begin{align*}
M_{\beta\gamma,\alpha}
&=
\mathds 1\{\beta\subseteq\alpha\}
(\ell\delta)^{|\alpha|-|\beta|}
\bigl(p_0(1-p_0)\bigr)^{|\beta|/2}\\
& \times \E_{\QQ_\Theta} \left[
\left(\frac{q_\ell(1-q_\ell)}{p_0(1-p_0)}\right)^{s(\beta;\Theta)/2}
\mathds 1\{A_e=1 \ \forall e\in\alpha\setminus\beta\}
\prod_{i=1}^n
\left(\frac{\mathds 1\{\Theta_i\neq 0\}-\rho}{\sqrt{\rho(1-\rho)}}\right)^{\gamma_i}
\right],
\end{align*}
where for edge $e=(i,j)$,
\begin{equation}\label{eq:pds-ae-sbeta}
A_e:=\mathds 1\{\Theta_i=\Theta_j\in[\ell]\} 
\qquad \text{and} \qquad
s(\beta;\Theta):=\sum_{e\in\beta}A_e.
\end{equation}
\end{lemma}
    \begin{proof} We first compute $c_\alpha=\E_\PP[\phi_\alpha(Y)]$. Conditioning on $\Theta$, we have 
\begin{align*} \mathbb E_{\PP}[\phi_\alpha(Y)\mid \Theta] &=\mathbb E_{\PP} \left[\prod_{(i,j)\in E(\alpha)}(Y_{ij}-p_0)\ \Big|\ \Theta\right] =\prod_{(i,j)\in E(\alpha)}\mathbb E_{\PP}[Y_{ij}-p_0\mid \Theta]. \end{align*} Since $Y_{ij}\mid\Theta\sim\mathrm{Ber} \big(p_0+\ell'\delta\,\mathds 1\{\Theta_i=\Theta_j\in[\ell']\}\big)$ under $\PP$,  \begin{align*} \mathbb E_{\PP}[Y_{ij}-p_0\mid\Theta] &=\mathbb E_{\PP}[Y_{ij}\mid\Theta]-p_0 =\ell'\delta\,\mathds 1\{\Theta_i=\Theta_j\in[\ell']\}. \end{align*} 
Hence 
\begin{align*}
\E_{\PP}[\phi_\alpha(Y)\mid \Theta]
&=(\ell'\delta)^{|\alpha|}
\prod_{(i,j)\in E(\alpha)}
\mathds 1\{\Theta_i=\Theta_j\in[\ell']\}.
\end{align*}
Taking expectation with respect to $\Theta$ gives
\begin{align*}
c_\alpha=(\ell'\delta)^{|\alpha|} \E_{\PP_\Theta}\prod_{(i,j)\in E(\alpha)} \mathds 1\{\Theta_i=\Theta_j\in[\ell']\}&=(\ell'\delta)^{|\alpha|}\prod_{\tau \in \mathcal{C}(\alpha)}
\sum_{c=1}^{\ell'}\left(\frac{\rho}{\ell'}\right)^{|V(\tau)|}\\
&=(\ell'\delta)^{|\alpha|}\prod_{\tau \in \mathcal{C}(\alpha)}\frac{\rho^{|V(\tau)|}}{(\ell')^{|V(\tau)|-1}}\\
&=(\ell')^{|\mathcal{C}(\alpha)|+|\alpha|-|V(\alpha)|}\delta^{|\alpha|}\rho^{|V(\alpha)|}.
\end{align*}

\noindent Recall that $d_\alpha=\E_\QQ[\phi_\alpha(Y)]$. By the same computation used for $c_\alpha=\E_\PP[\phi_\alpha(Y)]$, and using that $\QQ$ differs from $\PP$ only by replacing $\ell'$ with $\ell$, we obtain
\[d_\alpha=\ell^{|\mathcal{C}(\alpha)|+|\alpha|-|V(\alpha)|}\delta^{|\alpha|}\rho^{|V(\alpha)|}.\]

\noindent We now compute $M_{\beta\gamma,\alpha}$.
Note the following \begin{align*}
    M_{\beta\gamma,\alpha}
:= \E_{\mathbb Q}\big[\phi_\alpha(Y)\psi_{\beta\gamma}(Y,\Theta)\big]
= \E_{\QQ_\Theta}\Big[\,\E_\QQ\big[\phi_\alpha(Y)\psi_{\beta\gamma}(Y,\Theta)\mid\Theta\big]\,\Big].
\end{align*}
To simplify notation, let us define the following recalling $q_\ell=p_0+\ell\delta$, 
\begin{align*}
    h_i(\Theta) := \frac{\mathds 1[\Theta_i\neq 0]-\rho}{\sqrt{\rho(1-\rho)}}, \quad i \in [n], \qquad
    g_e(Y,\Theta):=
\begin{cases}
\dfrac{Y_e-q_\ell}{\sqrt{q_\ell(1-q_\ell)}} & \text{if }A_e=1\\[0.8em]
\dfrac{Y_e-p_0}{\sqrt{p_0(1-p_0)}} & \text{if }A_e=0
\end{cases},
\end{align*}
where $A_e$ is as defined in \eqref{eq:pds-ae-sbeta}. Note that $A_e=1$ implies $Y_e\mid\Theta\sim\mathrm{Ber}(q_\ell)$
while $A_e=0$ implies $Y_e\mid\Theta\sim\mathrm{Ber}(p_0)$,
so $g_e$ is mean-zero with unit variance conditional on $\Theta$ in both cases.
If $\beta\not\subseteq\alpha$ then $M_{\beta\gamma,\alpha}=0$.
If there exists $e\in\beta\backslash\alpha$, then $\phi_\alpha(Y)$ is independent of~$Y_e$.
Moreover, conditional on $\Theta$,
$
\E_{\QQ}[g_e(Y,\Theta)\mid\Theta]=0
$
and by conditional independence
of edges given $\Theta$ this forces the full conditional expectation to vanish. Hence
$M_{\beta\gamma,\alpha}=0$ unless $\beta\subseteq\alpha$. Henceforth assume $\beta\subseteq\alpha$ and write $E_0:=\alpha\backslash\beta$.
\medskip
\noindent Given $\Theta$, edges are independent, so
\[
\E_{\QQ}\Big[\phi_\alpha(Y)\prod_{e\in\beta}g_e(Y,\Theta)\ \Big|\ \Theta\Big]
=
\prod_{e\in E_0}\E_{\QQ}[(Y_e-p_0)\mid\Theta]\cdot
\prod_{e\in \beta}\E_{\QQ}[(Y_e-p_0)g_e(Y,\Theta)\mid\Theta].
\]
Since $\E_{\QQ}[Y_e\mid\Theta]=p_0+\ell\delta\, A_e$, we get $\E_{\QQ}[(Y_e-p_0)\mid\Theta]=\ell\delta\, A_e$. Hence,
\[
\prod_{e\in E_0}\E_{\QQ}[(Y_e-p_0)\mid\Theta]
=(\ell\delta)^{|\alpha|-|\beta|}\,\mathds 1\{A_e=1,\ \forall e\in\alpha\setminus\beta\}.
\]
For $e\in\beta$, if $A_e=0$ then $Y_e\mid\Theta\sim\mathrm{Ber}(p_0)$ and
$g_e=(Y_e-p_0)/\sqrt{p_0(1-p_0)}$, giving
$\E_{\QQ}[(Y_e-p_0)g_e\mid\Theta]=\sqrt{p_0(1-p_0)}$; if $A_e=1$ then
$Y_e\mid\Theta\sim\mathrm{Ber}(q_\ell)$ and $g_e=(Y_e-q_\ell)/\sqrt{q_\ell(1-q_\ell)}$, giving
\[
\E_\QQ[(Y_e-p_0)g_e\mid\Theta]
= \frac{\E_\QQ[(Y_e - q_\ell + \ell\delta)(Y_e - q_\ell)\mid\Theta]}{\sqrt{q_\ell(1-q_\ell)}}
= \frac{q_\ell(1-q_\ell)}{\sqrt{q_\ell(1-q_\ell)}}
= \sqrt{q_\ell(1-q_\ell)}.
\]
Writing $s(\beta;\Theta):=\sum_{e\in\beta}A_e$, yields
\[
\prod_{e\in\beta}\E_{\QQ}[(Y_e-p_0)g_e(Y,\Theta)\mid\Theta]
=
\big(p_0(1-p_0)\big)^{|\beta|/2}
\left(\frac{q_\ell(1-q_\ell)}{p_0(1-p_0)}\right)^{s(\beta;\Theta)/2}.
\]
Averaging over $\Theta$ and including the vertex factor $\prod_{i\in\gamma}h_i(\Theta)$ gives
\begin{align*}
M_{\beta\gamma,\alpha}
&=
\mathds 1\{\beta\subseteq\alpha\} 
(\ell\delta)^{|\alpha|-|\beta|} 
\big(p_0(1-p_0)\big)^{|\beta|/2} \\
&\quad\cdot
\E_{\QQ_\Theta} \left[
\left(\frac{q_\ell(1-q_\ell)}{p_0(1-p_0)}\right)^{s(\beta;\Theta)/2}
\prod_{i=1}^n
\left(\frac{\mathds 1\{\Theta_i\neq 0\}-\rho}{\sqrt{\rho(1-\rho)}}\right)^{\gamma_i}
\mathds 1\{A_e=1\ \forall e\in\alpha\setminus\beta\}\right],
\end{align*}
where $s(\beta;\Theta):=\sum_{e\in\beta} A_e$ and $A_e:=\mathds 1\{\Theta_i=\Theta_j\in[\ell]\}\in\{0,1\}$.
\end{proof}

\subsubsection*{Excluding bad terms}
In the PDS setting the graph indices are simple graphs, rather than multigraphs. As in the PSM case, the good graphs in the sense of Definition~\ref{def:good_graphs} are exactly those whose connected components contain cycles. This follows directly from the expressions for $c_\alpha$ and $d_\alpha$ in Lemma~\ref{lem:PDS-cMd}.

\subsubsection*{Constructing $u$}
For PDS, the entries of $M_{\beta\gamma,\alpha}$ include an additional conditional-variance factor coming from the orthonormalization. Lemma~\ref{lem:pds-cancellation} shows that this factor is exactly compensated in the certificate constraint.
\begin{proposition}\label{prop:u-PDS}
    Define $u_{\varnothing\varnothing}=1$. For every nonempty $\alpha\in\goodI$ with connected components
$\alpha_1,\ldots,\alpha_C$, and every $\gamma\subseteq V(\alpha)$, let
    \begin{equation*}
        u_{\alpha\gamma}=\left(-\sqrt{\frac{\rho}{1-\rho}}\right)^{|\gamma|} \frac{\prod_{i=1}^C (c_{\alpha_i}-d_{\alpha_i})}{(p_0(1-p_0))^{|\alpha|/2}}, \qquad \forall \alpha\gamma\in \goodJ.
    \end{equation*}
    Set $u_{\alpha\gamma}=0$ for $\alpha\gamma\notin\goodJ$. Then $u$ is supported on $\goodJ$, is component-consistent, and satisfies
\[
\sum_{\beta\gamma\in\goodJ}
u_{\beta\gamma}M_{\beta\gamma,\alpha}
=
c_\alpha \qquad \text{for every connected} \ \alpha \in \goodI.
\]
\end{proposition}
We use the following lemma to prove Proposition~\ref{prop:u-PDS}.
\begin{lemma}\label{lem:pds-cancellation}
    Let $\alpha\in\widehat{\mathcal I}$ be connected and nonempty, and let $\beta\subseteq\alpha$ with $\beta\in\widehat{\mathcal I}$. Then,
    \[\sum_{\gamma \subseteq V(\beta)}\left(-\sqrt{\frac{\rho}{1-\rho}}\right)^{|\gamma|}M_{\beta\gamma,\alpha}=\begin{cases}
        d_\alpha, & \beta=\varnothing,\\
        \left(p_0(1-p_0)\right)^{|\alpha|/2}, & \beta=\alpha,\\
        0, & \varnothing\neq\beta\lneq\alpha.
    \end{cases}\]
\end{lemma}
\begin{proof}
    Recall the expression for $M_{\beta\gamma,\alpha}$, where $\beta\subseteq\alpha$, from Lemma~\ref{lem:PDS-cMd}. First suppose $\beta=\varnothing$. Then $V(\beta)=\varnothing$, so the only possible choice is $\gamma=\varnothing$. Hence
\[
M_{\varnothing\varnothing,\alpha}=(\ell\delta)^{|\alpha|}
\E_{\QQ_\Theta} \left[\mathds 1\{A_e=1,\ \forall e\in\alpha\}\right].
\]
Since $\alpha$ is connected, the event $\{A_e=1\ \forall e\in\alpha\}$ means that all vertices in $V(\alpha)$ have the same nonzero label in $[\ell]$. Therefore $
\E_{\QQ_\Theta} \left[\mathds 1\{A_e=1\ \forall e\in\alpha\}\right]=
\ell\left(\frac{\rho}{\ell}\right)^{|V(\alpha)|}.$
Thus,
\[
M_{\varnothing\varnothing,\alpha}
=
\ell^{1+|\alpha|-|V(\alpha)|}\delta^{|\alpha|}\rho^{|V(\alpha)|}=d_\alpha,
\]
where the last equality uses the fact that $\alpha$ is connected. Next, suppose $\beta=\alpha$. Then $\alpha\setminus\beta=\varnothing$, so the edge indicator is absent. Summing over $\gamma\subseteq V(\alpha)$ gives
\begin{align}
\notag
&\sum_{\gamma\subseteq V(\alpha)}
\left(-\sqrt{\frac{\rho}{1-\rho}}\right)^{|\gamma|}
M_{\alpha\gamma,\alpha} \\
\notag
&=
\left(p_0(1-p_0)\right)^{|\alpha|/2}
\E_{\QQ_\Theta} \left[
\left(\frac{q_\ell(1-q_\ell)}{p_0(1-p_0)}\right)^{s(\alpha;\Theta)/2}
\sum_{\gamma\subseteq V(\alpha)}
\prod_{i\in V(\alpha)}
\left(-\frac{\mathds 1[\Theta_i\neq0]-\rho}{1-\rho}\right)^{\gamma_i}
\right].
\end{align}
Using
\begin{align*}
\sum_{\gamma\subseteq V(\alpha)}
\prod_{i\in V(\alpha)}
\left(-\frac{\mathds 1[\Theta_i\neq0]-\rho}{1-\rho}\right)^{\gamma_i}&= \prod_{i\in V(\alpha)} \left(1-\frac{\mathds 1[\Theta_i\neq0]-\rho}{1-\rho}\right)\\
&=(1-\rho)^{-|V(\alpha)|} \mathds 1\{\Theta_i=0\ \forall i\in V(\alpha)\},
\end{align*}
and noting that \(s(\alpha;\Theta)=0\) on this event, we now have that
\[ \sum_{\gamma\subseteq V(\alpha)}
\left(-\sqrt{\frac{\rho}{1-\rho}}\right)^{|\gamma|}
M_{\alpha\gamma,\alpha} = \left(p_0(1-p_0)\right)^{|\alpha|/2}
(1-\rho)^{-|V(\alpha)|}
\mathrm{Pr}\{\Theta_i=0\ \forall i\in V(\alpha)\}.
\]
Since $\mathrm{Pr}\{\Theta_i=0,\ \forall i\in V(\alpha)\}=(1-\rho)^{|V(\alpha)|},$
we obtain
\[\sum_{\gamma\subseteq V(\alpha)}
\left(-\sqrt{\frac{\rho}{1-\rho}}\right)^{|\gamma|}M_{\alpha\gamma,\alpha}=\left(p_0(1-p_0)\right)^{|\alpha|/2}.
\]
Finally, suppose $\varnothing \neq \beta \subsetneq \alpha$. As above, summing over $\gamma \subseteq V(\beta)$ gives 
\begin{align*}
    &\sum_{\gamma \subseteq V(\beta)}\left(-\sqrt{\frac{\rho}{1-\rho}}\right)^{|\gamma|}M_{\beta\gamma,\alpha}\\
    &=(\ell\delta)^{|\alpha|-|\beta|}
\bigl(p_0(1-p_0)\bigr)^{|\beta|/2}
(1-\rho)^{-|V(\beta)|} \\
&\quad\times\E_{\QQ_\Theta} \left[
\left(\frac{q_\ell(1-q_\ell)}{p_0(1-p_0)}\right)^{s(\beta;\Theta)/2}
\mathds 1\{A_e=1\ \forall e\in\alpha\setminus\beta\}
\mathds 1\{\Theta_i=0\ \forall i\in V(\beta)\}\right].
\end{align*}
Since $\alpha$ is connected, $\beta\neq\varnothing$, and  $\beta\subsetneq\alpha$, some edge of $\alpha\setminus\beta$ is incident to a vertex in \(V(\beta)\). The indicator $\mathds 1\{A_e=1\ \forall e\in\alpha\setminus\beta\}$ requires that vertex
to have a nonzero label, while the last indicator requires every vertex in $V(\beta)$ to have label $0$. These requirements are incompatible, so the expectation is zero. This proves the third case.
\end{proof}
\begin{proof}[Proof of Proposition~\ref{prop:u-PDS}]
Fix a connected nonempty $\alpha\in\goodI$. Since
$M_{\beta\gamma,\alpha}=0$ unless $\beta\subseteq\alpha$, it suffices to sum over $\beta\subseteq\alpha$. Split the sum into the cases $\beta=\varnothing, \beta=\alpha$, and $\varnothing\neq\beta\subsetneq\alpha$, and recall their contributions from Lemma~\ref{lem:pds-cancellation}. Therefore, using the definition of $u$ and the fact that $\alpha$ is connected,
\begin{align*}
\sum_{\beta\gamma\in\goodJ}
u_{\beta\gamma}M_{\beta\gamma,\alpha}
&=
u_{\varnothing\varnothing}d_\alpha
+
\frac{c_\alpha-d_\alpha}{(p_0(1-p_0))^{|\alpha|/2}}\sum_{\gamma\subseteq V(\alpha)}
\left(-\sqrt{\frac{\rho}{1-\rho}}\right)^{|\gamma|}
M_{\alpha\gamma,\alpha}
+
0 \\
&=
d_\alpha
+
\frac{c_\alpha-d_\alpha}{(p_0(1-p_0))^{|\alpha|/2}}
(p_0(1-p_0))^{|\alpha|/2}=
c_\alpha.
\end{align*}
This proves the claim for connected $\alpha\in\goodI$.   
\end{proof}
\subsubsection*{Putting it all together}
Combining Proposition~\ref{prop:u-PDS} and Lemma~\ref{lem:pds-cancellation}, we are ready to prove part (i) of Theorem~\ref{thm:strong_PDS}.
\begin{proof}[Proof of Theorem~\ref{thm:strong_PDS} \textnormal{(i)}]
Fix a nonempty $\alpha\in\goodI$, and let
$\alpha_1,\dots,\alpha_C$ be its connected components.
For each component $\alpha_i$,
\[
c_{\alpha_i}-d_{\alpha_i}=\left(\ell'^{1+|\alpha_i|-|V(\alpha_i)|}-
\ell^{1+|\alpha_i|-|V(\alpha_i)|}\right)\delta^{|\alpha_i|}\rho^{|V(\alpha_i)|}.
\]
Hence, with $L:=2\max\{\ell,\ell'\}$ and $\lambda$ as in \eqref{eq:pds-lambda},
\[\frac{(c_{\alpha_i}-d_{\alpha_i})^2}{(p_0(1-p_0))^{|\alpha_i|}}\le L^{2(1+|\alpha_i|-|V(\alpha_i)|)} \lambda^{2|\alpha_i|}\rho^{2|V(\alpha_i)|}.\]
Since $u$ is defined multiplicatively over connected components, summing over $\gamma\subseteq V(\alpha)$ gives
\[\|u\|^2\le 1+
\sum_{\alpha\in\widehat{\mathcal I}:\alpha\neq\varnothing}
L^{2(|\mathcal{C}(\alpha)|+|\alpha|-|V(\alpha)|)}\lambda^{2|\alpha|}\tilde\rho^{\,2|V(\alpha)|},
\qquad
\tilde\rho=\frac{\rho}{\sqrt{1-\rho}}.
\]
This is upper-bounded by the norm bound from the PSM lower bound. Here the PDS sum is over simple graphs, so the multigraph enumeration used in the PSM proof gives a valid upper bound. Hence the same application of
\Cref{lemma:nr-multigraphs} gives $\|u\|^2=O(1)$ provided
\[
\lambda\le (1-\varepsilon/2)(\tilde\rho\sqrt{en})^{-1},
\qquad
D\le \lambda^{-2}/C_0.
\]
Since $\tilde\rho=\rho(1+o(1))$ under $\rho=o(1)$, the theorem assumption
$\lambda\le (1-\varepsilon)(\rho\sqrt{en})^{-1}$ implies the first displayed
condition for all sufficiently large $n$. Thus
$\Adv_{\le D}^2\le \|u\|^2=O(1)$, after increasing
$C_0$ if necessary. Hence $\Adv_{\le D}=O(1)$.
\end{proof}

\subsubsection{Upper Bound}\label{sec:PDS:upper-bound-strong}
The strong-testing upper bound uses the BUG statistic built from centered and
normalized Bernoulli edge variables. Define the set of BUGs $\mathcal U_k$ as in Section~\ref{sec:upperbound_PSM}, and denote the BUG polynomial for the planted dense subgraph as follows:
\begin{equation}\label{def:BUG_f_PDS}
    f(Y)=\sum_{\alpha \in \mathcal U_k}\prod_{(i,j)\in \alpha}\frac{Y_{ij}-p_0}{\sqrt{p_0(1-p_0)}}.
\end{equation}

We now compute the moment quantities associated with this $f$. First, note that,
\begin{equation}\label{eq.ExpForPDS}
    \E_\PP[f(Y)]=|\mathcal U_k|\ell' \rho^D \left(\frac{p_1-p_0}{\sqrt{p_0(1-p_0)}}\right)^D.
\end{equation}
Next, we bound the variance, beginning with the second moment,
\begin{align}\label{eq.defn_Palphabeta}
\E_\QQ[f(Y)^2]&=\sum_{\alpha,\beta\in\mathcal U_k} \E_\QQ\left[\left(\prod_{(i,j)\in \alpha}\frac{Y_{ij}-p_0}{\sqrt{p_0(1-p_0)}}\right)\left(\prod_{(i,j)\in \beta}\frac{Y_{ij}-p_0}{\sqrt{p_0(1-p_0)}}\right)\right]=:\sum_{\alpha,\beta \in \mathcal{U}_k} P_{\alpha\beta}.
\end{align}
For a pair $(\alpha,\beta) \in \mathcal{U}_k$, recall the definitions of $m_\triangle, m_\cap$ from just before Lemma~\ref{lem:second-moment-not-equal}.
Define
\begin{equation}\label{eq.defn_eta_lambda_PDS}
    \tilde\eta:=\frac{\ell\delta}{p_0(1-p_0)} \quad \mathrm{and}\quad \lambda:=\frac{\delta}{\sqrt{p_0(1-p_0)}}.
\end{equation}

\begin{lemma}\label{lem:pds-second-moment}Let $\delta\geq 0$ and let $\tilde\eta, \lambda$ be as in \eqref{eq.defn_eta_lambda_PDS}.
For $\alpha,\beta \in \mathcal{U}_k$, we have the following bounds on $P_{\alpha\beta}$ (defined in~\eqref{eq.defn_Palphabeta}). 
\begin{itemize}
    \item If $\alpha =\beta$ then $P_{\alpha\beta}\leq(\tilde\eta\rho+1)^{|\alpha|}$.
    \item If $\alpha \neq \beta$ and $\alpha \cap \beta$  contains a cycle, then \begin{equation*}
        P_{\alpha\beta}\leq \ell^{m_\triangle}(\ell\lambda)^{|\alpha \triangle \beta|}
 \left(\frac{\rho}{\ell}\right)^{|V(\alpha \triangle \beta)|}
 (\tilde\eta + 1)^{b-m_\cap+1}
 (\tilde\eta \rho + 1)^{|\alpha \cap \beta| - (b - m_\cap)-1}.
    \end{equation*}
    \item If $\alpha \neq \beta$ and $\alpha \cap \beta$ is a forest, then
    \begin{equation*}
        P_{\alpha\beta}\leq \ell^{m_\triangle}(\ell\lambda)^{|\alpha \triangle \beta|}
 \left(\frac{\rho}{\ell}\right)^{|V(\alpha \triangle \beta)|}
 (\tilde\eta + 1)^{b-m_\cap}
 (\tilde\eta \rho + 1)^{|\alpha \cap \beta| - (b - m_\cap)}.
    \end{equation*}
\end{itemize}
\end{lemma}

The proof of Lemma~\ref{lem:pds-second-moment} is deferred to Section~\ref{sec:proof:lem:pds-second-moment}.  Note that Lemma~\ref{lem:pds-second-moment} is analogous to Lemmas~\ref{lem:second-moment-equal} and~\ref{lem:second-moment-not-equal} for the PSM.

\begin{proof}[Proof of Theorem~\ref{thm:strong_PDS} \textnormal{(ii)}]
By Lemma~\ref{lem:pds-second-moment}, the pairwise moment bounds for the PDS BUG statistic have the same overlap form as the PSM bounds in Lemmas~\ref{lem:second-moment-equal} and~\ref{lem:second-moment-not-equal}, with $\eta=(\ell\lambda)^2$
replaced by $\tilde\eta=\frac{\ell\delta}{p_0(1-p_0)}.$ After this replacement, the rest of the proof of Proposition~\ref{prop:BUG-moments} only uses the combinatorial overlap data of pairs $\alpha,\beta\in\mathcal U_k$, namely $s,m_\cap,m_\triangle,b,w$, together with the decomposition into the cases $\alpha\cap\beta=\varnothing$, $\alpha=\beta$, cyclic core, and forest core. These quantities are unchanged in the PDS setting, since the same family $\mathcal U_k$ is used. Hence the same enumeration gives
\[
\max\{\operatorname{Var}_{\QQ}(f),\operatorname{Var}_{\PP}(f)\}=o(N),
\qquad
N=(k+1)^{-6}(en\lambda\rho)^{2D},
\]
provided
\[
\lambda \geq (1+\varepsilon)(\rho\sqrt{en})^{-1},\qquad
\rho=\omega(n^{-1}\log^7 n),\qquad
\rho(\tilde\eta+1)=o(\log^{-7}n).
\]
Under $\PP$ the same argument is applied with $\ell$ replaced by $\ell’$, which only changes constants since $\ell,\ell’$ are fixed. Moreover, by~\eqref{eq.ExpForPDS},
\[
|\E_\PP f-\E_\QQ f|=|\ell'-\ell|\,|\mathcal U_k|\rho^D\lambda^D=\Omega(\sqrt N).
\]
Hence $f$ strongly separates $\PP$ and $\QQ$ under these conditions. It remains to show that the last condition follows from the stated assumptions $\rho=o(\log^{-7}n)$ and $p_0=\omega(n^{-1}\log^{14}n)$, building on the thinning argument of~\cite[Proof of Theorem~2.4(b)]{sohn2025sharp}.  Independently replace each edge by an independent $\mathrm{Ber}(p_0)$ variable with probability $s\in[0,1]$. Conditional on $\Theta$, this leaves the baseline probability $p_0$ unchanged and replaces $\delta$ by $\delta^{(s)}=(1-s)\delta$, hence replaces $\lambda$ by $(1-s)\lambda$. The constraint $p_0+\max\{\ell,\ell'\}\delta<1$ is preserved since $\delta^{(s)}\le\delta$. Thus, if the result is proved at the reduced signal $\delta^{(s)}$, the corresponding polynomial can be averaged over the independent resampling randomness to obtain a polynomial in the original observations of the same degree. Its mean gap is preserved, and its variance is at most the variance before averaging by conditional Jensen. Hence it suffices to prove the upper bound after reducing $\delta$ so that $\lambda=(1+\varepsilon)(\rho\sqrt{en})^{-1}.$

Recall $\tilde \eta$ from \eqref{eq.defn_eta_lambda_PDS}, which can be written as $\tilde \eta=\frac{\ell\lambda}{\sqrt{p_0(1-p_0)}}$. If $p_0\le 1/2$, then
\[\rho\tilde\eta \le \frac{\sqrt{2}\ell\rho\lambda}{\sqrt{p_0}}
=\frac{\sqrt{2}\ell(1+\varepsilon)}{\sqrt{enp_0}}=o(\log^{-7}n),
\]
using $p_0=\omega(n^{-1}\log^{14}n)$. If $p_0>1/2$, then
$p_0+\max\{\ell,\ell'\}\delta< 1$ implies $p_0+\ell\delta< 1$, and therefore $\tilde\eta=\frac{\ell\delta}{p_0(1-p_0)}\le \frac1{p_0}\le 2.$ Since $\rho=o(\log^{-7}n)$, we obtain $\rho(\tilde\eta+1)=o(\log^{-7}n)$. The same argument applies under $\PP$ with $\ell$ replaced by $\ell'$, since $\ell$ and $\ell'$ are fixed positive integers. \qedhere
\end{proof}

\subsection{Weak Testing}
Weak testing in PDS differs from the PSM case because the model has no diagonal loop statistics.  The theorem below identifies the weak-testing scale to matching order in the low-degree framework: below
$o((\rho\sqrt n)^{-1})$, the degree-$D$ advantage is $1+o(1)$, while at signal strength $\lambda=\Omega\left((\rho\sqrt n)^{-1}\right)$, the signed triangle statistic weakly separates under the stated assumptions.

\begin{theorem}[Weak testing: PDS]\label{thm:Weak_PDS}
Given parameters $n,\ell,\ell',\rho,p_0,p_1$, with 
$\ell,\ell'$ fixed distinct positive integers, $\rho\in(0,1)$, and
$p_0,p_1\in(0,1)$, define
$\QQ:=\mathbb{P}_{\mathrm{PDS}}(n,\ell,\rho,p_0,p_1)$ and
$\PP:=\mathbb{P}_{\mathrm{PDS}}(n,\ell',\rho,p_0,p_1)$.
Set $\lambda$ as in~\eqref{eq:pds-lambda}, and assume
$p_0+\max\{\ell,\ell'\}(p_1-p_0)<1$. There exists a constant $C_0\equiv C_0(\ell,\ell')>0$ such that the following hold.
\begin{enumerate}
  \item[(i)] \emph{(Lower bound)}.
 If
  $$\lambda = o \Big(({\rho}\sqrt{n})^{-1}\Big), \qquad D\leq \lambda^{-2}/C_0, \qquad \rho=o(1)$$ 
  then $\Adv_{\leq D}(\PP,\QQ)=1+o(1).$\\  
  \item[(ii)] \emph{(Upper bound)}.  If
  $$ \lambda=\Omega\Big(({\rho}\sqrt{n})^{-1}\Big), \qquad n\rho = \Omega(1),  \qquad n(p_1-p_0)\rho = \Omega\left(|1-2p_0|\right),
  $$
  then there exists a degree-$3$ polynomial that weakly separates $\PP$ and $\QQ$.
\end{enumerate}
\end{theorem}
\begin{proof}[Proof of Theorem~\ref{thm:Weak_PDS} \textnormal{(i)}]
The PDS lower-bound calculation from Theorem~\ref{thm:strong_PDS}\textnormal{(i)} gives the same norm bound as in the PSM proof after the normalization of $\lambda$ as defined in \eqref{eq:pds-lambda}. The PDS simple-graph sum is upper-bounded by the multigraph enumeration used there. Hence, the proof of Theorem~\ref{thm:weak_PSM}\textnormal{(i)} applies with the same argument giving $\Adv_{\le D}(\PP,\QQ)=1+o(1)$
whenever $\lambda=o((\rho\sqrt n)^{-1})$ and $D\le \lambda^{-2}/C_0$.
\end{proof}
\begin{proof}[Proof of Theorem~\ref{thm:Weak_PDS} \textnormal{(ii)}] The proof proceeds by introducing the statistic $\widehat{R}$, i.e.\ the signed triangle count, as in~\cite{bubeck2016testing,rush2023easier}, defined as follows: 
\[f(Y):= \widehat{R}
    = \sum_{1 \le i < j < k \le n} \widehat{R}_{ijk}\]
where for each pair \(1 \le i < j \le n\), we set
\[
R_{ij} = Y_{ij} - p_0 \quad \text{and} \quad \widehat{R}_{ijk} = R_{ij} R_{jk} R_{ik}.
\]
\noindent It suffices to carry out the variance bound under $\QQ$, since the same
calculation under $\PP$ is obtained by replacing $\ell$ with $\ell'$. In what follows, we omit the boundary indices \(1\) and \(n\) from the notation whenever the intended range is clear. Then, 
\begin{equation}\label{eq:R_signed}
\mathbb{E}_{\mathbb Q}[\widehat{R}] =  \sum_{i < j < k} \mathbb{E}_{\mathbb Q}[ \widehat{R}_{ijk}].
\end{equation}

\noindent For $S\subseteq[n]$, write $\Theta_S=(\Theta_k)_{k\in S}$ for the restriction of the latent labels to $S$; in particular, $\Theta_{S\cup T}$ denotes the labels on the union of two index sets. When $S$ has few elements we use the shorthand $\Theta_{ij}$ for $\Theta_{\{i,j\}}$, and similarly for larger sets. Conditional probabilities or expectations given $\Theta_S$ are understood as functions of the corresponding label realization.

Since the conditional law of $R_{ij}$, equivalently of $Y_{ij}$, given $\Theta$ depends only on the two labels $\Theta_i$ and $\Theta_j$, any conditional expectation involving this single edge depends only on $\Theta_{ij}$. Thus, for any $i<j$,
\begin{equation}\label{eq:Rij}
 \mathbb{E}_{\mathbb Q}[R_{ij} \mid \Theta] = \mathbb{E}_{\mathbb Q}[R_{ij} \mid \Theta_{ij}]= \mathbb{E}_{\mathbb Q}[Y_{ij} \mid \Theta_{ij}] -p_0=\ell (p_1-p_0) \mathds{1}\{\Theta_i = \Theta_j\in [\ell]\}. 
\end{equation}

\noindent To evaluate \eqref{eq:R_signed}, fix a triple $i < j < k$,
\begin{align*}
   \mathbb{E}_{\mathbb Q}[\widehat{R}_{ijk}] &= \sum_{\Theta_{ijk}}\mathbb{E}_{\mathbb Q}[\widehat{R}_{ijk} \mid \Theta_{ijk}]\mathbb P(\Theta_{ijk})\\
   &=\sum_{\Theta_{ijk}}\mathbb{E}_{\mathbb Q}[R_{ij} \mid \Theta_{ij} ]\mathbb{E}_{\mathbb Q}[R_{ik} \mid \Theta_{ik}]\mathbb{E}_{\mathbb Q}[R_{jk} \mid \Theta_{jk}]\mathbb P(\Theta_{ijk})\\
    &=\ell^3\delta^3 \sum_{\Theta_{ijk}} \mathds{1}\{\Theta_i = \Theta_j = \Theta_k \in [\ell]\}\mathbb P(\Theta_{ijk}) = \ell\delta^3\rho^3 \qquad \text{by \eqref{eq:Rij}}.
\end{align*}
Hence,
\begin{equation}\label{eq:EQ-EP}
\mathbb{E}_{\mathbb Q}[\widehat{R}] = \binom{n}{3}\ell\delta^3\rho^3.
\end{equation}

\noindent We now compute the second moment $\mathbb{E}_{\mathbb Q}[\widehat{R}^2]$. 
Decompose $\mathbb{E}_{\mathbb Q}[\hat R_{ijk} \hat R_{abc}]$ according to 
\[
    t := |\{i,j,k\} \cap \{a,b,c\}|.
\]
For each value of $t$, we count the number of contributing ordered pairs of triples \((i,j,k), (a,b,c)\), and evaluate the corresponding expectation.

\paragraph*{Case $t=0$ (no intersection)} 
Number of terms: $\binom{n}{3}\binom{n-3}{3}$.
\begin{align}
    \notag\mathbb{E}_{\mathbb Q}[\widehat{R}_{ijk}\widehat{R}_{abc}] 
    &= \sum_{\Theta_{ijkabc}}\mathbb{E}_{\mathbb Q}[R_{ij} \mid \Theta_{ij}] \cdots\mathbb{E}_{\mathbb Q}[R_{bc} \mid \Theta_{bc}]\mathbb P(\Theta_{ijkabc})\\
    \notag &= \ell^6\delta^6 \sum_{\Theta_{ijkabc}}
    \notag \mathds{1}\{\Theta_i = \Theta_j = \Theta_k \in [\ell]\}
    \notag\mathds{1}\{\Theta_a = \Theta_b = \Theta_c \in [\ell]\}
    \notag \mathbb P(\Theta_{ijkabc})\\
    \notag &= \ell^{2}\delta^6\rho^6.
\end{align}

\paragraph*{Case $t = 1$ (one common vertex)}
Number of terms: $3 \binom{n}{3}\binom{n-3}{2}$.
\begin{align*}
   \mathbb{E}_{\mathbb Q}[\widehat{R}_{ijk}\widehat{R}_{ibc}] 
    &= \sum_{\Theta_{ijkbc}}\mathbb{E}_{\mathbb Q}[R_{ij} \mid \Theta_{ij}]  \cdots
    \mathbb E_\mathbb Q [R_{bc} \mid \Theta_{bc}]\mathbb P(\Theta_{ijkbc})\\
    &= \ell^6\delta^6 \sum_{\Theta_{ijkbc}}
    \mathds{1}\{\Theta_i = \Theta_j = \Theta_k = \Theta_b
    = \Theta_c \in [\ell]\}\mathbb P(\Theta_{ijkbc})= \ell^{2} \delta^6\rho^5.
\end{align*}

\noindent Before considering the case $t=2$, we characterize the conditional second moment of an edge $(i,j)$
based on the configuration of the pair \((\Theta_i, \Theta_j)\). We fix the following complementary events:\\

\noindent \emph{Scenario A.} Event \(\mathcal{A}_{ij} := (\Theta_i=\Theta_j \in [\ell]).\)
\begin{align*}
\mathbb{E}_{\mathbb Q}[R_{ij}^2 \mid \Theta_{ij} \wedge \mathcal{A}_{ij}]
= (1-2p_0)(p_0+\ell\delta)+p_0^2
&= p_0(1-p_0)+\ell\delta(1-2p_0):= \mu_{_A}.
\end{align*}

\noindent \emph{Scenario B.} Event \(\mathcal{B}_{ij} := \mathcal{A}_{ij}^c\).
\begin{align*}
\mathbb{E}_{\mathbb Q}[R_{ij}^2 \mid \Theta_{ij} \wedge \mathcal{B}_{ij}]
= p_0-2p_0^2+p_0^2&= p_0(1-p_0):= \mu_{_B}.
\end{align*}

\paragraph*{Case $t = 2$ (one common edge)}
Number of terms: $3\binom{n}{3}\binom{n-3}{1}$.
\begin{align*}
   \mathbb{E}_{\mathbb Q}[\widehat{R}_{ijk}\widehat{R}_{ijc}] &= \sum_{\Theta_{ijkc}} \mathbb
    E_\mathbb Q[R_{ij}^2 \mid \Theta_{ij}]
    E_\mathbb Q[R_{ik} \mid \Theta_{ik}] \cdots
    \mathbb E_\mathbb Q [R_{jc} \mid \Theta_{jc}] \mathbb P(\Theta_{ijkc})\\
    &= \sum_{\Theta_{ijkc}}\mathbb{E}_{\mathbb Q}[R_{ij}^2 \mid \Theta_{ij}] \ell^4\delta^4
    \mathds{1}\{\Theta_i = \Theta_j = \Theta_k = \Theta_c \in [\ell]\}\mathbb P(\Theta_{ijkc})
    \intertext{here, $ \mathds{1}\{\Theta_i = \Theta_j = \Theta_k = \Theta_c \in [\ell]\}$
    forces Scenario A for $(i,j)$, hence,}
    &= \sum_{\Theta_{ijkc}} \mu_{_A} \ell^4\delta^4
    \mathds{1}\{\Theta_i = \Theta_j = \Theta_k = \Theta_c \in [\ell]\}\mathbb P(\Theta_{ijkc})= \mu_{_A}\ell \delta^4 \rho^4.
\end{align*}
\paragraph*{Case $t = 3$ (identical triangles)}
 Number of terms: $\binom{n}{3}$.
\begin{align}
   \mathbb{E}_{\mathbb Q}[R_{ij}^2 R_{ik}^2 R_{jk}^2] 
    &= \sum_{\Theta_{ijk}}\mathbb{E}_{\mathbb Q}[R_{ij}^2 \mid \Theta_{ijk}]
   \mathbb{E}_{\mathbb Q}[R_{ik}^2 \mid \Theta]
   \mathbb{E}_{\mathbb Q}[R_{jk}^2 \mid \Theta_{ijk}]\mathbb P(\Theta_{ijk}). \label{eq:RRR}
\end{align}
We now evaluate the right-hand side by partitioning into three subcases.\\

\noindent \emph{Subcase (i):} Event $\Gamma_1 := (\Theta_i = \Theta_j = \Theta_k \in [\ell])$. Then, \eqref{eq:RRR} becomes
\begin{align*}
    &\sum_{\Theta_{ijk}}\mathbb{E}_{\mathbb Q}[R_{ij}^2 \mid \Theta_{ijk} \wedge \Gamma_1]
   \mathbb{E}_{\mathbb Q}[R_{ik}^2 \mid \Theta_{ijk} \wedge \Gamma_1]
   \mathbb{E}_{\mathbb Q}[R_{jk}^2 \mid \Theta_{ijk} \wedge \Gamma_1]\mathbb P(\Theta_{ijk} \wedge \Gamma_1)= \mu_{_A}^3 \frac{\rho^3}{\ell^2}.
\end{align*}

\noindent \emph{Subcase (ii):} Event $\Gamma_2$ where \begin{align*}
     \Gamma_2 &:= \Big((\Theta_i = \Theta_j \in [\ell]) \wedge
(\Theta_i \neq \Theta_k)\Big) \vee \Big((\Theta_i = \Theta_k \in [\ell]) \wedge
(\Theta_i \neq \Theta_j)\Big)\\
& \quad \vee \Big((\Theta_j = \Theta_k \in [\ell]) \wedge
(\Theta_i \neq \Theta_j)\Big).
\end{align*} 
Owing to the symmetry of the expression with respect to $i, j, k,$ we consider only one case and the full contribution is then obtained by multiplying by 3. Hence, without loss of generality, assume the event $\Gamma_2^{\prime} := \Big((\Theta_i = \Theta_j \in [\ell]) \wedge
(\Theta_i \neq \Theta_k)\Big)$. In this case, \eqref{eq:RRR} reduces to
\begin{align*}
    3&\sum_{\Theta_{ijk}}\mathbb{E}_{\mathbb Q}[R_{ij}^2 \mid \Theta_{ijk} \wedge \Gamma_2^{\prime}]
   \mathbb{E}_{\mathbb Q}[R_{ik}^2 \mid \Theta_{ijk} \wedge \Gamma_2^{\prime}]
   \mathbb{E}_{\mathbb Q}[R_{jk}^2 \mid \Theta_{ijk} \wedge \Gamma_2^{\prime}]\mathbb P(\Theta_{ijk} \wedge \Gamma_2^{\prime})\\
    &= 3\mu_{_A}\mu_{_B}^2\sum_{\Theta_{ijk}}\mathbb P(\Theta_{ijk} \wedge \Gamma_2^{\prime})= 3\mu_{_A}\mu_{_B}^2 \frac{\rho^2}{\ell}\Big(1-\frac{\rho}{\ell}\Big).
\end{align*}
\noindent \emph{Subcase (iii):} Event $\Gamma_3 := (\Gamma_1 \vee \Gamma_2)^c$.
Here we obtain
\begin{align*}
    &\sum_{\Theta_{ijk}}\mathbb{E}_{\mathbb Q}[R_{ij}^2 \mid \Theta_{ijk} \wedge \Gamma_3]
   \mathbb{E}_{\mathbb Q}[R_{ik}^2 \mid \Theta_{ijk} \wedge \Gamma_3]
   \mathbb{E}_{\mathbb Q}[R_{jk}^2 \mid \Theta_{ijk} \wedge \Gamma_3]\mathbb P(\Theta_{ijk} \wedge \Gamma_3)\\
    &= \mu_{_B}^3 \sum_{\Theta_{ijk}} \mathbb P(\Theta_{ijk} \wedge \Gamma_3)= \mu_{_B}^3  \mathbb P(\Gamma_3)= \mu_{_B}^3\Big(1 - \frac{\rho^3}{\ell^2} 
    - 3 \frac{\rho^2}{\ell}\Big(1-\frac{\rho}{\ell}\Big)\Big).
\end{align*}
\noindent Consequently, combining the three subcases,
\[
   \mathbb{E}_{\mathbb Q}[R_{ij}^2 R_{ik}^2 R_{jk}^2] = \mu_{_A}^3 \frac{\rho^3}{\ell^2} 
    + 3\mu_{_A}\mu_{_B}^2 \frac{\rho^2}{\ell}\Big(1-\frac{\rho}{\ell}\Big)
    + \mu_{_B}^3\Big(1 - \frac{\rho^3}{\ell^2} 
    - 3 \frac{\rho^2}{\ell}\Big(1-\frac{\rho}{\ell}\Big)\Big).
\]
This concludes the $t=3$ case.

\noindent Now, by summing over the four values of $t$,
\begin{align*}
\mathbb{E}_{\mathbb Q}[\widehat{R}^2]
&= \binom{n}{3}\Bigg[\binom{n-3}{3}\ell^{2}\delta^6\rho^6+ 3\binom{n-3}{2}\ell^{2} \delta^6 \rho^5+ 3\binom{n-3}{1} \mu_{_A}\ell\delta^4 \rho^4 \\
    &+ \Bigg(\mu_{_A}^3 \frac{\rho^3}{\ell^2}
    + 3\mu_{_A}\mu_{_B}^2 \frac{\rho^2}{\ell}\Big(1-\frac{\rho}{\ell}\Big)+ \mu_{_B}^3\Big(1 -\frac{\rho^3}{\ell^2} 
    - 3\frac{\rho^2}{\ell}\Big(1-\frac{\rho}{\ell}\Big)\Bigg)\Bigg].
\end{align*}

\noindent By \eqref{eq:EQ-EP}, 
\begin{align*}
{\rm Var}_\mathbb{Q}[\widehat{R}] &=\frac{1}{2} \binom{n}{3}\Big( \ell^2 \delta ^6 \rho ^5 \Big[3 n (-n\rho +n+5 \rho -7)-20 \rho +36\Big]\\ 
&+ 6\ell^2 (n-3) \left(1- 2 p_0\right) \delta ^5 \rho ^4+ 6 \ell (n-3) \left(1-p_0\right) p_0 \delta ^4 \rho ^4+ 2\ell \left(1- 2 p_0\right)^3 \delta ^3 \rho ^3\\
   &+ 6 \left(1-p_0\right) p_0 \left(1-2 p_0\right){}^2 \delta ^2 \rho ^3+ 6 \left(1-p_0\right)^2 p_0^2 \left(1- 2 p_0\right) \delta  \rho ^2+ 2\left(1-p_0\right)^3 p_0^3\Big).
\end{align*}
Moreover, 
\[
(\mathbb{E}_{\mathbb Q}[\widehat{R}] -\mathbb{E}_\mathbb P[\widehat{R}])^2 = 
\binom{n}{3}^2(\ell-\ell^\prime)^{2}\delta^6\rho^6.
\]

\noindent To finish the proof, by the definition of weak separation, it suffices to show that ${\rm Var}_\mathbb{Q}[\widehat{R}] = O\Big((\E_\mathbb Q[\widehat{R}] - \E_\mathbb P[\widehat{R}])^2\Big)$; equivalently, removing lower order terms in the above, and since $\ell$ and $\ell'$ are fixed positive integers, the problem reduces to proving
\begin{eqnarray*}
   &&\delta ^6 \rho ^5 n^2+ n \left(1- 2 p_0\right) \delta ^5 \rho ^4+ n \left(1-p_0\right) p_0 \delta ^4 \rho ^4+ \left(1- 2 p_0\right){}^3 \delta ^3 \rho ^3\\
   &&+ \left(1-p_0\right) p_0 \left(1-2 p_0\right){}^2 \delta ^2 \rho ^3+\left(1-p_0\right){}^2 p_0^2 \left(1- 2 p_0\right) \delta  \rho ^2+\left(1-p_0\right){}^3 p_0^3\\
   &&= O(n^3 \delta ^6 \rho^6).
\end{eqnarray*}

\noindent The seven terms of the sum are all in $O(n^3 \delta ^6 \rho^6)$ under
\[
n\delta^2 \rho^2 = \Omega(p_0(1-p_0)), \quad 
n\rho = \Omega(1), \quad
n\delta \rho = \Omega(|1-2p_0|). 
\]
The first condition is exactly the assumption
$\lambda=\Omega((\rho\sqrt n)^{-1})$, since $\lambda=\delta/\sqrt{p_0(1-p_0)}$.
\end{proof}

\section*{Acknowledgements}
AS, DGE, and FS were partially supported by the Wallenberg AI,
Autonomous Systems and Software Program (WASP), funded by the Knut and Alice
Wallenberg Foundation. DGE was additionally supported by the Secretaría de Ciencia, Humanidades, Tecnología e Innovación (SECIHTI) and the Royal Swedish Academy of Sciences (Kungl. Vetenskapsakademien). Part of this material is based upon work supported by the National Science Foundation under Grant No. DMS-1928930, while FS was in residence at the Simons Laufer Mathematical Sciences Institute (MSRI) during the Spring 2025 semester. ASW was partially supported by a Sloan Research Fellowship and NSF CAREER Award CCF-2338091. 

\newpage
\bibliographystyle{alpha}
\bibliography{ref}

\newpage
\appendix
\section{Proof of Technical Lemmas}
 \subsection{Lower Bounds}
\subsubsection{Proof of Lemma \ref{lemma:nr-multigraphs}}\label{app:pf-counting-lemma}
This lemma counts the number of multigraphs satisfying a specified property, which is then used to bound the advantage. We prove the stronger statement that the number of such components is \[\frac{ n_{(v)}}{C!}  \binom{M+k-1}{k}
\sum_{\substack{v_1+\cdots+v_C=v \\ v_i \ge 1}}
\prod_{i=1}^C
\frac{v_i^{v_i-2}m_i}{v_i!}\]
   where $n_{(v)}=n(n-1)\ldots(n-v+1)$, $m_i=\frac{v_i(v_i+1)}{2}$ and $M:=\sum_{i=1}^C m_i$. To see that this implies the lemma, note that $m_i=v_i(v_i+1)/2 \leq v_i^2$, and that $\sum_i m_i \leq \sum_i v_i^2\leq v^2$.

\begin{proof}
Recall that $\alpha \in \goodI$ implies that there are no tree components. Let the connected components have vertex sizes $v_1,\dots,v_C$ and edge counts $d_1,\dots,d_C$, so that $v_1+\cdots+v_C=v$ and $d_1+\cdots+d_C=d$. Since each component is connected and not a tree, we have $d_i\ge v_i$ for every $i$. Writing $d_i=v_i+k_i$ with $k_i\ge 0$, it follows that $k_1+\cdots+k_C=\sum_{i=1}^C(d_i-v_i)=d-v=:k$.

We first choose the $v$ vertices, giving a factor $\binom{n}{v}$. For fixed $v_1,\dots,v_C$, the number of ways to partition these vertices into $C$ labeled blocks of sizes $v_1,\dots,v_C$ is $\frac{v!}{\prod_{i=1}^C v_i!}$. Since the components are unordered, we divide by $C!$.

Now fix one component on a prescribed labeled vertex set of size $v_i$, with $d_i=v_i+k_i$ edges. Such a connected multigraph contains a spanning tree, and there are $v_i^{\,v_i-2}$ choices for this tree by Cayley's formula.

Once the tree is fixed, the component has $d_i-(v_i-1)=k_i+1$ edges beyond the spanning tree. We first choose one such compulsory non-tree edge, in at most $m_i=\binom{v_i+1}{2}=\frac{v_i(v_i+1)}{2}$ ways. The remaining $k_i$ edges may then be chosen as a multiset from the same $m_i$ possible edge types, giving at most $\binom{m_i+k_i-1}{k_i}$ choices.

Therefore, for fixed $(v_i)$ and $(k_i)$, the number of possibilities is at most

\[
\frac{v!}{C!} 
\prod_{i=1}^C
\left(
\frac{v_i^{\,v_i-2}m_i}{v_i!}\binom{m_i+k_i-1}{k_i}
\right).
\]

Summing over all compositions $v_1+\cdots+v_C=v$ with $v_i\ge 1$, and all $k_1+\cdots+k_C=k$ with $k_i\ge 0$, gives

\[
\binom{n}{v} \frac{v!}{C!} 
\sum_{\substack{v_1+\cdots+v_C=v \\ v_i \ge 1}}
\prod_{i=1}^C
\frac{v_i^{\,v_i-2}m_i}{v_i!}
\sum_{\substack{k_1+\cdots+k_C=k \\ k_i \ge 0}}
\prod_{i=1}^C
\binom{m_i+k_i-1}{k_i}.
\]

Applying Lemma~\ref{lem:combinatorial_bound} to the sum over $k_1,\dots,k_C$, with
\[
M=\sum_{i=1}^C m_i=\sum_{i=1}^C\frac{v_i(v_i+1)}2,
\]
we obtain the upper bound
\[
\binom{n}{v}\frac{v!}{C!}
\sum_{\substack{v_1+\cdots+v_C=v \\ v_i \ge 1}}
\prod_{i=1}^C
\frac{v_i^{\,v_i-2}m_i}{v_i!}
\binom{M+k-1}{k}.
\]
\end{proof}
\begin{remark}
One could instead choose all $k_i+1$ non-tree edges in the $i$-th component at once,
which gives the alternative upper bound with the factor $\binom{M+k+C-1}{k+C}$. We use the slightly different decomposition above, choosing one compulsory non-tree edge
per component first, because it isolates the remaining $k$ excess edges and leads to the
factor $\binom{M+k-1}{k}$,
which streamlines the subsequent summation over $k$.
\end{remark}

\subsubsection{Proof of the combinatorial identity used in Lemma~\ref{lemma:nr-multigraphs}}\label{app:pf-lemma-combinatorial}

This lemma is used in the proof of the lower bound for strong testing in the PSM model.

\begin{lemma}\label{lem:combinatorial_bound}
Let $C\geq 1$ be an integer and let $m_1,\dots,m_C\in\mathbb{N}$ with $m_i\geq 1$ for all $i\in[C]$, and set $M:=\sum_{i=1}^C m_i.$
Then, for every integer $k\geq 0$,
\[
\sum_{\substack{k_1,\dots,k_C\in\mathbb{N}\\ k_1+\cdots+k_C=k}}
\prod_{i=1}^C
\binom{m_i+k_i-1}{k_i}
=
\binom{M+k-1}{k}.
\]
\end{lemma}

\begin{proof}
We use the upper negation identity
$\binom{n}{k}=(-1)^k\binom{k-n-1}{k}$, for $n\in\mathbb Z,\ k\in\mathbb N_0.$
Thus, for each $i$,
\[
\binom{m_i+k_i-1}{k_i}=(-1)^{k_i}\binom{-m_i}{k_i}.
\]
Hence,
\[\sum_{\substack{k_1,\dots,k_C\in\mathbb{N}\\ k_1+\cdots+k_C=k}}
\prod_{i=1}^C
\binom{m_i+k_i-1}{k_i}=(-1)^k\sum_{\substack{k_1+\ldots+k_C=k\\ k_i \ge 0}}\prod_{i=1}^C \binom{-m_i}{k_i}:=S,\]
since $\sum_{i=1}^C k_i=k$. By the generalized Chu-Vandermonde identity,
\[\sum_{\substack{k_1+\ldots+k_C=k\\ k_i \ge 0}}\prod_{i=1}^C \binom{-m_i}{k_i}=\binom{-M}{k}.\]
Applying upper negation once more, 
$\binom{-M}{k}=(-1)^k\binom{M+k-1}{k},$
and therefore 
\[S=(-1)^{2k}\binom{M+k-1}{k}=\binom{M+k-1}{k}.
\]
\end{proof}

\subsection{Upper bounds}\label{app:pf:technical-lemmas-second-moment-psm}
We provide the proofs of the auxiliary lemmas used in the upper-bound arguments for PSM
and PDS.
\subsubsection*{Proofs of Auxiliary Lemmas for the Planted Submatrix Upper Bound}
\subsubsection{Proof of Lemma~\ref{lem:second-moment-equal}}
 This lemma forms part of the second moment calculations for the BUG polynomial, defined on line~\eqref{def:BUG_f}, for the planted submatrix model.
\begin{proof}
Observe that \begin{equation*}\label{eq.linear_in_Z_trick}
\mathbb{E}_\QQ[Y_{ij}^2] = \mathbb{E}_\QQ[ \big( \ell \lambda\sum_{c=1}^\ell \mathds{1}[\Theta_i = \Theta_j = c] + Z_{ij}\big)^2 ] =
\mathbb{E}_{\QQ_\Theta}[ \ell^2 \lambda^2 \sum_{c=1}^\ell \mathds{1}[\Theta_i = \Theta_j = c] + 1 ],
\end{equation*}
where the second equality followed since at most one of the indicators is non-zero and the terms which are linear in $Z_{ij}$ are zero.
Recall that $\eta:=(\ell\lambda)^2$. Since the $Z_{ij}$'s are independent we get that for any graph $\alpha$,
\begin{equation}
\label{eq:equal-eta-comparison}
\begin{aligned}
\E_\QQ[Y^{2\alpha}]
&= \E_{\QQ_\Theta} \left[\prod_{(i,j) \in \alpha} (\eta \sum_{c=1}^\ell \mathds{1}[\Theta_i = \Theta_j = c] + 1) \right]. 
\end{aligned}
\end{equation}
\noindent 
We now restrict to $\alpha \in \cU_k$.
To bound this expectation, we construct an auxiliary directed graph $\tilde{\alpha}$ from $\alpha$ by directing edges in the triangle to form a directed cycle, and directing the other edges away from the cycle. Note that for any two directed edges $ij$ and $i'j'$ we have $j\neq j'$. (For vertices~$j$ in the directed cycle this is easy to see. For other vertices $j$, i.e. those in the trees hanging from the cycle, if there were two incoming edges $ij$ and $i'j$ to vertex $j$, this would create a cycle in that hanging tree in $\alpha$, a contradiction).
For each $c\in[\ell]$,
$\mathds{1}\{\Theta_i=c\}\mathds{1}\{\Theta_j=c\}
\le
\mathds{1}\{\Theta_j=c\}.$
Summing over $c$ gives
$
\sum_{c=1}^{\ell}\mathds{1}\{\Theta_i=\Theta_j=c\}
\le
\sum_{c=1}^{\ell}\mathds{1}\{\Theta_j=c\}
$. Hence, 
\begin{align}
\E_\QQ[Y^{2\alpha}]
= \notag \E_{\QQ_\Theta} \left[\prod_{ij \in \tilde{\alpha}} (\eta  \sum_{c=1}^\ell \mathds{1}[\Theta_i = \Theta_j = c] + 1) \right] &\leq \notag \E_{\QQ_\Theta} \left[\prod_{ij \in \tilde{\alpha}} (\eta   \sum_{c=1}^\ell \mathds{1}[\Theta_j = c] + 1) \right]\\
\label{eq:upper_bound_common_eq} &= \prod_{ij \in \tilde{\alpha}} \E_{\QQ_\Theta} \left[ \eta \sum_{c=1}^\ell \mathds{1}[\Theta_j = c] + 1 \right]  \\
\notag &= (\eta\rho+1)^{|\alpha|},
\end{align}
where line~\eqref{eq:upper_bound_common_eq} followed since each $\Theta_j$ appears at most once in the line above it and thus the terms in the product are independent. 
\end{proof}
\subsubsection{Proof of Lemma~\ref{lem:second-moment-not-equal}}
 As with the previous lemma, this lemma forms part of the second moment calculations for the BUG polynomial.
\begin{proof}
This proof is similar to the proof of Lemma~\ref{lem:second-moment-equal}. Note that for pair $(\alpha, \beta)$ we have
{\small \begin{align}
    \notag &\E_\QQ[Y^{\alpha+\beta}]\\
    \notag &=\E_\QQ \big[ \prod_{(i,j) \in \alpha \triangle \beta}( Z_{ij}+\ell\lambda \sum_{c=1}^\ell \mathds{1}[\Theta_i=c]\mathds{1}[\Theta_j=c]) \prod_{(i,j)\in \alpha \cap \beta} (Z_{ij}+\ell\lambda \sum_{c=1}^\ell\mathds{1}[\Theta_i=c]\mathds{1}[\Theta_j=c])^2 \big]\\
    \notag &=\E_{\QQ_\Theta} \big[ \prod_{(i,j) \in \alpha \triangle \beta} \ell\lambda \sum_{c=1}^\ell \mathds{1}[\Theta_i=c]\mathds{1}[\Theta_j=c] \prod_{(i,j)\in \alpha \cap \beta} (1+(\ell\lambda)^2 \sum_{c=1}^\ell\mathds{1}[\Theta_i=c]\mathds{1}[\Theta_j=c])   \big].
    \end{align}}
Let $\QQ_\Theta$ denote the distribution of the community assignments $\Theta$ under $\QQ$. Then,
 {\small   \begin{align}
    \E_{\QQ}[Y^{\alpha+\beta}] 
    &=(\ell\lambda)^{|\alpha\triangle\beta|}
    \E_{ \QQ_\Theta} \Big[ 
    \prod_{(i,j) \in \alpha \triangle \beta} \mathds{1}[\Theta_i=\Theta_j \in [\ell]] 
  \prod_{(i,j)\in \alpha \cap \beta}
\left(1+\eta\sum_{c=1}^{\ell}\mathds{1}\{\Theta_i=\Theta_j=c\}\right).\label{eq:upper_bound_common_neq}
\end{align}}
\noindent Now, define the event
\[\mathcal E_0:=\bigcap_{(i,j)\in \alpha\triangle\beta}\{\Theta_i=\Theta_j\in[\ell]\},\]
which gives $\prod_{(i,j) \in \alpha \triangle \beta}\mathds{1}\{\Theta_i=\Theta_j\in[\ell]\}
=\mathds{1}[\mathcal E_0].$
So, we obtain
\begin{align*}
\E_\QQ[Y^{\alpha+\beta}]
&=(\ell\lambda)^{|\alpha\triangle\beta|}
\E_{\QQ_\Theta} \Big[ 
\mathds{1}[\mathcal E_0] 
\prod_{(i,j)\in \alpha \cap \beta} \Big(1+\eta\,\mathds{1}\{\Theta_i=\Theta_j\in[\ell]\}\Big) \Big].
\end{align*}
Finally, for any integrable random variable $X$ and event $ \mathcal E_0$ with $\Pr( \mathcal E_0)>0$,
\[
\E[\mathds 1[\mathcal E_0]\,X]=\Pr(\mathcal E_0)\,\E[X\mid \mathcal E_0].
\]
Apply this with
$
X:=\prod_{(i,j)\in \alpha \cap \beta} \big(1+\eta\,\mathds{1}\{\Theta_i=\Theta_j\in[\ell]\}\big),
$
to conclude
\begin{align}
\E_\QQ[Y^{\alpha+\beta}]
&=(\ell\lambda)^{|\alpha\triangle\beta|} 
\Pr(\mathcal E_0) 
\E_{\QQ_\Theta}\bigg[
\prod_{(i,j)\in \alpha \cap \beta} \Big(1+\eta\,\mathds{1}\{\Theta_i=\Theta_j\in[\ell]\}\Big)
\ \Big|\ \mathcal E_0\bigg]. \label{eq.in_terms_of_exp}
\end{align}
To calculate $\Pr(\mathcal E_0)$, note that the event $\mathcal{E}_0$ holds if and only if each connected compoent of $\alpha\triangle\beta$ is monochromatic with color in $[\ell]$.
\[\mathds{1}[\mathcal E_0] = \prod_{(i,j) \in \alpha \triangle \beta}\mathds{1}\{\Theta_i=\Theta_j\in[\ell]\}
= \prod_{\tau \in \mathcal{C}(\alpha \triangle \beta)} \sum_{c=1}^\ell  \mathds{1}[\Theta_i =c ,   \forall i \in V(\tau)].\]
Note also that for $\tau$ connected, $\sum_{c=1}^\ell \Pr\big( \Theta_i =c ,   \forall i \in V(\tau)\big) = \ell\left(\frac{\rho}{\ell}\right)^{|V(\tau)|} $ and hence,
\begin{equation}\label{eq.upperbound_PE0} \Pr(\mathcal E_0) 
=\ell^{m_{\triangle}}\left(\frac{\rho}{\ell}\right)^{|V(\alpha \triangle \beta)|}. \end{equation} 
Now, we bound 
\[\E_{\QQ_\Theta}\Bigg[
\prod_{(i,j)\in \alpha \cap \beta} \Big(1+\eta\,\mathds{1}\{\Theta_i=\Theta_j\in[\ell]\}\Big)
\ \Big|\ \mathcal E_0\Bigg],\]
where we have two cases again, i.e. $\alpha \cap \beta$ is a forest or unicyclic. Let
\[
B:=V(\alpha\triangle\beta)\cap V(\alpha\cap\beta),
\qquad b:=|B|.
\]
Under $\mathcal E_0$, every vertex in $B$ has a nonzero label, i.e.,
$\mathds{1}[\Theta_v \in [\ell]]=1 \ \forall v \in {B}.$ Write $\delta = \alpha \cap \beta$ and construct an auxiliary directed graph $\tilde{\delta}$ as follows. If there is a unicyclic component in $\delta$, direct the edges in the cycle so that this is a directed cycle in $\tilde{\delta}$ and for other edges direct the edge away from the cycle. For each tree component $\tau$, pick a root $r_\tau\in B$ and direct edges away from the root. (Note that we always pick the root vertex in $B$ since~$\alpha \neq \beta$.)
\paragraph*{Case 1: $\delta= \alpha \cap \beta$ is a forest}
Consider a directed edge $ij \in \tilde{\delta}$ (note by construction $i$ is closer to the root). If $j \in B$, then we upper bound $\mathds{1}\{\Theta_i=\Theta_j\in[\ell]\}$ by $1$. If $j  \notin B$, we upper bound $\mathds{1}\{\Theta_i=\Theta_j\in[\ell]\}$ by $\mathds{1}\{\Theta_j \in [\ell]\}$.
Consider an arbitrary tree component $\tau$ in $\delta$, and denote by $\tilde{\tau}$ its inherited directed version and by $r_\tau\in B$ the root vertex. 
Note that for any function which is a function $h$ only of the `tail' of each directed edge the product over $ij \in \tilde{\tau}$ can be rewritten as follows
\begin{align}\label{eq.only_tails_for_trees}
\prod_{ij\in \tilde{\tau} } h(j) = \prod_{j \in V(\tau)\backslash r_{\tau}} h(j) = 
\prod_{j\in (V(\tau) \cap B) \backslash r_\tau  } h(j) \prod_{j\in V(\tau) \backslash B} h(j).
\end{align}
Hence we have,
\begin{align*}
&\E_{\QQ_\Theta}\Bigg[
\prod_{ij\in \tilde{\tau} } \Big(1+\eta\,\mathds{1}\{\Theta_i=\Theta_j\in[\ell]\} \Big) \Big| \mathcal E_0 \Bigg]\\
&\le \E_{\QQ_\Theta}\Bigg[\prod_{j\in (V(\tau) \cap B) \backslash r_\tau  } \Big(1+\eta \Big) \prod_{j\in V(\tau) \backslash B} \Big(1+\eta\,\mathds{1}\{\Theta_j\in[\ell]\} \Big)\Bigg| \mathcal E_0\Bigg]\\
&= \Big(1+\eta \Big)^{|V(\tau)\cap B|-1}   \E_{\QQ_\Theta}\Bigg[ \prod_{j\in V(\tau) \backslash B} \Big(1+\eta\,\mathds{1}\{\Theta_j\in[\ell]\} \Big)\Bigg]\\
&= \Big(1+\eta \Big)^{|V(\tau)\cap B|-1}   \Big(1+\eta\rho \Big)^{|V(\tau) \backslash B|},
\end{align*}
where in the second line we can remove the conditioning since the product consists only of random variables~$\Theta_j$ for $j\notin B$. 
Thus,
\begin{align*}
 \E_{\QQ_\Theta}\Bigg[
\prod_{ij\in \tilde{\delta} } \Big(1+\eta\,\mathds{1}\{\Theta_i=\Theta_j\in[\ell]\} \Big) \Big| \mathcal E_0 \Bigg]
&\le \prod_{\tau \in \mathcal{C}(\delta)} \Big(1+\eta \Big)^{|V(\tau)\cap B|-1}   \Big(1+\eta\rho \Big)^{|V(\tau) \backslash B|}\\
&= \Big(1+\eta \Big)^{b-m_\cap}   \Big(1+\eta\rho \Big)^{|V(\alpha\cap \beta)| - b}\\
&= \Big(1+\eta \Big)^{b-m_\cap}   \Big(1+\eta\rho \Big)^{|\alpha\cap \beta| + m_\cap - b}.
\end{align*}
This, together with~\eqref{eq.in_terms_of_exp} and~\eqref{eq.upperbound_PE0} concludes the proof for this case.
\paragraph*{Case 2: $\delta=\alpha \cap \beta$ is unicyclic}
Let $\pi$ denote the unicyclic component of $\delta$. As in
\eqref{eq.only_tails_for_trees}, the contribution can be written in terms of tails of directed edges. In the present unicyclic case, however, each vertex
appears as a tail in the directed-edge product, and therefore
\begin{align}\label{eq.only_unicyclic}\notag
\prod_{ij\in \tilde{\pi} } h(j) = \prod_{j \in V(\pi)} h(j) = 
\prod_{j\in (V(\pi) \cap B) } h(j) \prod_{j\in V(\pi) \backslash B} h(j).
\end{align}
Thus we have
\begin{align*}
& \E_{\QQ_\Theta}\bigg[
\prod_{ij\in \tilde{\pi} } \Big(1+\eta\,\mathds{1}\{\Theta_i=\Theta_j\in[\ell]\} \Big) \, \Big| \, \mathcal E_0 \bigg]\\
&\le \E_{\QQ_\Theta}\bigg[\prod_{j\in (V(\pi) \cap B)   } \Big(1+\eta \Big) \prod_{j\in V(\pi) \backslash B} \Big(1+\eta\,\mathds{1}\{\Theta_j\in[\ell]\} \Big) \, \Big| \, \mathcal E_0\bigg]\\
&\leq \Big(1+\eta \Big)^{|V(\pi)\cap B|}   \Big(1+\eta\rho \Big)^{|V(\pi) \backslash B|}.
\end{align*}
Hence,
\begin{align*}
 \E_{\QQ_\Theta}\bigg[
\prod_{(i,j)\in \tilde{\delta}} \Big(1+\eta\,\mathds{1}\{\Theta_i=\Theta_j\in[\ell]\}\Big)
\ \Big|\ \mathcal E_0\bigg]\leq (1+\eta)^{b-m_\cap+1}(1+\eta \rho)^{|\alpha\cap \beta|-(b-m_\cap)-1}.
\end{align*}
As in the case of $\alpha\cap\beta$ being a forest, this with~\eqref{eq.in_terms_of_exp} and~\eqref{eq.upperbound_PE0} yields the desired bound.\end{proof}
\subsubsection{Proof of Lemma~\ref{lem:forest-completion-count}}\label{app:pf-lem-forest-count-completion}
The proof follows the forest-overlap enumeration in
\cite[Proof of Theorem~2.2(b), Case~3]{sohn2025sharp}, with the modification
that the triangle-containing component of the core has already been fixed in
our application. We include the details to make clear how the remaining forest
part is counted.
\begin{proof}
Let the remaining forest core be denoted by $F$. By assumption, $F$ has $s-s_T$ edges and $m-1$ connected components, and hence $|V(F)|=(s-s_T)+(m-1)$. By the Cayley forest bound, the number of ways to choose its vertices and realize the forest on them is at most 
\[\binom{n}{s-s_T+m-1}(s-s_T+m-1)^{(s-s_T+m-2)}\leq (en)^{(s-s_T+m-1)}.\]
Next choose the $w$ additional common vertices outside the core, and then choose the remaining vertices of $\alpha$ and $\beta$. Set
\[
u := D-s_T-(s-s_T)-(m-1)-w=D-s-m+1-w.
\]
Now choose the branch points among the already fixed core vertices. This
contributes at most $\binom{D-s_T+2}{b}$ choices. After these vertices and branch
points have been fixed, the graph $\alpha\setminus\beta$ is a forest. It has at
most $v:=D-s_T-(s-s_T)-(m-1)+b=D-s-m+1+b$ vertices. Hence, by Cayley's bound for forests, the number of possible choices
for $\alpha \setminus \beta$ is at most $(v+1)^{v-1}$, and the same holds for
$\beta \setminus \alpha$. Therefore, for fixed $(s,s_T,m,b,w)$, the number of
completions is at most
\begin{align*}
    &(en)^{s-s_T+m-1}\binom{n}{w}\binom{n}{u}^2
    \binom{D-s_T+2}{b}(v+1)^{2(v-1)}\\
    &\leq
    (en)^{s-s_T+m-1} n^w (D-s_T+2)^b
    \left(\frac{en}{u}\right)^{2u}(v+1)^{2(v-1)}.
\end{align*}
Using $(1+1/u)^{2u}\le e^2$ and $v-u=b+w$, we have\begin{align*}
    \frac{(v+1)^{2(v-1)}}{u^{2u}}\leq \left(\frac{v+1}{u}\right)^{2u}(v+1)^{2(v-u)-2}
    \leq e^{2(b+w+1)}(v+1)^{2b+2w-2}.
\end{align*}
Plugging this into the count for \((\alpha,\beta)\), and using
\(v+1\le D-s_T+2\), gives
\begin{align*}
&e^2(en)^{s-s_T+m-1+2u} n^w
(D-s_T+2)^b e^{2(b+w+1)}(v+1)^{2b+2w-2} \\
&\leq
e^4
\left(e^2(D-s_T+2)^3\right)^b
\left(\frac{(D-s_T+2)^2}{n}\right)^w
(en)^{s-s_T+m-1+2u+2w}. \end{align*}\end{proof}
\vspace{-2em}
\subsubsection{Proof of Lemma~\ref{lem:b-k-conditions}}\label{app:pf:lem:b-k-conditions}
This lemma is used to control the sum over BUG overlaps. It excludes the
case where both the overlap size $s$ and the number of branch points $b$ are small.
\begin{proof}
Assume $\alpha,\beta\in\cU_k$ and $\alpha\cap\beta\neq \varnothing$. We first consider the case where $\alpha\cap\beta$ is a forest. In this case, the overlap is purely tree-like, and the balanced-tree argument from~\cite[Proof of Theorem~2.2(b)]{sohn2025sharp} applies to the shared forest inside the BUGs. It follows that either $b\ge 2$ or $s\ge k+3$. It remains to handle the genuinely unicyclic case, where $\alpha\cap\beta$ contains a cycle.  If $\alpha\neq\beta$, then $\alpha\triangle\beta\neq\varnothing$. Since both BUGs
are connected and share the cycle, every component of $\alpha\triangle\beta$
must attach to the common core. Hence $b\ge1$. Since $\alpha$ and~$\beta$ are unicyclic, their
intersection can contain a cycle only by containing the (unique) cycle, i.e.\ the triangle of each BUG.
Write this common triangle as $\{a,u,v\}$, and let
$T_u(\alpha),T_v(\alpha)$ (resp.\ $T_u(\beta),T_v(\beta)$)
denote the two hanging $k$-edge trees attached at $u$ and $v$. Let $
B:=V(\alpha\triangle\beta)\cap V(\alpha\cap\beta)$
be the set of branch points. We have shown that $b=|B|\ge1$. If $b\ge2$, then we are done, so it remains to consider the case $b=1$. Since the two hanging trees are attached at the distinct triangle vertices $u$ and $v$, a single branch point can obstruct at most one of the two $k$-edge
hanging trees. Hence at least one of these two hanging trees contains no branch point. By symmetry, assume $V(T_u(\alpha))\cap B=\varnothing$. We claim that $T_u(\alpha)\subseteq \alpha\cap\beta$.
Indeed, suppose not. Traverse the tree $T_u(\alpha)$ starting from $u$ and moving away from the triangle,
and let $e$ be the first edge encountered that lies in $\alpha\triangle\beta$. Let $x$ be the endpoint of
$e$ closer to $u$. By choice of $e$, all edges on the path from $u$ to $x$ are shared, hence $x\in V(\alpha\cap\beta)$,
while $x$ is also incident to $e\in E(\alpha\triangle\beta)$, hence $x\in V(\alpha\triangle\beta)$. Therefore
$x\in B$, contradicting $V(T_u(\alpha))\cap B=\varnothing$. This proves $T_u(\alpha)\subseteq \alpha\cap\beta$. Consequently, $\alpha\cap\beta$ contains the three triangle edges plus all $k$ edges of $T_u(\alpha)$, and so
$
s = |\alpha\cap\beta|\ \ge\ 3+k \,.
$
Thus, whenever $\alpha\cap\beta\neq \varnothing$, either $b\ge 2$ or $s\ge k+3$.
   \end{proof}
\subsubsection*{Proofs of the Auxiliary Lemma for the Planted Dense Subgraph Upper Bound}
\subsubsection{Proof of Lemma~\ref{lem:pds-second-moment}}\label{sec:proof:lem:pds-second-moment}
This lemma is used in the second moment calculation for the BUG polynomial $f$ defined in Equation~\eqref{def:BUG_f_PDS} for PDS, analogous to Lemmas~\ref{lem:second-moment-equal} and~\ref{lem:second-moment-not-equal}.
\begin{proof}
    With $P_{\alpha\beta}$ defined as in~\eqref{eq.defn_Palphabeta}, we first split the product over edges into those edges in $\alpha \cap \beta$ and those in $\alpha \triangle \beta$, and then expand out conditioned on $\Theta$, the community assignments of the vertices. We write $\QQ_\Theta$ for the distribution of $\Theta$ under $\QQ$.
    \begin{align*}
        P_{\alpha \beta}&=\E_\QQ \prod_{(i,j)\in\alpha \cap \beta}\left(\frac{Y_{ij}-p_0}{\sqrt{p_0(1-p_0)}}\right)^2 \prod_{(i,j)\in \alpha \triangle \beta}\left(\frac{Y_{ij}-p_0}{\sqrt{p_0(1-p_0)}}\right)\\
        &=\mathbb{E}_{ \QQ_\Theta}\Bigg[
\prod_{(i,j)\in\alpha\cap\beta}
\mathbb{E}_\QQ\left[\left(\frac{Y_{ij}-p_0}{\sqrt{p_0(1-p_0)}}\right)^2\Bigg| \Theta \right] 
\prod_{(i,j)\in\alpha\triangle\beta}
\mathbb{E}_\QQ\left[\left(\frac{Y_{ij}-p_0}{\sqrt{p_0(1-p_0)}}\right)\Bigg|\Theta
\right] 
\Bigg]
\\
&=\mathbb{E}_{ \QQ_\Theta}\Bigg[
\prod_{(i,j)\in\alpha\cap\beta}\Big(1+(1-2p_0)\tilde\eta\,\mathds{1}\{\Theta_i=\Theta_j\in[\ell]\}\Big)
\prod_{(i,j)\in\alpha\triangle\beta}\Big(\ell\lambda\,\mathds{1}\{\Theta_i=\Theta_j\in[\ell]\}\Big)
\Bigg].
\end{align*}
Factoring out $(\ell\lambda)^{|\alpha\triangle\beta|}$ from the second product gives
\begin{align*}
P_{\alpha\beta}
&=(\ell\lambda)^{|\alpha\triangle\beta|}
\mathbb{E}_{ \QQ_\Theta}\!\Bigg[
\prod_{(i,j)\in\alpha\triangle\beta} \!\!\! \mathds{1}\{\Theta_i=\Theta_j\in[\ell]\} \!\!
\prod_{(i,j)\in\alpha\cap\beta}\!\!\!\Big(1+(1-2p_0)\tilde\eta\,\mathds{1}\{\Theta_i=\Theta_j\in[\ell]\}\Big)\!
\Bigg]\\
&\le
(\ell\lambda)^{|\alpha\triangle\beta|}
\mathbb{E}_{ \QQ_\Theta}\Bigg[
\prod_{(i,j)\in\alpha\triangle\beta}\mathds{1}\{\Theta_i=\Theta_j\in[\ell]\} 
\prod_{(i,j)\in\alpha\cap\beta}\Big(1+\tilde\eta\,\mathds{1}\{\Theta_i=\Theta_j\in[\ell]\}\Big)
\Bigg].
\end{align*}
 and we note that we have obtained exactly the form in~\eqref{eq:upper_bound_common_eq} and~\eqref{eq:upper_bound_common_neq} in the proofs of Lemmas~\ref{lem:second-moment-equal} and~\ref{lem:second-moment-not-equal} respectively. (This is what motivated our particular definitions of~$\lambda$ and $\tilde\eta$ on line~\eqref{eq.defn_eta_lambda_PDS}). Hence, after this step, the proof proceeds exactly as in the proofs of Lemmas~\ref{lem:second-moment-equal} and~\ref{lem:second-moment-not-equal}, and the result follows.
\end{proof}

\end{document}